%% file: main.tex
\documentclass[11pt,3p,final,authoryear]{elsarticle}

\usepackage{natbib}

\usepackage[table]{xcolor} %{xcolor}, skip=10pt
\usepackage{hyperref}
\usepackage{url}

\usepackage{amssymb,amsmath,amsthm}
\theoremstyle{plain}
\newtheorem{theorem}{Theorem}[section]

\newtheorem{lemma}[theorem]{Lemma}

\theoremstyle{definition}
\newtheorem{definition}[theorem]{Definition}

\theoremstyle{remark}

\usepackage{amsfonts}
\usepackage{mathtools} %aligndat
\usepackage[ruled,vlined]{algorithm2e}
\usepackage{graphicx, subfigure}
\usepackage{mwe}

\usepackage{booktabs}
\usepackage{multirow}
\usepackage{multicol}
\usepackage{makecell}

\usepackage{lipsum}
\usepackage{dirtytalk}

\usepackage{caption} 

\newcommand{\mathbbm}[1]{\text{\usefont{U}{bbm}{m}{n}#1}}
\newcommand{\model}[1]{\texttt{#1}}
\newcommand{\NaiveBU}{\model{NaiveBU}}
\newcommand{\GroupBU}{\model{GroupBU}}

\newcommand{\TD}{\model{TD}}
\newcommand{\MO}{\model{MO}}
\newcommand{\ERM}{\model{ERM}}
\newcommand{\PERMBU}{\model{PERMBU}}
\newcommand{\MinT}{\model{MinT}}
\newcommand{\SHARQ}{\model{SHARQ}}
\newcommand{\HIRED}{\model{HIRED}}
\newcommand{\HierEtoE}{\model{HierE2E}}
\newcommand{\PROFHiT}{\model{PROFHiT}}

\newcommand{\ARIMA}{\model{ARIMA}}

\newcommand{\GLMPoisson}{\model{GLM-Poisson}}
\newcommand{\SNaive}{\model{SeasonalNaive}}
\newcommand{\Naive}{\model{Naive1}}

\newcommand{\ADAM}{\model{ADAM}}

\newcommand{\GTOP}{\model{GTOP}}

\newcommand{\MQForecaster}{\model{MQ-Forecaster}}
\newcommand{\MLP}{\model{MLP}}

\newcommand{\TempConv}{\model{TempConv}}

\newcommand{\MDN}{\model{MDN}}

\newcommand{\SeqtoSeqC}{\model{Seq2SeqC}}

\newcommand{\ours}{\model{DPMN}}
\newcommand{\ourscomplete}{\emph{Deep Poisson Mixture Network}}
\newcommand{\PMM}{\model{PMM}}

\newcommand{\GluonTS}{\model{GluonTS}}

\newcommand{\MXNet}{\model{MXNet}}
\newcommand{\HYPEROPT}{\model{HYPEROPT}}

\newcommand{\TourismLgains}{{11.8\%}}
\newcommand{\Favoritagains}{{8.1\%}}

\newcommand{\dataset}[1]{\texttt{#1}}

\newcommand{\TourismL}{\dataset{Tourism-L}}

\newcommand{\Traffic}{\dataset{Traffic}}
\newcommand{\Favorita}{\dataset{Favorita}}

\long\def\EDIT#1{{\color{black}{#1}\color{black}}}
\long\def\EDITtwo#1{{\color{black}{#1}\color{black}}}
\long\def\EDITthree#1{{\color{black}{#1}\color{black}}}
\newcommand\blambda{\boldsymbol \lambda}
\newcommand\btheta{\boldsymbol \theta}
\journal{Preprint submitted to International Journal of Forecasting}
\begin{document}
\begin{frontmatter}

\title{Probabilistic Hierarchical Forecasting with Deep Poisson Mixtures}

\affiliation[inst1]{organization={Auton Lab, School of Computer Science, Carnegie Mellon University},
            city={Pittsburgh},
            country={PA}}

\affiliation[inst2]{organization={Forecasting Science, Amazon},
            city={New York},
            country={NY}}

\author[inst1]{Kin G. Olivares}

\author[inst2]{O. Nganba Meetei*}
\ead{meeteio@amazon.com}
\cortext[cor1]{Corresponding author}

\author[inst2]{Ruijun Ma}

\author[inst2]{Rohan Reddy}

\author[inst2]{\\Mengfei Cao}

\author[inst2]{Lee Dicker}

\begin{abstract}
Hierarchical forecasting problems arise when time series have a natural group structure, and predictions at multiple levels of aggregation and disaggregation across the groups are needed.
In such problems, it is often desired to satisfy the aggregation constraints in a given hierarchy, referred to as hierarchical coherence in the literature.
Maintaining coherence while producing accurate forecasts can be a challenging problem, especially in the case of probabilistic forecasting.
We present a novel method capable of
accurate and coherent probabilistic forecasts for time series when reliable hierarchical information is present. 
We call it Deep Poisson Mixture Network (DPMN). It relies on the combination of neural networks and a statistical model for the joint distribution of the hierarchical multivariate time series structure. 
By construction, the model guarantees hierarchical coherence and provides simple rules for aggregation and disaggregation of the predictive distributions. We perform an extensive empirical evaluation comparing the DPMN to other state-of-the-art methods which produce hierarchically coherent probabilistic forecasts on multiple public datasets. \EDITthree{Comparing} to existing coherent probabilistic models, we \EDITthree{obtain} a relative improvement in the overall Continuous Ranked Probability Score (CRPS) of \TourismLgains\ on Australian domestic tourism data, and \Favoritagains\  on the Favorita grocery sales dataset, where time series are grouped with geographical hierarchies or travel intent hierarchies. For San Francisco Bay Area highway traffic, where the series' hierarchical structure is randomly assigned, and their correlations are less informative, our method does not show significant performance differences over statistical baselines.
%\EDIT{Command comments} \KO{}, \NM{}, \RM{}, \RR{}, \MF{}, \LD{}.
\end{abstract}


\end{frontmatter}

\section{Introduction and Motivation} \label{section:introduction}
\input{sections/section1_introduction.tex}

\clearpage
\section{Literature Review} \label{section:literature}
\input{sections/section2_literature.tex}

\newpage
\section{Probabilistic Model} \label{section:pmm}
\input{sections/section3_pmm_new.tex}
% \input{sections/section3_pmm.tex}

% \clearpage
\section{Parameter Estimation and Inference}\label{section:estimation_inference}
\input{sections/section4_estimation_inference.tex}

% \clearpage
\section{Deep Poisson Mixture Network} \label{section:pmmcnn}
\input{sections/section5_pmmcnn.tex}

% \clearpage
\section{Empirical Evaluation} \label{section:experiments}
\input{sections/section6_0_data.tex}
\input{sections/section6_2_quant.tex}
\input{sections/section6_2_quant_tables.tex}

\clearpage
\input{sections/section6_5_mean_forecast.tex}

\newpage
\section{Conclusions and Future Work} \label{section:conclusion}
\input{sections/section7_conclusion.tex}

\newpage
\bibliography{citations.bib}
\bibliographystyle{model5-names}
%\biboptions{longnamesfirst} % "longnamesfirst" made the text harder to read, the technical editor can use them once the paper is accepted

\clearpage
\appendix 
\label{section:appendix}
\input{sections/section_appendix.tex}

\end{document}

%% file: sections/section1_introduction.tex
We study the task of \EDIT{probabilistic coherent} time series forecasting where users need predictive distributions for all related time series organized into a hierarchy or group structure %under coherence constraints 
\EDITtwo{\citep{hyndman2016hierarchical_groupedfast, athanasopoulos2017hierarchical_temporal, spiliotis2020hierarchical_cross,panagiotelis2023probabilistic_reconciliation,taieb2017coherent_prob_forecasts}}. As \EDITtwo{forecasts} for different aggregation levels drive different decisions, forecast coherence %constraints 
is desired to ensure aligned decision-making across the hierarchies \citep{petropoulos2021forecasting}. Notable examples of hierarchical forecasting tasks include the necessity from energy planners to synchronize the electricity load at each level of the grid with total production \citep{taieb2019hierarchical_regularized_regression, jeon2019coherent_quantile_forecasts}, the short-term load category in the Global Energy Forecasting Competition 2012 (GEFCOM2012; \citealt{hong2014global}), and the efforts from the forecasting community manifested at the fifth Makridakis Competition (M5; \citealt{makridakis2020m5competition_results}).

% Here coherent forecast is defined as forecasts which satisfy the aggregation constraints of the hierarchy \citep{hyndman2017forecasting_book}. 
% This definition is easy to understand for mean forecasts which are additive. 
% For probabilistic forecasts, it means that the probability measure associated with events mapped to one another by the aggregation constraints are the same \citep{gamakumara_2018}. 
% An equivalent definition for probabilistic coherence is that marginal distributions at all aggregations levels are derived from a joint distribution \citep{taieb2017coherent_prob_forecasts,jeon2019hierarchical_regional_wind_electricity}.%\RM{added a reference}.

%Classically, 
Coherent forecasts are defined as those that satisfy the aggregation constraints of the hierarchy. That is, dis-aggregated \EDITtwo{forecasts} \say{add up} to the \EDITtwo{forecasts} of aggregate levels. This definition is accessible for mean forecasts, which are additive \EDITtwo{by linearity of the expectation}. 
For probabilistic forecasts, \EDITtwo{coherence is achieved} when the forecast distribution of the aggregate series is identical to the distribution of the sum of its children's forecast series under an implicit or explicit joint distribution \EDITtwo{\citep{taieb2017coherent_prob_forecasts,taieb2021hierarchical_electricity,wickramasuriya2023probabilistic_gaussian,panagiotelis2023probabilistic_reconciliation,panagiotelis2020hierarchical_probabilistic_coherence}}. 
Hierarchical reconciliation strategies provide an interesting approach for bringing back \EDITtwo{mean and probabilistic hierarchical coherence} into neural forecasting methods. Early work focused on reconciling independently generated mean base forecasts \citep{hyndman2011optimal_combination_hierarchical, wickramasuriya2019hierarchical_mint_reconciliation}. The reconciliation strategies improved accuracy, and recently a better understanding of the reconciliation process was provided through the language of forecast combinations \citep{hollyman2021hierarchical_understanding_reconciliation}. Similar two-step forecast reconciliation methods were later extended to probabilistic forecasts as well \citep{taieb2017coherent_prob_forecasts, gamakumara_2018}, first estimating the marginal distributions independently and then reconciling them. Finally \cite{han2021hierarchical_sharq} and~\cite{rangapuram2021hierarchical_e2e} proposed combining these two steps into a single neural network. \EDITtwo{Efficiently leveraging the \emph{cross-learning} approach~\citep{makridakis2018m4competition_results,spiliotis2021cross_learning} to improve accuracy while maintaining \EDITtwo{probabilistic} coherence remains a challenge.}

In this work, we present a novel method for producing \EDITtwo{probabilistic coherent} forecasts. It combines the strength of modern neural networks and an intuitive statistical model for the \EDITtwo{disaggregated-forecast} joint distribution. In contrast to earlier efforts~\citep{han2021hierarchical_sharq, rangapuram2021hierarchical_e2e}, our method is \EDITthree{an extension to} the \emph{Mixture Density Networks} (\MDN; \citealt{bishop_mdn_1994}). 
%\RM{Should we say the model is related to MDN, rather than is an MDN? Because MDN is a mixture of continuous distributions and one of the reviewer is very sensitive to the distinguish between continuous vs discrete distributions?} 
It is coherent by construction and does not require an explicit re-conciliation step, either as part of a single end-to-end network or as a separate step. 
We call it the \ourscomplete \ (\ours). The \ours \ models the joint probability \EDITthree{mass function} of the multivariate time series as a finite mixture of Poisson distributions and combines it with the well-established \MQForecaster \ neural architecture \citep{wen2017mqrcnn, eisenach2020mqtransformer}. This is possible because we formulate the problem as an \MDN, and we can choose a relevant class of probabilistic distributions for the statistical model and the neural architecture independently. The key advantages of our method are: 
%\RM{While sharing common neural network architecture with \MQForecaster, in this work, we do not seek to evaluate if probabilistic coherence promotes or worsens model accuracy at different granularities. We seek to understand how \ours \ performs on hierarchical probabilistic forecasting problems, comparing to other probabilistic coherent models.} 

\begin{enumerate}%[(i)]
    \item \textbf{Flexible Forecast Distribution}:
    We model the forecast distribution as a finite mixture of Poisson random variables, which is analogous to a Poisson kernel density. The resulting distribution is flexible, capable of accurately modeling a wide range of joint probability distributions, and compatible as an output layer with state-of-the-art neural architectures. We will demonstrate this empirically on three different forecasting tasks in Section~\ref{section:experiments}.
    \item \textbf{Computational Efficiency}: Learning coherent forecast distributions in a high dimensional hierarchical space can be computationally intractable. To alleviate this, we anchor the \ours \ on a multivariate distribution of the bottom level time series and employ composite likelihood optimization strategies, which enables it to extend to large-scale applications.
\end{enumerate}

% \newpage
The rest of the work is structured as follows. 
In Section~\ref{section:literature} we introduce mathematical notations and review the \emph{statistical} and \emph{neural-network} hierarchical forecast literature. 
We describe our method's probabilistic model in Section~\ref{section:pmm} and the learning and inference methods in Section~\ref{section:estimation_inference}. 
In Section~\ref{section:pmmcnn} we discuss the neural network architecture and in Section~\ref{section:experiments} we perform an empirical evaluation. 
Finally, in Section~\ref{section:conclusion} we discuss future work and conclude.
%  where we showcase \ours's unique characteristics and advantages compared to other coherent models.

%% file: sections/section2_literature.tex
\subsection{Hierarchical Forecasting Notation}
\label{section:hierarchical_forecasting_notation}
%\RM{We consider forecasting for discrete non-negative time series discrete time steps}
%\KO{I would keep the hierarchical notation agnostic, to Poissons in this section.} \NM{Agree with keeping notations distribution agnostic here. This is literature review. We can clarify that when we come to our model definition}
% \RM{We start to talk about $\tau$ here, maybe introducing it before hand is better.}
% \KO{I think that Section 2, Section 6, Appendix B, Appendix C and Figure 1 are already consistent across the paper. All the taus are meant for relative indexes rather than the fcd t. Which is a flexible notation capable of dealing with cross-validation double indexes, forking sequences and closely matching the Equations from marginal distributions and covariance structure.}
Mathematically a hierarchical multivariate time series can be denoted by the vector $\mathbf{y}_{[a,b],t} = [\,\mathbf{y}_{[a],t}^{\top} \;|\; \mathbf{y}_{[b],t}^{\top}\,]^{\top} \in \mathbb{R}^{(N_{a}+N_{b})}$ for each time point $t$; where $[a],[b]$ stand for \EDITtwo{the set of} all aggregate and bottom indices of the time series respectively. %, and $\alpha \in [a], \beta \in [b]$ are single aggregate and bottom time series indices. 
The total number of series in the hierarchy is $|\,[a,b]\,|=(N_{a}+N_{b})$, where $|\,[a]\,|=N_{a}$ is the number of aggregated series and $|\,[b]\,|=N_{b}$ the number of bottom series that are at the most \EDITthree{disaggregated} level possible. 
Time indices for past information until $t$ are given by the set $[t]$ with length $|\,[t]\,|=N_{t}$. With this notation the hierarchical aggregation constraints at each time point $t$ have the following matrix representation:
\begin{equation}\label{eqn:summation}
\mathbf{y}_{[a,b],\EDITthree{t}}  = \EDITtwo{\mathbf{S}_{[a,b][b]}} \mathbf{y}_{[b],\EDITthree{t}} \quad \Leftrightarrow \quad 
\begin{bmatrix}\mathbf{y}_{[a],\EDITthree{t}}
\\ %\hline
\mathbf{y}_{[b],\EDITthree{t}}\end{bmatrix} 
= \begin{bmatrix}
\EDITtwo{\mathbf{A}_{[a][b]}}\\ %\hline
\mathbf{I}_{[b][b]}
\end{bmatrix}\mathbf{y}_{[b],\EDITthree{t}}
\end{equation}

The matrix $\EDITtwo{\mathbf{S}_{[a,b][b]}} \in \mathbb{R}^{(N_{a}+N_{b})\times N_{b}}$ aggregates the bottom level to the series above.
It is composed by stacking the aggregation matrix $\EDITtwo{\mathbf{A}_{\mathrm{[a][b]}}} \in \mathbb{R}^{N_{a}\times N_{b}}$ and an $N_b\times N_b$ identity matrix $\mathbf{I}_{[b][b]}$.

For example, Figure \ref{fig:hierarchical_tree} represents a simple hierarchy where each parent node is the sum of its children. 
Here the dimensions are $N_{a}=3$, $N_{b}=4$, and the hierarchical, aggregated and base series are respectively:
\begin{equation}
    \begin{aligned}
        y_{\mathrm{Total},\EDITthree{t}} = y_{\beta_{1},\EDITthree{t}}+y_{\beta_{2},\EDITthree{t}}+y_{\beta_{3},\EDITthree{t}}+y_{\beta_{4},\EDITthree{t}} 
        \qquad \qquad \qquad \qquad \qquad \\
        \mathbf{y}_{[a],\EDITthree{t}}=\left[y_{\mathrm{Total},\EDITthree{t}},\; y_{\beta_{1},\EDITthree{t}}+y_{\beta_{2},\EDITthree{t}},\;y_{\beta_{3},\EDITthree{t}}+y_{\beta_{4},\EDITthree{t}}\right]^{\intercal} 
        \qquad
        \mathbf{y}_{[b],\EDITthree{t}}=\left[ y_{\beta_{1},\EDITthree{t}},\; y_{\beta_{2},\EDITthree{t}},\; y_{\beta_{3},\EDITthree{t}},\; y_{\beta_{4},\EDITthree{t}} \right]^{\intercal}
    \end{aligned}
\label{eq:hierarchical_aggregations}
\end{equation}

\begin{figure}[tb]
\centering
\includegraphics[width=112mm]{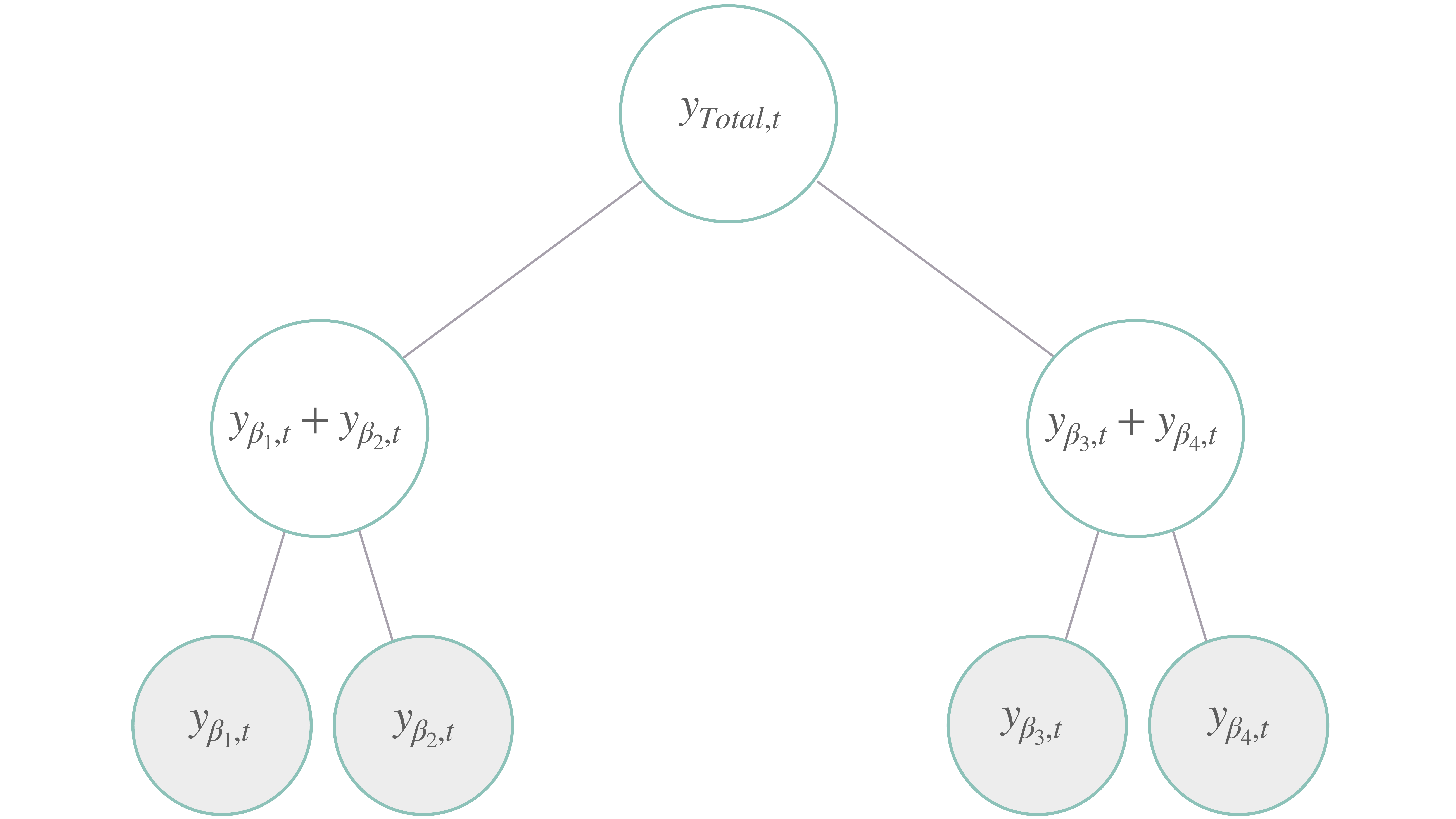}
\caption{\textit{A simple three level time series hierarchical structure, with four bottom level variables. The disaggregated bottom variables are marked with gray background. In this description each node represents non overlapping series for a single point in time.}}
\label{fig:hierarchical_tree}
\end{figure}

%The summing matrix of the Figure~\ref{fig:hierarchical_tree} example can be written as:
% \clearpage
The constraint matrix of the Figure~\ref{fig:hierarchical_tree} example and the corresponding aggregations from Equation~(\ref{eq:hierarchical_aggregations}) \EDIT{is the following}:
\begin{equation*}
\EDITtwo{\mathbf{S}_{[a,b][b]}}
=
\begin{bmatrix}
           \\
\EDITtwo{\mathbf{A}_{\mathrm{[a][b]}}} \\ 
           \\ %\hline%\hdashline[2pt/2pt]
           \\
\mathbf{I}_{\mathrm{[b][b]}} \\
           \\
\end{bmatrix}
=
\begin{bmatrix}
\;1 & 1 & 1 & 1 \\
\;1 & 1 & 0 & 0 \\
\;0 & 0 & 1 & 1 \\ \hline
\;1 & 0 & 0 & 0 \\
\;0 & 1 & 0 & 0 \\
\;0 & 0 & 1 & 0 \\
\;0 & 0 & 0 & 1 \\
\end{bmatrix}
\end{equation*}

%================================================================
% \clearpage
\subsection{Mean Forecast Reconciliation Strategies}
\label{sec:reconciliation_strategies}
%================================================================

\EDITthree{Given a forecast creation date $t$ and horizon $h$, the forecast indexes are denoted by $\tau \in [t+1:t+h]$.}    
Most of the prior statistical solutions to the hierarchical forecasting challenge implement a two-stage process, first generating base forecasts $\hat{\mathbf{y}}_{[a,b],\tau}\in \mathbb{R}^{N_{a}+N_{b}}$, and then revising them into coherent forecasts $\tilde{\mathbf{y}}_{[a,b],\tau}$ through reconciliation. The reconciliation is compactly expressed by:
\begin{equation}\label{eqn:hyndman_framework}
\tilde{\mathbf{y}}_{[a,b],\tau} = \EDITtwo{\mathbf{S}_{[a,b][b]}} \mathbf{P}_{[b][a,b]} \hat{\mathbf{y}}_{[a,b],\tau}
\end{equation}
where $\EDITtwo{\mathbf{S}_{[a,b]}}\in \mathbb{R}^{(N_{a}+N_{b}) \times N_{b}}$ is the hierarchical aggregation matrix and $\mathbf{P}_{[b][a,b]} \in \mathbb{R}^{N_{b} \times (N_{a}+N_{b})}$ \EDIT{is a matrix} determined by the reconciliation strategies. The most common reconciliation methods can be classified into \emph{top-down}, \emph{bottom-up} and \emph{alternative} approaches. %We describe the strategies below.

\begin{itemize}
    \item Bottom-Up: The simple \emph{bottom-up} strategy, abbreviated as \NaiveBU~\citep{orcutt1968hierarchical_bottom_up}, first generates bottom level forecasts and then aggregates them to produce \EDITtwo{forecasts} for all the series in the multivariate structure. %Here $\mathbf{P}=[\mathbf{0}_{\mathrm{[b],[a]}}\;|\;\mathbf{I}_{\mathrm{[b][b]}}]$. %A challenge of the bottom-up strategy is that bottom series tend to be noisy which makes the identification of seasonal and trend components difficult.
    
    \item Top-Down: The \emph{top-down} strategy, abbreviated as \TD~\citep{gross1990hierarchical_top_down, fliedner1999hierarchical_top_down2}, distributes the total forecast, and then disaggregates it down the hierarchy using proportions that can be historical actuals or forecasted separately. %In this strategy $\mathbf{P}=[\mathbf{p}_{\mathrm{[b]}}\;|\;\mathbf{0}_{\mathrm{[b][a,b\;-1]}}]$. This method has several variants, depending on the way the proportions $\mathbf{p}$ are created.%, like the \emph{average of historical proportions}, \emph{proportions of historical averages}, or the \emph{forecasted proportions}. %Challenges of this approach are that it is only defined for strictly hierarchical series, and not grouped series, and it does not always produce unbiased coherent forecasts \citep{hyndman2011optimal_combination_hierarchical}.
    
    \item Alternative: The more recent \emph{middle-out} strategies, denoted as \MO~\citep{hyndman2011optimal_combination_hierarchical, hyndman2017forecasting_book}), treat the second stage reconciliation as an optimization problem for the matrix $\mathbf{P}_{[b][a,b]}$. These reconciliation techniques include among others \emph{Game-Theoretically OPtimal} (\GTOP; \cite{van2015game}), learning a projection for reducing quadratic loss, a generalized least squares model for minimizing trace of the squared error matrix, namely the \emph{minimum trace} reconciliation (\MinT; \cite{wickramasuriya2019hierarchical_mint_reconciliation}) and the \emph{empirical risk minimization} approach (\ERM; \cite{taieb2019hierarchical_regularized_regression}).
\end{itemize}

\vspace{3mm}
Despite the advancements in alternative reconciliation strategies with statistical solutions, as mentioned in Section~\ref{section:introduction}, there are still fundamental limitations. 
First, most post-process reconciliation methods produce mean forecasts but not probabilistic forecasts, with some exceptions that \EDIT{have relied on univariate statistical methods with strong probability assumptions for the base series that may be restrictive \citep{taieb2017coherent_prob_forecasts, panagiotelis2020hierarchical_probabilistic_coherence}}.
\EDIT{Second, the mentioned methods independently learn the model parameters of the base level forecasts, limiting the base model's optimization inputs to single series. This approach induces an over-fitting prone setting for complex nonlinear methods, which, as noted by the forecasting community, is one of their biggest challenges \citep{makridakis2018machine_learning_concerns}. The implied data scarcity translates into a missed opportunity to leverage the flexibility of nonlinear methods.}

%and the induced data scarcity translates into a missed opportunity to share a common expressive model that leverages the power of nonlinear transformations of the data.

%================================================================
% \clearpage
\subsection{Probabilistic Coherent Forecasting}
\label{section:probabilistic_coherent_forecasting}
%================================================================

There is a large body of related work on Bayesian hierarchical\footnote{ Note that the term hierarchy in Bayesian hierarchy is different from its use in hierarchical forecasting. The former refers to the conditional dependence of a posterior distribution's parameters, and the latter refers to the aggregation constraints across multiple time series.} modeling of Spatio-temporal data; \EDITtwo{with joint coherent predictive distributions} \citep{wikle1998hierarchical, diggle2013}. These models, however, come with strong assumptions, as they typically assume a stationary Gaussian process to induce correlations and rely on Markov Chain Monte Carlo to estimate the posterior distribution \citep{diggle2001, diggle2013}. For modeling count data, variants of Bayesian Hierarchical Poisson Regression Models were introduced \citep{christiansen1997hierarchical, park2009application}, but they usually operate with linear \EDITthree{model} assumptions. There are few specialized methods on coherent probabilistic forecasting as most research on hierarchical forecasting has been limited to point predictions. Exceptions are the work by \cite{shang2017coherent_quantile_forecasts} and \cite{jeon2019coherent_quantile_forecasts} that provide an early exploration of forecast quantile reconciliation; \cite{taieb2017coherent_prob_forecasts} and \cite{taieb2021hierarchical_electricity} that propose the combination of bottom-level forecast marginal distributions with empirical copula functions describing their dependencies to create aggregate predictive distributions. To the best of our knowledge, a formal definition of probabilistic coherence has only been explored by \EDITthree{\citep{taieb2021hierarchical_electricity, gamakumara_2018, wickramasuriya2023probabilistic_gaussian, panagiotelis2023probabilistic_reconciliation}. \cite{taieb2021hierarchical_electricity}}
%\citealt{panagiotelis2023probabilistic_reconciliation}
provide\EDITthree{s} a convolution-based definition while~\EDITthree{\cite{panagiotelis2023probabilistic_reconciliation}} 
provides a generalized and intuitive framework\footnote{\HierEtoE\ by \cite{rangapuram2021hierarchical_e2e} informally introduce\EDITthree{d} a similar probabilistic coherence notion.} that we follow in our work and introduce below:

\begin{definition}
\label{def:probabilistic_coherence} (Probabilistic Coherence).
Let $(\Omega_{[b]}, \mathcal{F}_{[b]}, \hat{\mathbb{P}}_{[b]})$ be a probabilistic forecast space, \EDITthree{with sample space $\Omega_{[b]}$}, $\mathcal{F}_{[b]}$ its $\sigma$-algebra, \EDITthree{and $\hat{\mathbb{P}}_{[b]}$ a forecast probability}. 
\EDITthree{Let $\mathbf{S}_{[a,b][b]}(\cdot):\Omega_{[b]} \mapsto \Omega_{[a,b]}$ be the linear transformation implied by the constraints matrix}.
A coherent probabilistic forecast space $(\Omega_{[a,b]}, \mathcal{F}_{[a,b]}, \hat{\mathbb{P}}_{[a,b]})$ satisfies:
\begin{equation}
    \hat{\mathbb{P}}_{[a,b]}\left(\mathbf{S}_{[a,b][b]}(\mathcal{B})\right) = \hat{\mathbb{P}}_{[b]}\left(\mathcal{B}\right)
    \quad \text{for any set } \mathcal{B} \in \mathcal{F}_{[b]} 
    \text{ and \EDITthree{set's image} } \mathbf{S}_{[a,b][b]}(\mathcal{B}) \in \mathcal{F}_{[a,b]} 
\end{equation}

\noindent i.e., it assigns a zero probability to sets in $\mathbb{R}^{N_{a}+N_{b}}$ not containing any coherent forecasts.
\end{definition}

For a simple definition-satisfying example, consider three random variables ($Y_{\alpha}, Y_{\beta_{1}}, Y_{\beta_{2}}$) with $Y_{\alpha}\EDITthree{:=}Y_{\beta_{1}} + Y_{\beta_{2}}$. A coherent forecast assigns zero probability to the variable realizations $(y_{\alpha},\; y_{\beta_{1}},\; y_{\beta_{2}})$ if they do not satisfy the aggregation constraint $y_{\alpha}=y_{\beta_{1}} + y_{\beta_{2}}$. An equivalent definition of coherence is to require that the  marginal distributions are derivable from the joint distribution of the bottom level random variables. In this case, the probability function of interest is $\hat{\mathbb{P}}(Y_{\beta_{1}}, Y_{\beta_{2}})$, and the marginal probabilities can be derived from it \EDITtwo{using indicator functions} as follows:

\EDITtwo{
\begin{align*}
    \hat{\mathbb{P}}(Y_{\beta_{1}}=y_{\beta_1}) = \sum_{y_{\beta_{2}}} \hat{\mathbb{P}}(Y_{\beta_{1}}=y_{\beta_1}, Y_{\beta_{2}}=y_{\beta_2}) \quad \text{and} \quad %\\
    \hat{\mathbb{P}}(Y_{\beta_{2}}=y_{\beta_2}) = \sum_{y_{\beta_{1}}} \hat{\mathbb{P}}(Y_{\beta_{1}}=y_{\beta_1}, Y_{\beta_{2}}=y_{\beta_2}) \\
    \hat{\mathbb{P}}(Y_{\alpha}=y_{\alpha}) = 
    \sum_{y_{\beta_{1}},y_{\beta_{2}}} \hat{\mathbb{P}}(Y_{\beta_{1}}=y_{\beta_{1}}, Y_{\beta_{2}}=y_{\beta_{2}}) \; \mathbbm{1}(y_{\alpha}=y_{\beta_{1}} + y_{\beta_{2}})\EDITthree{,} \quad\quad\quad\quad\;
\end{align*}
}
\EDITthree{
\noindent where $\mathbbm{1}(y_{\alpha}=y_{\beta_{1}} + y_{\beta_{2}})$ equals 1 when $y_{\alpha}=y_{\beta_{1}} + y_{\beta_{2}}$, and 0 otherwise.}
The marginal probability for $Y_{\alpha}$, has the aggregation constraint built into its definition and is by construction a hierarchically coherent probability with respect to $\hat{\mathbb{P}}(Y_{\beta_{1}})$ and $\hat{\mathbb{P}}(Y_{\beta_{2}})$. Note that knowledge of the joint distribution is not required to generate hierarchically coherent forecast. However, if we had access to the joint distribution\EDITthree{,} constructing hierarchically coherent marginal distributions is straight forward.

%================================================================
% \clearpage
\subsection{Hierarchical Neural Forecasting}
\label{section:neural_forecasting_literature}
%================================================================

In the last decade, neural network-based forecasting methods have become ubiquitous in large-scale forecasting applications \citep{wen2017mqrcnn, bose2017probabilistic_scale, madeka2018sample, eisenach2020mqtransformer}, transcending industry boundaries into academia, as it has redefined the state-of-the-art in many practical tasks \citep{yu2017graphconv_traffic, ravuri2021weather_deepmind, olivares2022nbeatsx} and forecasting competitions \citep{makridakis2018m4competition_results, makridakis2020m5competition_results}. 

The latest neural network-based solutions to the hierarchical forecasting challenge include methods like the \emph{Simultaneous Hierarchically Reconciled Quantile Regression} (\SHARQ; \cite{han2021hierarchical_sharq}) and \emph{Hierarchically Regularized Deep Forecasting} (\HIRED; \cite{paria2021hierarchical_hired}) and \EDIT{the \emph{Probabilistic Robust Hierarchical Network} (\PROFHiT; \cite{kamarthi2022profhit_network})} that augment the training loss function with approximations to the hierarchical constraints. 
And \emph{Hierarchical End-to-End} learning (\HierEtoE; \cite{rangapuram2021hierarchical_e2e}) that integrates an alternative reconciliation strategy in its \EDITtwo{forecasts} through linear projections. 
With the exception of \HierEtoE, the rest of these methods encourage \EDITtwo{probabilistic coherence} through regularization but do not guarantee it. 
Additionally, if a user requires updating the hierarchical structure of interest, a whole new optimization of the networks would be needed for the existing methods to forecast the structure correctly. 

Our proposed method, \ours, can address these deficiencies by \EDIT{specifying any network’s output as our proposed Poisson Mixture predictive distribution}. With this predictive distribution, a model needs only to forecast the bottom-most level in the hierarchy, after which any desired hierarchical structure of interest can be predicted with guaranteed probabilistic hierarchical coherence.

%The \ours \ model with Poisson distributional assumption on observed data was primarily motivated by count data, but the model family can be extended to answer other time series forecasting problems.

%% file: sections/section3_pmm_new.tex
\begin{figure}[t]
\centering
\includegraphics[width=90mm]{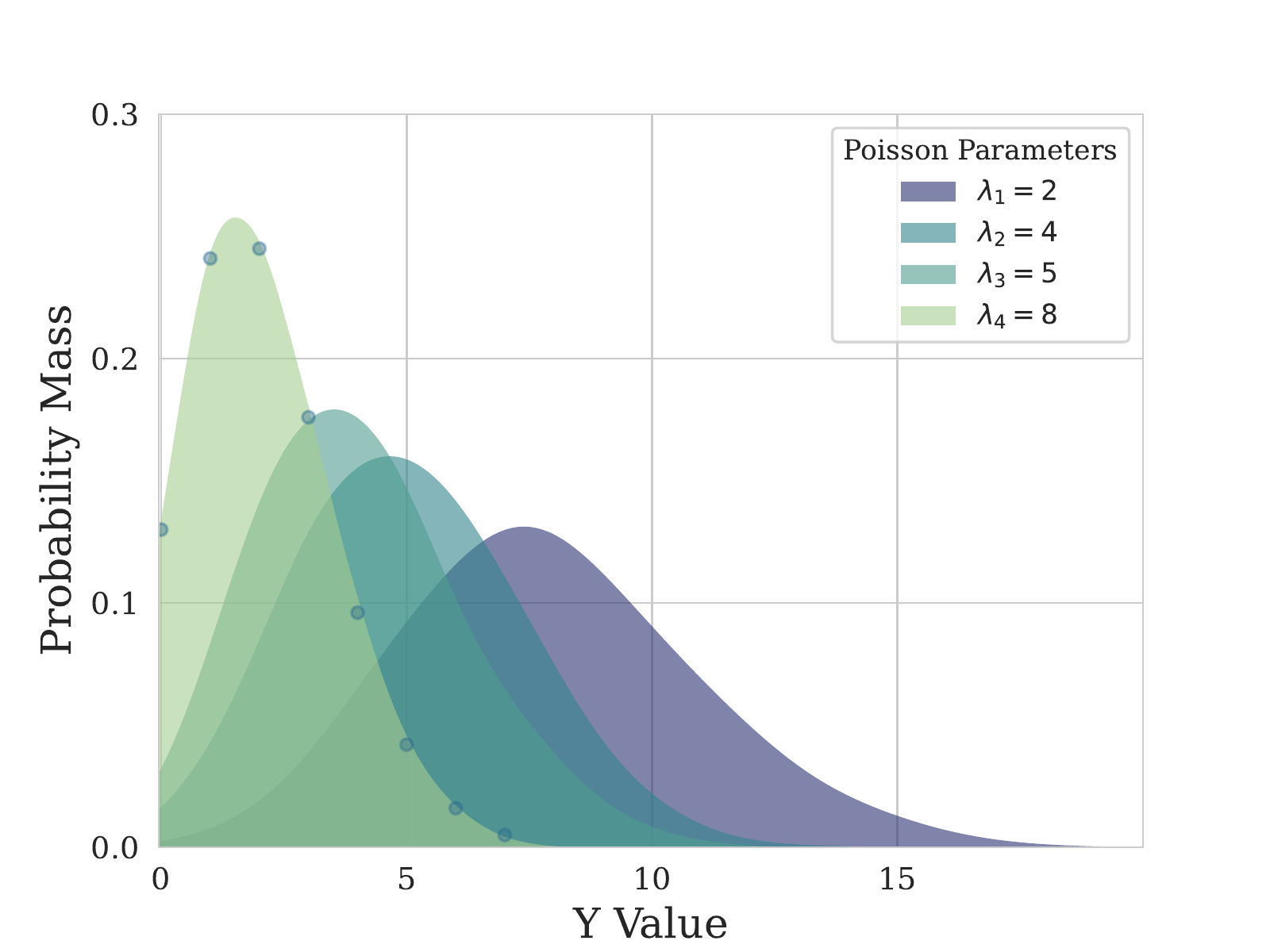}
\caption{The Poisson Mixture distribution has desirable properties that make it well-suited for probabilistic hierarchical forecasting for count data. 
Under minimal conditions, its aggregation rule implies probabilistic coherence of the random variables it models.
}
\label{fig:poisson_mixture}
\end{figure}

\subsection{Joint Poisson Mixture Distribution}

\EDITtwo{In this work, we focus our attention on hierarchical forecasting of \EDITthree{discrete events with non-negative variables}. Many forecasting problems fall in this category as shown by the data sets described in Section \ref{section:experiments}. However, the general framework presented here works for continuous distributions as well with a suitably chosen base class for the mixture distribution, e.g., Gaussian distribution. For discrete events,} we start by postulating that the forecast joint probability \EDITthree{of the bottom level future multivariate time series' realization $\mathbf{y}_{[b][t+1:t+h]} \in \mathbb{N}^{N_{b} \times h}$} is:
\begin{equation}
\label{eq:joint_distribution}
\begin{split}
    \hat{\mathbb{P}}\left(\mathbf{y}_{[b][t+1:t+\EDITthree{h}]}|\;\blambda_{[b][k][t+1:t+h]}\right) 
    &= \sum_{\EDITthree{\kappa}=1}^{\EDITthree{N_{k}}} w_\EDITthree{\kappa} \prod_{(\beta,\tau) \in [b][t+1:t+h]} 
    \mathrm{Poisson} \left(y_{\EDITthree{\beta},\EDITthree{\tau}}|\lambda_{\EDITthree{\beta},\EDITthree{\kappa},\EDITthree{\tau}} \right) \\
    &= \sum_{\EDITthree{\kappa}=1}^{\EDITthree{N_{k}}} w_\EDITthree{\kappa} \prod_{(\beta,\tau) \in [b][t+1:t+h]} 
    \frac{\lambda_{\beta,\EDITthree{\kappa},\tau}^{y_{\beta,\tau}}}{y_{\beta,\tau}!} e^{-\lambda_{\beta,\EDITthree{\kappa},\tau}}
\end{split}
\end{equation}

The mixture model describes \EDITthree{individual} bottom level \EDITthree{time series} and their correlations through the distribution of the mixture's latent variables defined by the Poisson rates $\blambda_{[b][k][t+1:t+h]} \in \mathbb{R}^{N_{b} \times N_{k} \times h}$ and the associated weights \EDITtwo{$\mathbf{w}_{[k]} \in [0,1]^{N_{k}}$}, with $\mathbf{w}_{[k]} \geq 0$ and $\sum_{\kappa=1}^{N_k}w_\kappa = 1$. The number of Poisson components $\mid[k]\mid = N_k$ is a hyperparameter of the model that controls the flexibility of the mixture distribution. The mixture distribution can also be interpreted as a multivariate kernel density approximation, with Poisson kernels, to the actual joint probability distribution \EDITtwo{$\mathbb{P}\left(\mathbf{y}_{[b][t+1:t+\EDITthree{h}]}\right)$}. We show an example of the Poisson mixture distribution in Figure~\ref{fig:poisson_mixture}.

The joint distribution in Equation~(\ref{eq:joint_distribution}), assumes that the modeled observations $\mathbf{y}_{[b][t+1:t+h]}$ are \emph{conditionally independent} given the latent Poisson rates $\blambda_{[b][k][t+1:t+h]}$. That is for all bottom level series and horizons $ (\beta, \tau) \neq (\beta',\tau') \text{, } (\beta, \tau), (\beta', \tau') \in [b][t+1:t+h]$ and $\kappa \in [k]$:
\begin{equation}
\label{eq:conditional_indep}
    \hat{\mathbb{P}}(Y_{\beta,\tau},  Y_{\beta',\tau'} \vert \lambda_{\beta,\EDITthree{\kappa},\tau},  \lambda_{\beta',\EDITthree{\kappa},\tau'} )
    = \hat{\mathbb{P}}(Y_{\beta,\tau} \vert \lambda_{\beta,\EDITthree{\kappa},\tau}) \hat{\mathbb{P}}(Y_{\beta',\tau'},  \vert \lambda_{\beta',\EDITthree{\kappa},\tau'})
\end{equation}

\subsection{Marginal Distributions for Bottom Series}\label{sec:bottom_marginal_distribution}
Equation~(\ref{eq:joint_distribution}) describes the joint distribution of all bottom level time series.
We can derive the marginal distribution for one of the bottom level series $\beta \in [b]$ and for a single future time period $\tau \in [t+1:t+h]$ by integrating out \EDITthree{the remaining series and time indices}. The marginal distribution is:
\begin{align}
    \hat{\mathbb{P}}(Y_{\beta, \tau} = y_{\beta, \tau}) 
    &= \sum_{\mathbf{y}_{[b][t+1:t+h] \setminus (\beta, \tau)}} \sum_{\kappa=1}^{N_{k}} w_\kappa \prod_{(\beta^\prime,\tau^\prime) \in [b][t+1:t+h]}\mathrm{Poisson}(y_{\beta^\prime,\tau^\prime} \vert \lambda_{\beta^\prime,\kappa,\tau^\prime}) \nonumber \\
    &= \sum_{\kappa=1}^{N_{k}} w_\kappa \; \mathrm{Poisson}(y_{\beta,\tau} \vert \lambda_{\beta,\kappa,\tau}) \times \nonumber \prod_{(\beta^\prime,\tau^\prime) \in [b][t+1:t+h] \setminus (\beta, \tau)} \quad \sum_{y_{\beta^\prime,\tau^\prime}}\mathrm{Poisson}(y_{\beta^\prime,\tau^\prime} \vert \lambda_{\beta^\prime,\kappa,\tau^\prime}) \nonumber \\
    &= \sum_{\kappa=1}^{N_k} w_\kappa \; \mathrm{Poisson}(y_{\beta,\tau} \vert \lambda_{\beta,\kappa,\tau})
\end{align}

\EDIT{ We get a clean final expression for the marginal distribution}, which is equivalent to simply dropping all other time series and time periods from the product in Equation~(\ref{eq:joint_distribution}).

\subsection{Marginal Distributions for Aggregate Series}
\label{sec:probabilistic_properties}
\EDIT{ An important consequence of the \emph{conditional independence} in Equation~(\ref{eq:conditional_indep}) is that the marginal distributions at aggregate levels $\mathbf{y}_{[a],\tau}$ can be computed via simple component-wise addition of the lower level Poisson rates. For example consider an aggregate level \EDITtwo{variable $Y_{\alpha, \tau} \EDITthree{:=} Y_{\beta_1, \tau} + Y_{\beta_2, \tau}$}. The marginal distribution for \EDITtwo{$Y_{\alpha, \tau}$} can be derived from the joint distribution in Equation~(\ref{eq:joint_distribution}) as follows. First marginalize all other series and time indices (as done in Section \ref{sec:bottom_marginal_distribution} above), which gives us the joint distribution for \EDITtwo{$Y_{\beta_1, \tau}$ and $Y_{\beta_2, \tau}$}
\EDITtwo{

\begin{equation*}
    \hat{\mathbb{P}}(Y_{\beta_1, \tau}=y_{\beta_1, \tau}, Y_{\beta_2, \tau}=y_{\beta_2, \tau}) = \sum_{\kappa=1}^{N_{k}} w_\kappa \; \mathrm{Poisson}(y_{\beta_1,\tau} \vert \lambda_{\beta_1,\kappa,\tau}) \times \mathrm{Poisson}(y_{\beta_2,\tau} \vert \lambda_{\beta_2,\kappa,\tau})
\end{equation*}

Now, the aggregate marginal distribution is 
\begin{eqnarray*}
    \hat{\mathbb{P}}(Y_{\alpha, \tau}=y_{\alpha, \tau}) &=& \sum_{y_{\beta_1, \tau}, y_{\beta_2, \tau}} \hat{\mathbb{P}}(Y_{\beta_1, \tau}=y_{\beta_1, \tau}, Y_{\beta_2, \tau}=y_{\beta_2, \tau}) \; \mathbbm{1}(y_{a,\tau}=y_{\beta_1, \tau} + y_{\beta_2, \tau})   
    %\\
    %            &=& \sum_{\kappa=1}^{N_{k}} w_\kappa \sum_{y_{\beta_1, \tau}, y_{\beta_2, \tau}} \mathrm{Poisson}(y_{\beta_1,\tau} \vert \lambda_{\beta_1,\kappa,\tau}) \; \mathrm{Poisson}(y_{\beta_2,\tau} \vert \lambda_{\beta_2,\kappa,\tau}) \times \\
    %            && \quad \mathbbm{1} (y_{a,\tau}=y_{\beta_1, \tau} + y_{\beta_2, \tau})
\end{eqnarray*}
}
For each mixture component \EDITthree{with index} $\kappa$, the distributions of $y_{\beta_1, \tau}$ and $y_{\beta_2, \tau}$ conditioned on respective Poisson rates $\lambda_{\beta_1, \kappa, \tau}$ and $\lambda_{\beta_2, \kappa, \tau}$ are independent Poisson distributions, and therefore, the distribution of the aggregate is another \EDITthree{Poisson Mixture} with parameters $\lambda_{\EDITthree{\alpha},\kappa, \tau} = \lambda_{\beta_1, \kappa, \tau} + \lambda_{\beta_2, \kappa, \tau}$.
\begin{equation}
    \hat{\mathbb{P}}(\EDITtwo{Y_{\alpha,\tau}=}y_{\alpha, \tau}) =  \sum_{\kappa=1}^{N_{k}} w_\kappa \mathrm{Poisson}(y_{\EDITthree{\alpha},\tau} \vert \lambda_{\EDITthree{\alpha},\kappa, \tau} = \lambda_{\beta_1, \kappa, \tau} + \lambda_{\beta_2, \kappa, \tau})
\end{equation}

The \emph{aggregation rule} can be concisely described as: 
\begin{equation} \label{eq:aggregation_rule_concise}
\blambda_{[a][k],\tau} = \mathbf{A}_{\mathrm{[a][b]}}\blambda_{[b][k], \tau}
\end{equation}

\noindent with $\mathbf{A}_{\mathrm{[a][b]}}$ the hierarchical aggregation matrix defined in Section~\ref{section:hierarchical_forecasting_notation}.} The joint predictive distribution is \EDITthree{hierarchically} coherent by construction. \EDITtwo{We offer a formal proof of \ours'\EDITthree{s} satisfaction of the probabilistic coherence property from Definition~\ref{def:probabilistic_coherence} in \ref{sec:probabilistic_coherence}.}

%Then for any aggregate level $\alpha \in [a]$ and time $\tau \in [t+1:t+h]$, the marginal distribution is:
%\EDITtwo{
%\begin{equation}
%    \hat{\mathbb{P}}(\mathbf{Y}_{\alpha,\tau}=\mathbf{y}_{[a], \tau}) = \sum_{\kappa=1}^{N_k} w_\kappa \mathrm{Poisson}(\mathbf{y}_{[a],\tau} \vert \blambda_{[a],\kappa, \tau})
%\end{equation}
%}

\subsection{Covariance Matrix}
\label{sec:covariance_modeling}
Using the law of total covariance and the \emph{conditional independence} from Equation~(\ref{eq:conditional_indep}), we show in \ref{appen:correlation} that the covariance of any two bottom level series naturally follows: \EDITtwo{
\begin{equation}\label{eq:covariance}
    \mathrm{Cov}(Y_{\beta, \tau}, Y_{\beta', \tau'}) = \overline{\lambda}_{\beta,\tau} \EDITthree{\mathbbm{1}(\beta=\beta')\mathbbm{1}(\tau=\tau')} + \sum_{\kappa=1}^{N_k} w_\kappa \left(\lambda_{\beta,\kappa,\tau} - \overline{\lambda}_{\beta,\tau}\right) \left( \lambda_{\beta',\kappa,\tau'} - \overline{\lambda}_{\beta',\tau'}\right)
\end{equation}
%\begin{equation}\label{eq:covariance}
%    \mathrm{Cov}(y_{\beta, \tau}, y_{\beta', \tau'}) = \overline{\lambda}_{\beta,\tau}\delta_{\beta, \beta'}\delta_{\tau, \tau'} + \sum_{\kappa=1}^{N_k} w_\kappa \left(\lambda_{\beta,\kappa,\tau} - \overline{\lambda}_{\beta,\tau}\right) \left( \lambda_{\beta',\kappa,\tau'} - \overline{\lambda}_{\beta',\tau'}\right)
%\end{equation}

\noindent where $\overline{\lambda}_{\beta,\tau} = \sum_{\kappa=1}^{N_k} w_\kappa \lambda_{\beta, \kappa, \tau} $ }. \EDITtwo{\ref{sec:componence_covariance_expressiveness} shows the non-diagonal covariance matrix expressivity, as determined by its rank, depends on the number of mixture components from Equation~(\ref{eq:joint_distribution}). 
%\RM{Is $\delta_{\beta,\beta'}$ equivalent to the indicator function ${1}(\beta=\beta')$? If so, could we use the latter one to avoid introducing a new notation?} \MF{Agree also because $\delta$ is used without definition; i edited equation 9; please check}\KO{That is Dirac's delta definition, lets keep the indicator notation.}
}
%\EDITtwo{This observation can be easily extended to the general multivariate covariance matrix:
%\begin{equation}
%\label{eq:covariance_matrix}
%    \mathrm{Cov}(\mathbf{y}_{[b],\tau}) = 
%    \sum^{K}_{\kappa=1} w_{\kappa}
%    (\blambda_{[b],\kappa,\tau}-\bar{\blambda}_{[b][k],\tau})
%    (\blambda_{[b],\kappa,\tau}-\bar{\blambda}_{[b][k],\tau})^{\intercal} \in \mathbb{R}^{N_{b} \times N_{b}}
%\end{equation} \NM{This feels redundant. We should keep one version, the explicitly indexed one above or the general notation below.}

%% file: sections/section4_estimation_inference.tex
\subsection{Maximum Joint Likelihood}

To estimate model parameters, we can maximize the joint likelihood implied by the joint distribution from Equation~(\ref{eq:joint_distribution}). 
Let $\btheta$ represent the neural network parameters as described in Section~\ref{section:pmmcnn}, we parameterize the probabilistic model with \EDITtwo{Poisson rates $\blambda_{[b][k][t+1:t+h]}$} as follows:

% \begin{equation}
% \begin{aligned}
%     \blambda_{[b][k][t+1:t+h]} &:= \hat{\blambda}_{[b][k][t+1:t+h]}(\btheta \;| \mathbf{y}_{[b][:t]}, \mathbf{x}^{(h)}_{[b][:t]}, \mathbf{x}^{(f)}_{[b][t+1:t+h]}, \mathbf{x}^{(s)}_{[b]}) \\
%     \mathbf{w}_{[k]} &:= \hat{\mathbf{w}}_{[k]}(\btheta \;| \tilde{\mathbf{x}}^{(h)}_{[:t]}, \tilde{\mathbf{x}}^{(s)})
% \end{aligned}
% \end{equation}
\begin{equation}
\begin{aligned}
    \blambda_{[b][k][t+1:t+h]} &:= \hat{\blambda}_{[b][k][t+1:t+h]}\EDITthree{(\btheta)} \\
    \mathbf{w}_{[k]} &:= \hat{\mathbf{w}}_{[k]}\EDITthree{(\btheta)}
\end{aligned}
\end{equation}

% \textcolor{red}{We condition the probability on the history of the bottom level time series, other associated historical covariates, future information available at the forecast generation time and static features, denoted by $\mathbf{y}_{[b][:t]}$, $\mathbf{x}^{(h)}_{[b][:t]}$, $\mathbf{x}^{(f)}_{[b][t+1:t+h]}$ and $\mathbf{x}^{(s)}_{[b]}$ respectively.} And the shared mixture weights $\mathbf{w}_{[k]}$, are conditioned on temporal and static aggregate features shared across the bottom series, $\tilde{\mathbf{x}}^{(h)}_{[:t]}$ and $\tilde{\mathbf{x}}^{(s)}$ respectively.

% We denote the combined conditioning information as
% \begin{equation}
%     \mathbf{x}^{(h)}_{[:\EDITthree{t}]} = \{\mathbf{x}^{(h)}_{[b][:\EDITthree{t}]}, \tilde{\mathbf{x}}^{(h)}_{[:\EDITthree{t}]}\} \quad \text{and} \quad \mathbf{x}^{(s)}=\{\mathbf{x}^{(s)}_{[b]}, \tilde{\mathbf{x}}^{(s)}\}
% \end{equation}

\EDITthree{The Poisson rates and the mixture weights are conditioned through the network's parameters on forecasting features discussed in Section~\ref{sec:hierarchical_forecasting_features}.} The negative log-likelihood can then be written\footnote{We kept notations simple and omitted the explicit conditioning on input features.}:
\begin{equation}\label{eq:negative_log_joint_likelihood}
    \mathcal{L}(\btheta) = - \mathrm{log} \left[\sum_{\kappa=1}^{N_{k}} \hat{w}_\kappa(\btheta) \prod_{(\beta,\tau) \in [b][t+1:t+h]}\left( \frac{(\hat{\lambda}_{\beta,\kappa,\tau}(\btheta))^{^{y_{\beta,\tau}}}\exp{\{-\hat{\lambda}_{\beta,\kappa,\tau}(\btheta)\}}}{(y_{\beta,\tau})!}\right) \right]
\end{equation}
 
\EDIT{This is the same expression as the multivariate probability \EDITthree{mass function} in Equation~(\ref{eq:joint_distribution}) but parametrized as a function of the neural network parameters $\theta$}. The maximum likelihood estimation method (MLE) has desirable properties like statistical efficiency and consistency. However, the mixture components cannot be estimated separately, \EDITthree{and} for this reason, MLE is only feasible for hierarchical time series with a small number of series. Additional exploration is needed to make the estimation scalable.

\subsection{Maximum Composite Likelihood}

\EDITthree{A computationally efficient} alternative to MLE for estimating the model parameters is to maximize the composite likelihood. This method involves breaking up the high-dimensional space into smaller sub-spaces, and the composite likelihood consists of the weighted product of the marginal likelihoods of the subspaces \citep{lindsay_1988}. For simplicity, we used uniform weights. %which can affect the efficiency of parameter estimation

In addition to the computational efficiency, maximizing the composite likelihood provides a robust and unbiased estimate of marginal model parameters with the drawback that the model inference may suffer from properties similar to a misspecified model \citep{varin_2011}. We will discuss variants of the composite likelihood below.

% \clearpage
\subsubsection{Naive Bottom Up}
A simple option of using composite likelihood is to define each bottom-level time series as its likelihood sub-space and treat them as independent during model training ~\citep{orcutt1968hierarchical_bottom_up}. We refer to this estimation method for the \ours \ as \emph{Naive Bottom Up} (\ours-\NaiveBU). The negative logarithm of the \ours-\NaiveBU \ composite likelihood is:
\begin{equation}\label{eq:composite_likelihood_naiveBU}
    \mathcal{L}_{\mathrm{NaiveBU}}(\btheta) = - \sum_{\beta \in [b]}\mathrm{log} \left[\sum_{\kappa=1}^{N_{k}} \hat{w}_\kappa(\btheta) \prod_{\tau \in [t+1:t+h]}\left( \frac{(\hat{\lambda}_{\beta,\kappa,\tau}(\btheta))^{^{y_{\beta,\tau}}}\exp{\{-\hat{\lambda}_{\beta,\kappa,\tau}(\btheta)\}}}{(y_{\beta,\tau})!}\right) \right]
\end{equation}
\EDIT{ Even though the sub-space consists of single bottom level time series, \ours-\NaiveBU \ is still a multi-variate model with the composite likelihood defined over multiple time points $[t+1:t+h]$. Maximizing the \ours-\NaiveBU \ composite likelihood will still learn correlations across the time points} and will generate coherent forecast distributions for aggregations in the time dimension. It does not attempt, however, to discover correlations across different time series. 
%\EDIT{Despite the independent time series assumption made in model training, the predictive joint distribution will remain in the form as described in Equation \eqref{eq:joint_distribution}, we defer the discussion to Section \ref{sec:forecast_inference}.}

\subsubsection{Group Bottom Up}
If prior information helps us identify groups of time series with interesting correlation structures, we may estimate them by including the groups in the composite likelihood. We refer to this estimation method for the \ours \ as \emph{Group Bottom Up} (\ours-\GroupBU). Let $\mathcal{G}=\{[g_i]\}$ be time-series groups, then the negative log composite likelihood for the \ours-\GroupBU \ is 
\begin{equation}\label{eq:composite_likelihood_groupBU}
    \mathcal{L}_{\mathrm{GroupBU}}(\btheta) = - \sum_{[g_{i}] \in \mathcal{G}}\mathrm{log} \left[\sum_{\kappa=1}^{N_{k}} \hat{w}_\kappa(\btheta) \prod_{(\beta,\tau) \in [g_i] [t+1:t+h]}\left( \frac{(\hat{\lambda}_{\beta,\kappa,\tau}(\btheta))^{^{y_{\beta,\tau}}}\exp{\{-\hat{\lambda}_{\beta,\kappa,\tau}(\btheta)\}}}{(y_{\beta,\tau})!}\right) \right]
\end{equation}
%In this paper, we view naive bottom up as a naive baseline for \ours, while group bottom up is the working solution we propose for jointly forecasting for hierarchical time series. MF: i removed this sentence because these two methods may perform differently under different conditions; when the prior information is informative, then GroupBU is good but when it's not, NaiveBU seems to be better. We should advise readers to use with caution, instead of GroupBU only.

The main advantage over \ours-\NaiveBU \ is that the model now learns cross-\EDITthree{series} relationships. \EDIT{In this paper we only \EDITthree{rely} on intuitive grouping like geographic proximity, but one could in principle employ more sophisticated methods like clustering to define the groups}.  To optimize the learning objective we use stochastic gradient descent, and sample train series batches at the group level.

% \clearpage
\subsection{Forecast Inference}\label{sec:forecast_inference}

As mentioned earlier, model inference from composite likelihood estimation suffers from problems similar to model misspecification. \EDIT{This is because of the independence assumed across sub-spaces. In our model, the maximum composite likelihood estimates do not understand how mixture components learnt for one sub-space relate to mixture components learnt for a different subspace. However, we need to identify mixture components across sub-spaces in order to define the joint distribution in Equation~(\ref{eq:joint_distribution}) across all bottom level time series. Knowing this joint distribution is at the crux of forecast inference from our model. Fortunately, there is a natural way for us to resolve this issue. Both in the \ours-\NaiveBU\ composite likelihood in Equation~(\ref{eq:composite_likelihood_naiveBU}) and in the \ours-\GroupBU\ composite likelihood in Equation~(\ref{eq:composite_likelihood_groupBU}), the weights $\hat{\mathbf{w}}_{[k]}(\theta) \in \mathbb{R}^{N_{k}}$ are shared across all sub-spaces. We identify components with the same weight as belonging to the same multivariate sample, and hence providing the full joint distribution. We call this method \emph{weight matching}.}

%If we use the \ours-\NaiveBU \ model for prediction, the model does not understand the covariance structure defined in Equation~(\ref{eq:covariance}). The prediction from the model does not consider the dependencies of the random variables is because the model is fitted to maximize the likelihood of each time series after marginalizing all other time series. To recover the joint distribution in Equation~(\ref{eq:joint_distribution}) we need a method for relating the marginal distributions learnt from the maximum composite likelihood method to one another. Here our choice of sharing sample weights $\mathbf{w}_{[k]}$ across all marginal distributions provide a natural scheme. We identify Poisson components with the same weight as part of a multivariate sample in the high dimensional space. In other words, we define a \emph{component matching method} where we associate the first mixture component for $\beta \in [b]$ with the first mixture component for any other $\beta'\in [b]$, the second component of $\beta$ with the second component of $\beta'$ and so on. With such mapping defined, recovering the joint distribution is trivial. This method of associating mixture components may look like a forecast reconciliation step, but it is different from the notion of forecast reconciliation used in other hierarchical forecasting methods to fix incoherence.
%We are not trying to fix incoherence, instead, we are filling in missing information in the model without any additional computational cost.  

For the \ours-\NaiveBU \ estimation, the weight matching method is \EDITthree{a strong statistical assumption, extrapolating from independent marginal distributions of each series to the full joint distribution}. The \ours-\GroupBU \ approach significantly alleviates this problem because the model parameters are well defined within each group $[g_i] \in \mathcal{G}$ and if most of the interesting correlations are already captured within each group, then much less burden is placed on the weight matching method. We show in the evaluation of Section \ref{section:experiments} that both \ours-\NaiveBU \ and \ours-\GroupBU \ models perform favorably when compared to other \EDITtwo{mean and probabilistic hierarchical} forecasting methods, and between the two, \EDIT{\ours-\GroupBU \ is generally more accurate when the time series grouping given is informative.}

%% file: sections/section5_pmmcnn.tex
Our primary goal is to create a \EDITtwo{probabilistic coherent} forecasting model that is accurate and efficient. 
For this purpose, we \EDITthree{opt} to extend the \MQForecaster\ family \citep{wen2017mqrcnn, eisenach2020mqtransformer}, \EDIT{proven by its history of industry service}, with the Poisson mixture distribution. We refer to this model as the \ourscomplete \ (\ours). \EDIT{Our \MQForecaster\ architecture selection is driven by its high computational efficiency, consequence of the \emph{forking-sequences} technique and multi-step forecasting strategy, \EDITthree{in} addition to its ability to incorporate static and \EDITthree{known} future temporal features.}

\begin{figure}[t]
\centering
\includegraphics[width=170mm]{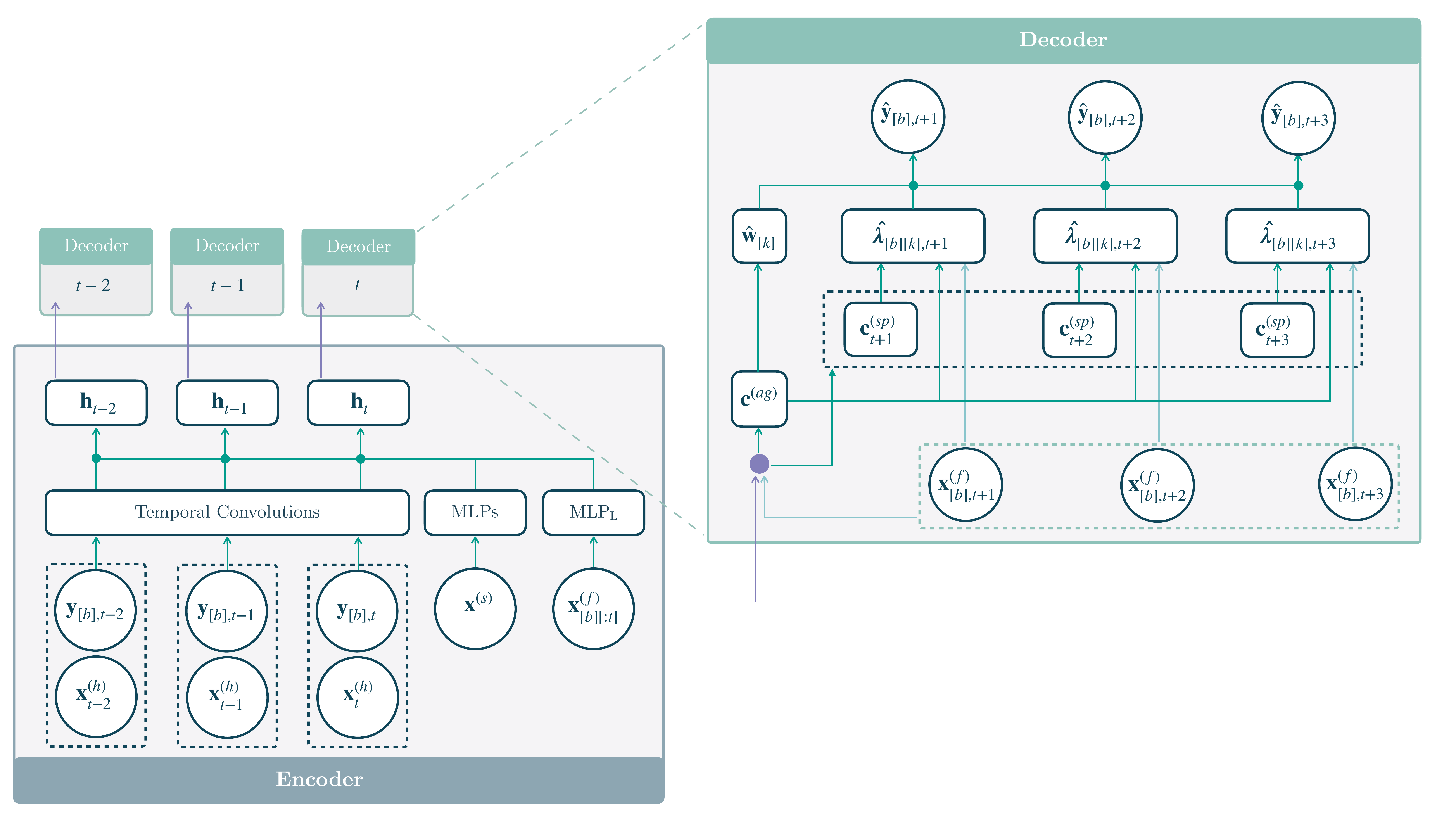}
\caption{The \ourscomplete \ (\ours) is a \emph{Sequence-to-Sequence with Context} network that uses dilated temporal convolutions as the primary encoder and multilayer perceptron based decoders for a direct \EDITthree{multi-step} forecast. The forked decoders share their parameters and create the \EDITthree{multivariate forecast distribution} for each time point in the encoder, making the architecture efficient in its optimization and \EDITtwo{predictions}.}
\label{fig:pmm_cnn}
\end{figure}

%================================================================
\subsection{Model Features}
\label{sec:hierarchical_forecasting_features}
%================================================================

As part of the innovations within our work, we propose to separate the bottom-level and aggregate-level features that we use in the \EDITtwo{forecasts}. \EDITthree{Sharing aggregate-level features across their respective bottom series simplifies the model's inputs and reduces redundant information, while greatly improving the model's memory usage efficiency.}

We follow the hierarchical forecasting literature practice for the \emph{static features} and \EDITthree{use} group identifiers implied by the hierarchy structure. We denote the static features, composed of features shared across the bottom series $\tilde{\mathbf{x}}^{(s)}$ and bottom level-specific features $\mathbf{x}_{[b]}^{(s)}$ by:
\begin{equation}
\mathbf{x}^{(s)} = \{\mathbf{x}^{(s)}_{[b]},\;\,\tilde{\mathbf{x}}^{(s)}\}.
\end{equation}

Regarding \emph{temporal features}, for the bottom level, we use the bottom series' past; for the aggregate level we use the parent node series' past. The future temporal information available can be \EDITthree{task specific  inputs} like prices or promotions \EDITthree{in the context of product demand forecasting}, or other simpler model forecasts \EDITthree{as inputs}, \EDITthree{such as \Naive\ or \SNaive\ that help the model predict series levels and seasonalities}. Similarly to the static features, we also distinguish the temporal shared features
$\tilde{\mathbf{x}}^{(h)}_{[:\EDITthree{t}]},\tilde{\mathbf{x}}^{(f)}_{[:\EDITthree{t+h}]}$
and the temporal bottom level-specific features $\mathbf{x}^{(h)}_{[b][t]},\mathbf{x}^{(f)}_{[b][:t+h]}$. With this consideration, we denote the historical and future temporal features \EDITthree{by:}

\begin{equation}
\mathbf{x}^{(h)}_{[:\EDITthree{t}]} = \{\mathbf{x}^{(h)}_{[b][:\EDITthree{t}]},\;\,\tilde{\mathbf{x}}^{(h)}_{[:\EDITthree{t}]}\} 
\qquad \text{and} \qquad
\mathbf{x}^{(f)}_{[:\EDITthree{t+h}]} = \{\mathbf{x}^{(f)}_{[b][:\EDITthree{t+h}]},\;\,\tilde{\mathbf{x}}^{(f)}_{[:\EDITthree{t+h]}}\}
\end{equation}

\subsection{Model Architecture}

In summary, the \ours \ builds on the \MQForecaster \ architecture that is based on \emph{Sequence-to-Sequence with Context} network (\SeqtoSeqC; \cite{cho2014Seq2SeqC}). 
The \ours \ uses dilated temporal convolutions (\TempConv; \cite{vandenoord2016wavenet}) to encode the available history into hidden states and uses forked decoders based on \emph{Multi-Layer Perceptron} (\MLP; \cite{rosenblatt1961principles}) in a direct multi-horizon forecast strategy \citep{atiya2016multi_step_forecasting}.  We describe below in further detail the components of the model. Other hyperparameter details are available in Table~\ref{table:model_parameters}.

% \clearpage
\subsubsection{Encoder}

As explained earlier the \ours \ main encoder is a stack of dilated temporal convolutions. 
Additionally, we use a global dense layer to encode the static features and a local dense layer, shared across time, to encode the available future information. 
The encoder for each time $\EDITthree{t}$ and its components are described in Equation~(\ref{equation:encoders}). 
% $\mathbf{x}^{(s)}_{[b]}$ $\mathbf{x}^{(s)}_{[b][t:t+h]}$

% \begin{equation}
% \begin{aligned}
%     \mathbf{h}^{(p)}_{[a,b],\EDITthree{t}}  &= [\mathbf{h}_{[a],\EDITthree{t}}^{(p)} \,|\, \mathbf{h}_{[b],\EDITthree{t}}^{(p)}] = [\mathbf{TempConv}(\mathbf{x}^{(p)}_{[a][:\EDITthree{t}]}) \;|\; \mathbf{TempConv}(\mathbf{x}^{(p)}_{[b][:\EDITthree{t}]})] \\
%     \mathbf{h}^{(s)}_{[a,b]} &= [\mathbf{h}^{(s)}_{[a]} \;\,\,|\,\mathbf{h}^{(s)}_{[b]}]\;\, =  [\mathbf{MLP}(\mathbf{x}^{(s)}_{[a]}) \;|\; \mathbf{MLP}(\mathbf{x}^{(s)}_{[b]})] \\
%     \mathbf{h}^{(f)}_{\EDITthree{t}}  &= \mathbf{MLP}_{L}(\mathbf{x}^{(f)}_{[b][:\EDITthree{t}+h]}) \\
%     %\mathbf{h}_{\EDITthree{t}}        &= \{\mathbf{h}_{[a],\EDITthree{t}},\;\, \mathbf{h}_{[b],\EDITthree{t}}\} 
%     %= \{[\mathbf{h}^{(p)}_{[a],\EDITthree{t}}| \mathbf{h}^{(s)}_{[a]}],\;\; [\mathbf{h}^{(p)}_{[a,b],\EDITthree{t}} | \mathbf{h}^{(s)}_{[a,b]} | \mathbf{h}^{(f)}_{\EDITthree{t}}]\} \\    
%     \label{equation:encoders}
% \end{aligned}
% \end{equation}

\begin{equation}
\begin{aligned}
    \mathbf{h}^{(h)}_{\EDITthree{t}}  &= \{\mathbf{h}_{1,\EDITthree{t}}^{(h)} \,,\, \mathbf{h}_{2,\EDITthree{t}}^{(h)}\} = \{\mathbf{TempConv}(\mathbf{x}^{(h)}_{[b][:\EDITthree{t}]}) \;,\; \mathbf{TempConv}(\tilde{\mathbf{x}}^{(h)}_{[:\EDITthree{t}]})\} \\
    \mathbf{h}^{(s)} &= \{\mathbf{h}_1^{(s)}\,,\,\mathbf{h}_2^{(s)}\} =  \{\mathbf{MLP}(\mathbf{x}^{(s)}_{[b]}) \;,\; \mathbf{MLP}(\tilde{\mathbf{x}}^{(s)})\} \\
    \mathbf{h}^{(f)}_{\EDITthree{t}}  &= \mathbf{MLP}_{L}(\mathbf{x}^{(f)}_{[b][:\EDITthree{t}+h]}) \\
    %\mathbf{h}_{\EDITthree{t}}        &= \{\mathbf{h}_{1,\EDITthree{t}},\; \mathbf{h}_{2,\EDITthree{t}}\} 
    %= \{[\mathbf{h}^{(h)}_{1,\EDITthree{t}} | \mathbf{h}^{(h)}_{2,\EDITthree{t}} | \mathbf{h}^{(s)}_1 | \mathbf{h}^{(s)}_2 | \mathbf{h}^{(f)}_{\EDITthree{t}}],\; [\mathbf{h}^{(h)}_{2,\EDITthree{t}}| \mathbf{h}^{(s)}_2]\} \\
    \label{equation:encoders}
\end{aligned}
\end{equation}
% \RM{What does the subscript $L$ under MLP mean in the equation above?}
The encoder's output in Equation~(\ref{equation:encoders2}) is a set of shared and bottom level encoded features $\mathbf{h}_{1,\EDITthree{t}}$ and $\mathbf{h}_{2,\EDITthree{t}}$. The first concatenates all the encoded past $\mathbf{h}^{(h)}_{1,\EDITthree{t}},\,\mathbf{h}^{(h)}_{2,\EDITthree{t}} \in \mathbb{R}^{N_{cf}}$, static $\mathbf{h}^{(s)}_{1},\,\mathbf{h}^{(s)}_{2} \in \mathbb{R}^{N_{s}}$ and available future $\mathbf{h}^{(f)}_{\EDITthree{t}} \in \mathbb{R}^{N_{f}}$ information\footnote{\EDITthree{The local horizon-specific \MLP$_{L}$\ aligns future seasonalities and events and improves the forecast's sharpness.}}. The second one concatenates the encoded past $\mathbf{h}^{(h)}_{2,\EDITthree{t}} \in \mathbb{R}^{N_{cf}}$ and static $\mathbf{h}^{(s)}_{2} \in \mathbb{R}^{{N}_{s}}$ shared features.
\begin{equation}
\begin{aligned}
    %\mathbf{h}_{\EDITthree{t}}        &= \{\mathbf{h}_{[a],\EDITthree{t}},\;\, \mathbf{h}_{[b],\EDITthree{t}}\} 
    %= \{[\mathbf{h}^{(p)}_{[a],\EDITthree{t}}| \mathbf{h}^{(s)}_{[a]}],\;\; [\mathbf{h}^{(p)}_{[a,b],\EDITthree{t}} | \mathbf{h}^{(s)}_{[a,b]} | \mathbf{h}^{(f)}_{\EDITthree{t}}]\} \\
    \mathbf{h}_{\EDITthree{t}}        &= \{\mathbf{h}_{1,\EDITthree{t}},\; \mathbf{h}_{2,\EDITthree{t}}\} 
    = \{[\mathbf{h}^{(h)}_{1,\EDITthree{t}} | \mathbf{h}^{(h)}_{2,\EDITthree{t}} | \mathbf{h}^{(s)}_1 | \mathbf{h}^{(s)}_2 | \mathbf{h}^{(f)}_{\EDITthree{t}}],\;\;\, [\mathbf{h}^{(h)}_{2,\EDITthree{t}}| \mathbf{h}^{(s)}_2]\} \\    
    \label{equation:encoders2}
\end{aligned}
\end{equation}

\subsubsection{Forked Decoders}
%\RM{We used $[\EDITthree{t}+1:\EDITthree{t}+h]$ below but used $[t+1:t+h]$ in Section 1. These two notations have been used exchange-ably throughout the paper (e.g. equation 20), if they have the same meaning, could we align on using one notation?}
%\KO{Done}
The \ours \ uses a two-branch \MLP \ decoder. 
The first decoder branch, summarizes the encoder output and future available information into two contexts: The horizon-agnostic context set $\mathbf{c}^{(ag)} \in \mathbb{R}^{N_{ag}}$ and the horizon-specific context $\mathbf{c}^{(sp)}_{[\EDITthree{t}+1:\EDITthree{t}+h]} \in \mathbb{R}^{N_{sp}\times h}$ that provides structural awareness of the forecast horizon and plays a crucial role in expressing recurring patterns in the time series if any. Equation~(\ref{equation:forked_decoders1}) describes the first decoder branch: %  common across the forecast horizon

%captures common data across the forecast horizon
%distance between the forecast creation date and the specific horizon

% \begin{equation}
% \begin{aligned}
%     &\mathbf{c}^{(ag)} = \{\mathbf{c}^{(ag)}_{[a]},\, \mathbf{c}^{(ag)}_{[b]} \} = \{\mathbf{MLP}(\mathbf{h}_{[a],\EDITthree{t}}),\;\, \mathbf{MLP}(\mathbf{h}_{[b],\EDITthree{t}})\} \\  &\mathbf{c}^{(sp)}_{[b][\EDITthree{t}+1:\EDITthree{t}+h]} = \mathbf{MLP}_{L}(\mathbf{h}_{[b],\EDITthree{t}}) \\
%     \label{equation:forked_decoders1}
% \end{aligned}
% \end{equation}

\begin{equation}
\begin{aligned}
    &\mathbf{c}^{(ag)} = \{\mathbf{c}^{(ag)}_1,\, \mathbf{c}^{(ag)}_2 \} = \{\mathbf{MLP}(h_{1,\EDITthree{t}}),\, \mathbf{MLP}(\mathbf{h}_{2,\EDITthree{t}})\} \\  &\mathbf{c}^{(sp)}_{[\EDITthree{t}+1:\EDITthree{t}+h]} = \mathbf{MLP}_{L}(\mathbf{h}_{1,\EDITthree{t}}) \\
    \label{equation:forked_decoders1}
\end{aligned}
\end{equation}    
%\RM{Bold $h$ in equation 19 above?}
%\KO{Thanks}
The second decoder branch adapts the horizon-specific and horizon-agnostic contexts into the parameters of the Poisson mixture distribution. For the horizon-specific Poisson rates, we use the forking-sequence technique with a series of decoders with shared parameters for each time point \EDITthree{in $[t+1:t+h]$} and for the mixture weights, we apply an \MLP\ followed by a softmax on the aggregate horizon agnostic context. Equation~(\ref{equation:forked_decoders2}) describes the second decoder branch:

% \begin{equation}
% \begin{aligned}
%     \hat{\blambda}_{[b][k][t+1:t+h]} &= \mathbf{MLP}_{L}(\mathbf{c}^{(ag)}_{[b]},\, \mathbf{c}^{(sp)}_{[b][\EDITthree{t}+1:\EDITthree{t}+h]},\, \mathbf{x}^{(f)}_{[b][\EDITthree{t}+1:\EDITthree{t}+h]}) \\
%     \hat{\mathbf{w}}_{[k]} &= \mathrm{SoftMax}(\mathbf{MLP}(\mathbf{c}^{(ag)}_{[a]}))
%     \label{equation:forked_decoders2}
% \end{aligned}
% \end{equation}

\begin{equation}
\begin{aligned}
    \hat{\blambda}_{[b][k][t+1:t+h]} &= \mathbf{MLP}_{L}(\mathbf{c}^{(ag)}_{1},\, \mathbf{c}^{(sp)}_{[\EDITthree{t}+1:\EDITthree{t}+h]},\, \mathbf{x}^{(f)}_{[b][\EDITthree{t}+1:\EDITthree{t}+h]}) \\
    \hat{\mathbf{w}}_{[k]} &= \mathrm{SoftMax}(\mathbf{MLP}(\mathbf{c}^{(ag)}_{2}))
    \label{equation:forked_decoders2}
\end{aligned}
\end{equation}

%% file: sections/section6_0_data.tex
%\addtolength{\tabcolsep}{-2pt}    
\subsection{Hierarchical Forecasting Datasets}
\label{section:hierarchical_datasets}

To evaluate our method, 
%on probabilistic hierarchical forecasting accuracy and scalability, 
we consider three forecasting tasks where the objective is to \EDITtwo{provide} quantile forecasts for each time series in the group or hierarchy structure. 
All the three datasets that we use in the empirical evaluation~\footnote{
\Traffic \ is available at the \href{https://archive.ics.uci.edu/ml/datasets/PEMS-SF}{\textcolor{blue}{UCI ML repository}}.
\TourismL \ is available at \href{https://robjhyndman.com/publications/mint/}{\textcolor{blue}{\MinT \ reconciliation web page}}. 
\Favorita \ is available in its \href{https://www.kaggle.com/c/favorita-grocery-sales-forecasting/}{\textcolor{blue}{Kaggle Competition url}}.} are publicly available and have been used in the hierarchical forecasting literature~\citep{wickramasuriya2019hierarchical_mint_reconciliation, taieb2019hierarchical_regularized_regression, rangapuram2021hierarchical_e2e,paria2021hierarchical_hired}. \EDIT{Table~\ref{table:datasets_summary} summarizes the datasets' characteristics\EDITthree{.}\footnote{\EDITtwo{We include more details for the \Traffic, \TourismL, and \Favorita\ datasets in \ref{section:dataset_details}}.}}
%\MF{Added reference for HierE2E that used Traffic and Tourism. Need to add reference that used Favorita for experiments. One possible issue is that, we motivated the problem by giving the example of electricity planning, but we did not really experiment with electricity dataset though.}\KO{\citep{paria2021hierarchical_hired} uses the Favorita dataset.}

\input{tables/datasets.tex}

\begin{figure}[t]
    \centering
    \subfigure[Traffic]{
    \label{fig:traffic}
    \includegraphics[width=0.23\linewidth]{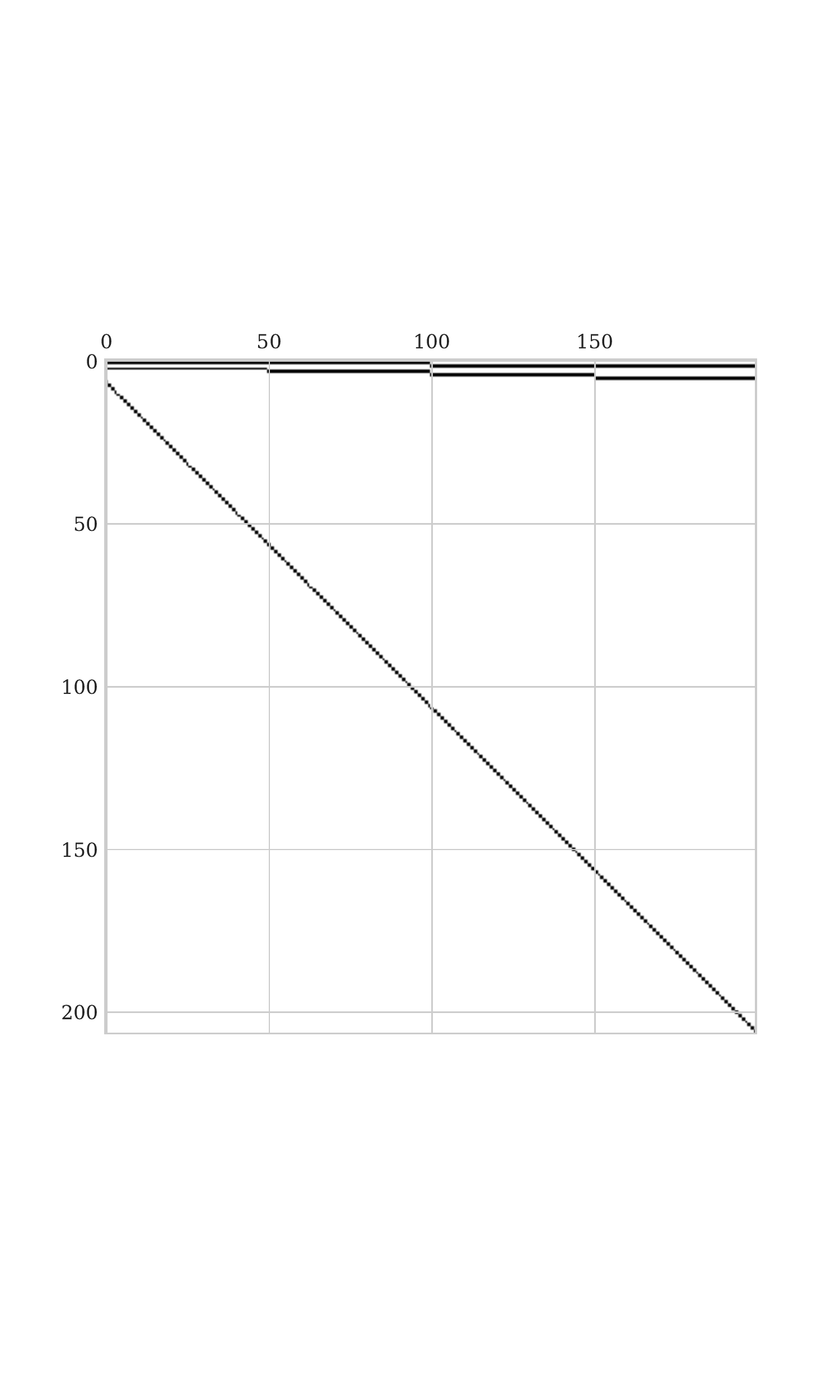}
    }
    \subfigure[Tourism]{
    \label{fig:tourism}
    \includegraphics[width=0.23\linewidth]{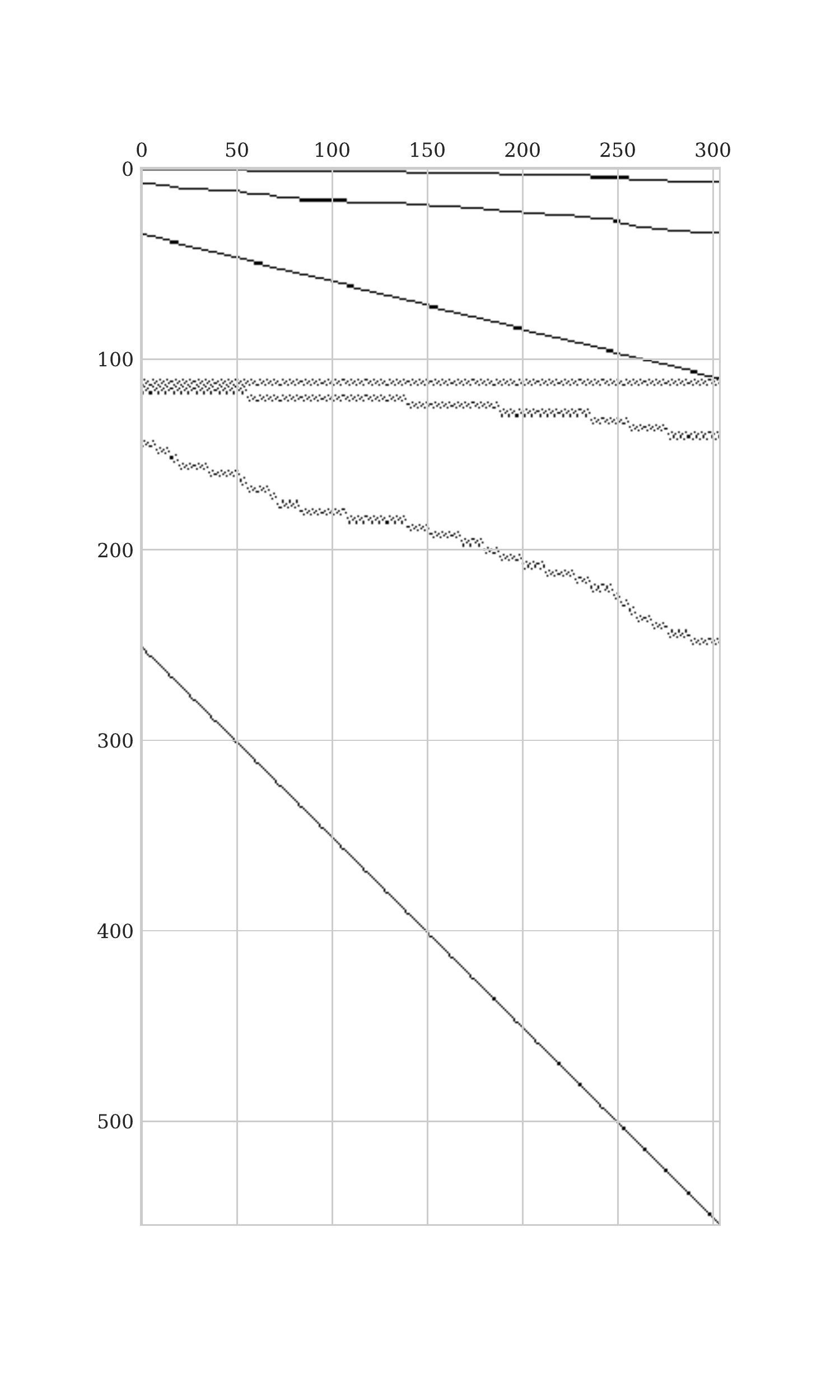}
    }    
    \subfigure[Favorita]{
    \includegraphics[width=0.23\linewidth]{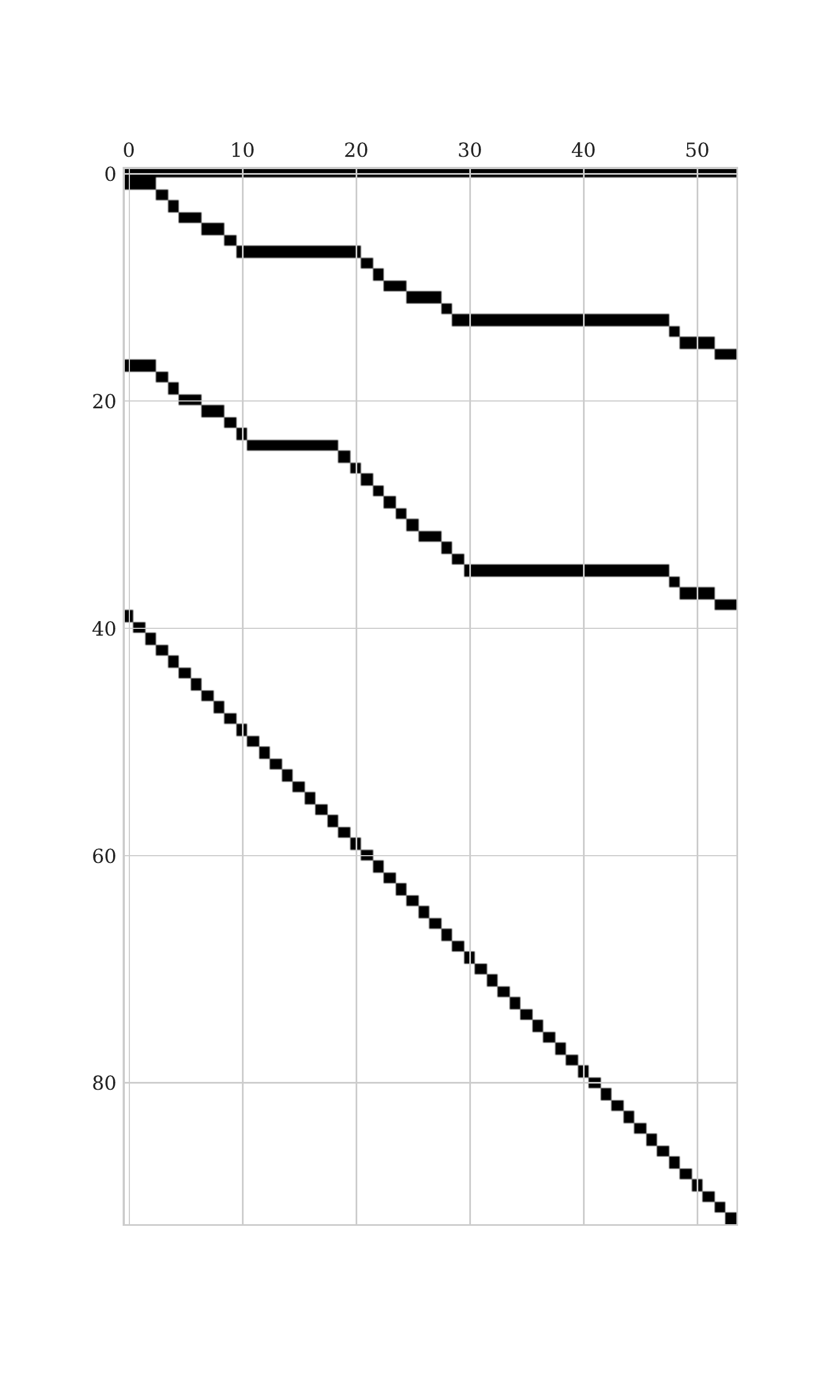}
    \label{fig:favorita}
    }
    \caption{\EDITtwo{Visualization of the hierarchical constraints of the empirical evaluation datasets. (a) \Traffic\ groups 200 highways' occupancy series into quarters, halves and total. (b) \TourismL\ groups its 555 regional visit series, into a combination \EDITthree{of} travel purpose, zones, states and country geographical aggregations. (c) \Favorita\ groups its grocery sales geographically, by store, \EDITthree{city, state, and country levels.}}}
    \label{fig:summation_matrices}
\end{figure}

The \TourismL~\citep{canberra2019tourism} is an Australian Tourism dataset that contains 555 monthly visit series from 1998 to 2016, grouped by \EDITthree{geographical regions} and travel purposes. \Favorita~\citep{favorita2018favorita_dataset} is a Kaggle competition dataset of grocery item sales daily history with additional information on promotions, items, stores, and holidays, containing 371,312 series from January 2013 to August 2017, with a geographic hierarchy of states, cities, and stores. \EDITtwo{We show their hierarchical constraints matrix in Figure~\ref{fig:summation_matrices}.}
\Traffic~\citep{dua2017traffic_dataset,taieb2019hierarchical_regularized_regression} measures the occupancy of \EDITtwo{200} car lanes in the San Francisco Bay Area, \EDITtwo{randomly} grouped into a year of daily observations with 207 series hierarchical structure. 

The datasets provide an opportunity to showcase the broad applicability of the \ours, as each has unique characteristics. \TourismL \ allows us to test the \ours \ to model group structures with multiple hierarchies. \Favorita \ allows us to test the \ours \ on a large-scale dataset. \Traffic \  is composed of randomly assigned hierarchical groupings that may not have any informative structures for the \ours \ to learn with \GroupBU. Finally, \Favorita \ contains some non-count demand values \EDIT{(grocery produce sold by weight)} and \Traffic \ aggregated occupancy rates are non-count data, so modeling these datasets with a Poisson mixture limits the maximum accuracy we can achieve.

\begin{figure}[tbp]
\centering
\subfigure[\ours-\NaiveBU]{\label{fig:a}
\includegraphics[width=78mm]{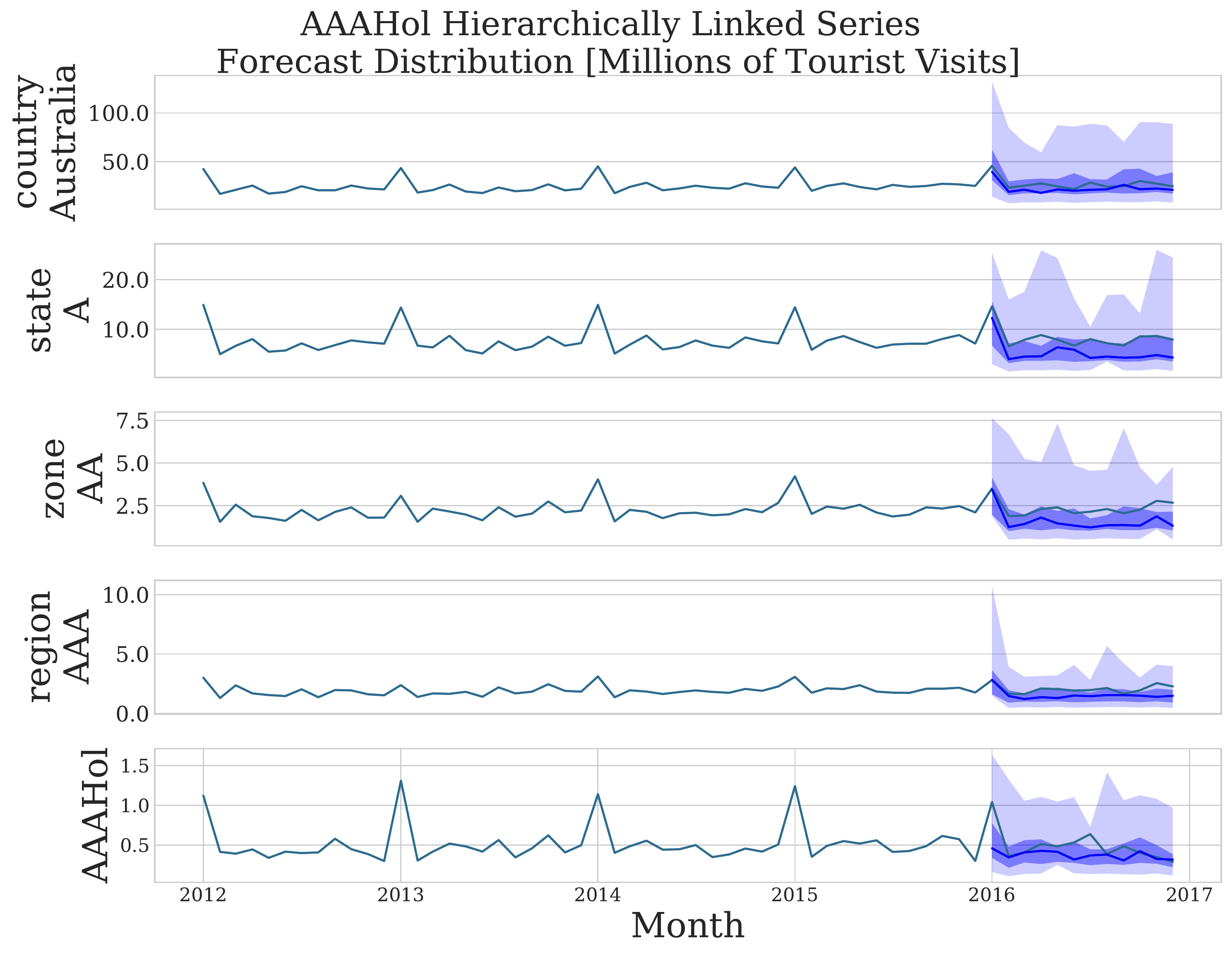}
\label{fig:pmmcnn_naivebu_hierarchical_predictions}
}
\subfigure[\ours-\GroupBU]{\label{fig:b}
\includegraphics[width=78mm]{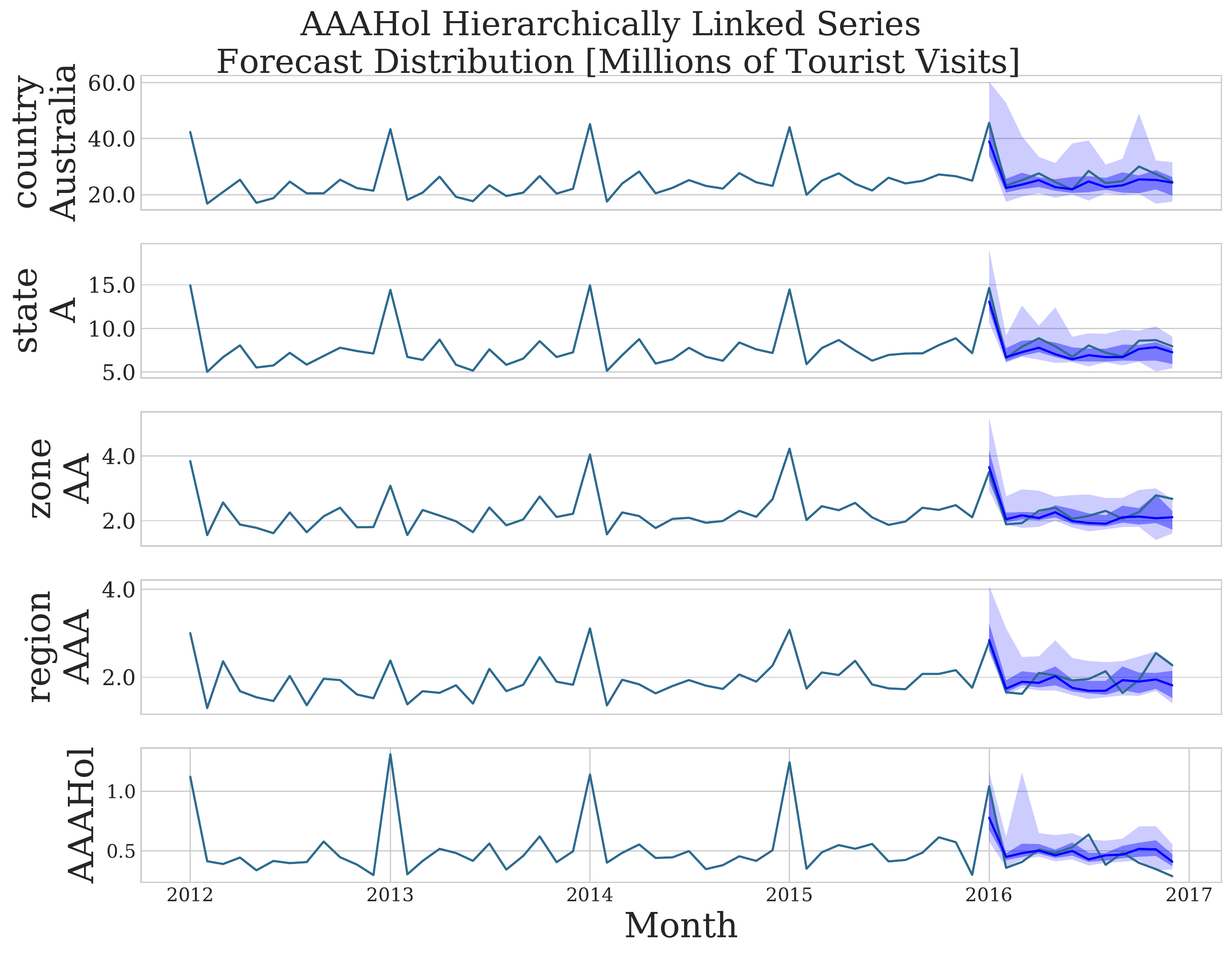}
\label{fig:pmmcnn_groupbu_hierarchical_predictions}
}
\caption{\ours-\NaiveBU\ and \ours-\GroupBU\ forecast distributions on a \TourismL\ hierarchically linked time series.
The top row shows total tourist visits in Australia, the second row shows the visits to Australia for the North South Wales state (A), the third row shows the holiday visits in the metropolitan Area of New South Wales (AA), the fourth row shows the Sidney total visits (AAA), the final row shows the holiday visits to Sidney. Forecast distributions, 99\% and 75\% prediction intervals in light and dark blue.}
\label{fig:forecast_comparison}
\end{figure}

\subsection{Time Series Covariance Modeling}

\EDITthree{We present in this subsection an illustrative example demonstrating how \ours\ leverages the flexibility and expressiveness of the multivariate Poisson mixture distribution to capture interesting correlations present in hierarchical time series datasets to improve forecast sharpness. Figure~\ref{fig:forecast_comparison} shows a comparison of forecasts generated by the \ours-\NaiveBU\ and the \ours-\GroupBU\ methods at various aggregation levels of the \TourismL\ dataset. The \ours-\GroupBU\ method accurately estimates correlations in bottom-level series, improving the forecast distribution concentration of the upper-level series. In contrast, the \ours-\NaiveBU\ method performs well on disaggregated series and mean forecasts, but suffers from significant model misspecification which reduces the sharpness of forecasts at the aggregated levels. We find that the \ours-\GroupBU\ generally does better when informative group series structure is present in the data, like in \TourismL\ and \Favorita\ datasets. \ours-\NaiveBU\ produces comparable results at disaggregated levels of all datasets and outperforms hierarchical forecasting baselines when hierarchical structure is noisy or not informative, as in the \Traffic\ dataset. Our intuitions are validated by the empirical results presented in Section~\ref{subsection:forecasting_results}.}

\begin{figure*}[ht]
\centering
\includegraphics[width=0.95\linewidth]{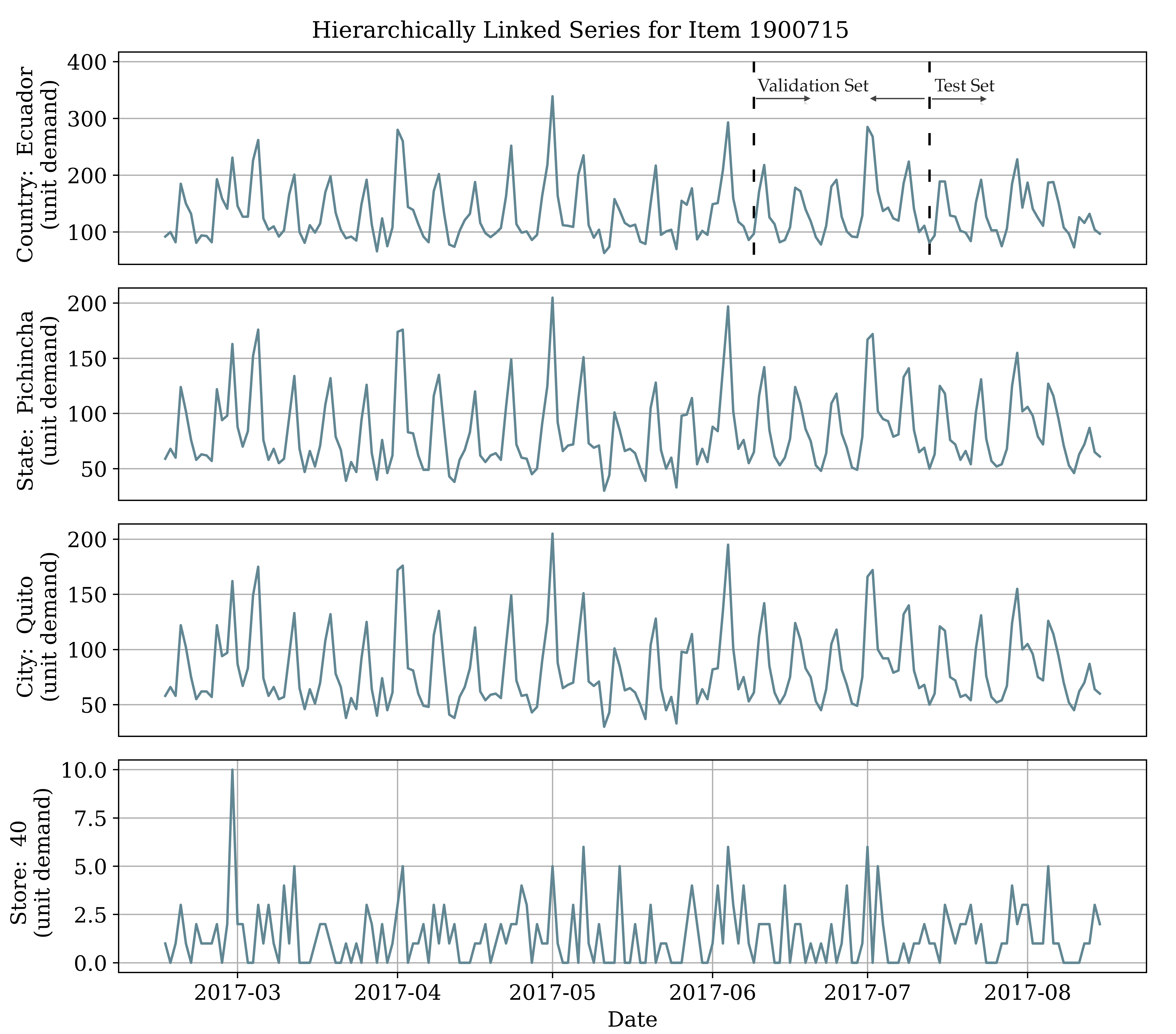} %0.85
\caption{Example of a \Favorita\ geographically linked time series. The top level shows the sales for a grocery item in the country of Ecuador. The second level shows the sold units within Pichincha state, the third level shows the sales for Quito city, the final level shows Store 40 item sales. For this dataset, the training set comprises all the observations preceding the validation and test sets. The validation set (between the first and second dotted lines) is the 34 days before the test set. The held-out test set (marked by the last dotted line) is the last 34 observations.} \label{fig:train_methodology}
\end{figure*}

% \clearpage
\subsection{Datasets Partition and Preprocessing}
\label{section:partition_preprocess}

\EDITtwo{For the main experiments we separate} the train, validation and test datasets' \emph{partition} as follows: we hold out the final horizon-length observations as the \emph{test set}. 
In a sliding-window fashion, we use the horizon-length that \EDITtwo{precedes} the \emph{test set} as the \EDITtwo{\emph{validation set}} and \EDITthree{treat the rest of the past information as the \emph{training set}}. 
A \emph{partition} example is depicted in Figure~\ref{fig:train_methodology}.  

For comparability purposes with the most recent hierarchical forecasting literature, we keep ourselves as close as possible to the preprocessing and wrangling of the datasets to that of \cite{rangapuram2021hierarchical_e2e}\footnote{\EDIT{The pre-processed datasets are available in the hierarchical forecasting extension to the \href{https://github.com/rshyamsundar/gluonts-hierarchical-ICML-2021}{\textcolor{blue}{\GluonTS \ library}.}}}. 
In general, the \emph{static} variables that we consider on all the datasets correspond to the hierarchical and group designators as categorical variables implied by the hierarchical constraint matrix. 
The \emph{temporal} covariates that we consider are the time series for the upper levels of the hierarchy, as well as calendar covariates associated with the time series frequency of each dataset. 
As \emph{future} data, we include calendar covariates to help the \ours \ capture seasonalities.

% $\mathbf{x}^{(s)}_{[b]}$ $\mathbf{x}^{(p)}_{[b][:t]}$ $\mathbf{x}^{(f)}_{[b][t:t+h]}$

\clearpage

%% file: tables/datasets.tex
\begin{table*}[ht] %t
\footnotesize
%	\vskip 0.15in
	\begin{center}
    \caption{Summary, hierarchical structure and forecast horizon of datasets used in our empirical study.}
    \label{table:datasets_summary}
	\begin{sc}
		\begin{tabular}{lcccccc}
			\toprule
			Dataset        & Total   & Aggregated    & Bottom     & Levels   & Observations & Horizon ($h$)\\
			\midrule
			\Traffic       & 207     & 7         & 200             & 4          &   366  &  1      \\
			\TourismL      & 555     & 175       & $76 \,/\, 304$  & $4\,/\, 5$ &   228  & 12      \\
			\Favorita      & 371,312 & 153,386   & 217,944         & 4          & 1,688  & 34      \\
			\bottomrule
		\end{tabular}
	\end{sc}
	\end{center}
	\vskip -0.1in
\end{table*}

%% file: sections/section6_2_quant.tex
%\clearpage

%% file: sections/section6_2_quant_tables.tex
\input{tables/hyperparameters.tex}

\subsection{Evaluation Metrics}

The \EDIT{primary evaluation metric} of the model's \EDITtwo{forecasts} is based on the \emph{quantile loss / pinball loss} (QL)~\citep{matheson1976evaluation_crps}. \EDITthree{For a given a forecast creation date $t$ and horizon indexes $\tau \in [t+1:t+h]$} consider the estimated cumulative distribution function $\hat{F}_{i,\tau}$ of the variable $\EDITthree{Y_{i,\tau}}$ and its observation $y_{i,\tau}$, the loss is defined as: %then the quantile loss is defined by:
\begin{equation}
%\scriptsize 
%\footnotesize
\label{equation:quantile_loss}
    % \mathrm{QL}(\textcolor{blue}{\hat{F}}_{i,\tau}, y_{i,\tau})_{q} = 
    % 2
    % \left(\mathbbm{1}\{y_{i,\tau}\leq\textcolor{blue}{\hat{F}}^{-1}_{i,\tau}(\,q\,)\}-q\right)
    % \left(\textcolor{blue}{\hat{F}}^{-1}_{i,\tau}(\,q\,)-y_{i,\tau}\right)
    \mathrm{QL}(\hat{F}_{i,\tau}, y_{i,\tau})_{q} = 
    %2
    \left(\mathbbm{1}\{y_{i,\tau}\leq\hat{F}^{-1}_{i,\tau}(\,q\,)\}-q\right)
    \left(\hat{F}^{-1}_{i,\tau}(\,q\,)-y_{i,\tau}\right)    
\end{equation}
%\RM{There are constants of 2 in both equation 21 and 22 above, do we need both?}
%\KO{Removed it, thanks for catching this}
We summarize the evaluation, for convenience of exposition and to ensure the comparability of our results with the existing literature, using the \emph{continuous ranked probability score}, abbreviated as CRPS~\citep{matheson1976evaluation_crps}\footnote{In practice the evaluation of the CRPS uses numerical integration technique, that discretizes the quantiles and treats the integral with a left Riemann approximation, averaging over uniformly distanced quantiles.}. 
We use the following mean scaled CRPS~\EDITtwo{\citep{bolin2019scaled_crps,makridakis2022m5_uncertainty}} version:
\EDITthree{
\begin{align}
\label{equation:crps}
\mathrm{CRPS}(\hat{F}_{[i\,],\tau},\mathbf{y}_{[i\,],\tau}) = \frac{2}{|[\,i\,]|} \sum_{i}
    \int^{1}_{0}
    \mathrm{QL}(\hat{F}_{i,\tau}, y_{i,\tau})_{q} dq \\ 
\mathrm{sCRPS}(\hat{F}_{[i\,],\tau},\mathbf{y}_{[i\,],\tau}) = 
    \frac{\mathrm{CRPS}(\hat{F}_{[i\,],\tau},\mathbf{y}_{[i\,],\tau})}{\sum_{i} | y_{i,\tau} |}
\end{align}
}

\EDITtwo{The CRPS measures the forecast distributions' accuracy} and has desirable properties~\EDITtwo{\citep{gneiting2011comparing,bolin2019scaled_crps}}. \EDIT{For instance it is a \emph{proper scoring rule}, since for any \EDITtwo{forecast} distribution $\hat{F}_{i,\tau}$ and true distribution $F_{i,\tau}$, the expected score satisfies:
\begin{equation}
  \mathbb{E}_{Y_{i,\tau}}\left[\mathrm{CRPS}(F_{i,\tau},\EDITthree{Y_{i,\tau}})\right] \leq \mathbb{E}_{Y_{i,\tau}}\left[ \mathrm{CRPS}(\hat{F}_{i,\tau},\EDITthree{Y_{i,\tau}})\right]  
\end{equation}
which implies that it will prefer an ideal probabilistic forecasting system over any other.
}
%\RM{In the inequality above, do we need another expectation over $y_{i,\tau}$? Otherwise, $F_{i,\tau}$ is a point mass at $y_{i,\tau}$}
%\KO{Thanks}

The main focus of the paper is probabilistic coherent forecasting, \EDITthree{and the main results comparing sCRPS of \ours\ to other hierarchical forecasting methods are presented in Section~\ref{subsection:forecasting_results}. We complement the main results with a comparison of mean hierarchical forecasts accuracy in Section~\ref{subsection:complementary_mean_accuracy}. It demonstrates the robustness of our method in point forecasting tasks as well.}

% \clearpage
\subsection{Training Methodology and Hyperparameter Optimization}
\label{section:training_methodology}

\EDITtwo{For the overall hyperparameter selection, we used a standard two-stage approach where we first fixed the architecture, and the estimated probability distribution, and a second stage where we optimized the architecture’s training procedure. Keeping a second stage explored hyperparameter space small  serves two purposes: It keeps space exploration computationally tractable and showcases DPMN’s robustness, broad applicability, and accuracy with minor modifications. We defer some hyperparameter selection details to \ref{section:parameter_details}.}

\EDITtwo{In the first stage we select the number of \ours's mixture components that are responsible for single-series forecasting and modeling bottom-level correlations, as stated in Section~\ref{section:estimation_inference} and shown in \ref{sec:componence_covariance_expressiveness}. For each dataset, we selected the components optimally using temporal cross-validation in \ref{sec:num_components_ablation_study} ablation study, where we found that complex correlation structures favored a higher number of components. To observe the effects of modeling the series covariance, we compared \ours-\GroupBU\ and \ours-\NaiveBU\ variants.}

In the second stage, as shown in Table~\ref{table:hyperparameters}, the hyperparameter space that we consider for optimization is minimal. We only tune the learning rate, random seed to escape underperforming local minima, and the number of SGD epochs as a form of regularization \citep{yuan2007early_stopping}. During the \emph{hyperparameter optimization phase}, we measure the model sCRPS performance on the validation set described in Section~\ref{section:partition_preprocess}, and use \HYPEROPT \ ~\citep{bergstra2011hyperopt}, a Bayesian optimization library, to efficiently explore the hyperparameters based on the validation measurements.

% \EDIT{In addition, \ours \ uses a flexible class of probabilistic model, the finite Poisson mixture to model count data, and the number of components is a hyperparameter specific to \ours \ models. \EDITtwo{Based on properties of the Poisson Mixture covariance matrix shown in \ref{appen:correlation}, we believe the number of mixture components should follow the complexity of the dataset's time series structure, for example, a dataset with complex time series correlations would benefit from fitting \ours \ models with higher number of components}. For each dataset, we use data from validation period to select the optimal number of components for minimizing mean sCRPS in \ours-\GroupBU, and the same number of components is used for both model variants of \ours \ when reporting model performances. sCRPS performance of \ours-\GroupBU \ on validation datasets are included in \ref{sec:num_components_ablation_study}.}

 After the optimal hyperparameters are determined, we estimate the model parameters again by shifting the training window forward, noted as the \emph{retrain phase},  and predict for the final test set. We refer to the combination of the \emph{hyperparameter optimization} and \emph{retrain} phases as a \emph{run}. The \ours \ is implemented using \MXNet~\citep{chen2015mxnet}. To train the network, we minimize the negative log-composite likelihood\ variant from Section~\ref{section:estimation_inference}, using stochastic gradient descent with \emph{Adaptive Moments} (\ADAM; \citealt{kingma2014adam_method}).

% % \clearpage
% \input{tables/crps_results.tex}
\input{tables/crps_results2.tex}

\clearpage
\subsection{Probabilistic Forecasting Results}
\label{subsection:forecasting_results}

%We compare against the following point and probabilistic hierarchical methods, across the hierarchical levels: (1) \NaiveBU \ that produces univariate bottom-level time series forecasts independently and then sums them according to the hierarchical constraints, a distribution is generated using Gaussian assumptions on the errors. (2) \MinT~\citep{wickramasuriya2019hierarchical_mint_reconciliation} \ that reconciles unbiased independent forecasts and minimizes the variance of the forecast errors. (3) \ERM~\citep{taieb2019hierarchical_regularized_regression} that improves on the unbiasedness assumption of the base forecasts in \MinT \ and aims to minimize the bias-variance trade-off of the errors.  (4) \PERMBU~\citep{taieb2017coherent_prob_forecasts} \ that performs a probabilistic hierarchical aggregation and (5) \HierEtoE~\citep{rangapuram2021hierarchical_e2e} that combines a deep vector autoregressive approach with hierarchical constraints\footnote{The \HierEtoE \ benchmark models and experiments are available in a \href{https://github.com/rshyamsundar/gluonts-hierarchical-ICML-2021}{\textcolor{blue}{\GluonTS \ library}} extension.

We compare against the \EDITtwo{forecasts} of the following probabilistic methods across the hierarchical levels: \EDIT{(1) \HierEtoE~\citep{rangapuram2021hierarchical_e2e} that combines DeepVAR with hierarchical constraints on a multivariate normal\footnote{The \HierEtoE\ and \PERMBU-\MinT\  baseline models are available in a \href{https://github.com/rshyamsundar/gluonts-hierarchical-ICML-2021}{\textcolor{blue}{\GluonTS \ library}} extension.}, (2) \PERMBU-\MinT~\footnote{The original \PERMBU-\MinT\ is implemented in \href{https://github.com/bsouhaib/prob-hts}{\textcolor{blue}{supplementary material}} of the work of \cite{taieb2017coherent_prob_forecasts}}~\citep{taieb2017coherent_prob_forecasts}\ that synthesizes hierarchy levels' information with a probabilistic hierarchical aggregation of \ARIMA\ \EDITtwo{forecasts}, (3) automatic \ARIMA~\citep{hyndman2008automatic_arima} that performs a step-wise exploration of ARIMA models using AIC, and (4) \GLMPoisson\ a special case of generalized linear models regression suited for count data \citep{nelder1972generalized_linear}}.

For our proposed methods, we report the \ours-\NaiveBU \ and the \ours-\GroupBU. As described in Section~\ref{section:estimation_inference}, the \ours-\NaiveBU \ treats the bottom level series as independent observations, and the \ours-\GroupBU \ considers groups of time series during its composite likelihood estimation. Both methods obtain probabilistic coherent \EDITtwo{forecasts} using the bottom-up reconciliation. The comparison of the \ours \ variants serves as an ablation experiment to better analyze the source of the accuracy improvements. It also showcases the ability of the Poisson Mixture model to give good results for unseen hierarchical structures, 
and in the case of the \Traffic \ dataset, of uninformative or noisy time-series group structure, to explore the limits of the \GroupBU \ estimation method.

Table~\ref{table:empirical_evaluation} contains the sCRPS measurements for the predictive distributions at each \EDITthree{aggregate} level through the whole dataset hierarchy. The top row reports the overall sCRPS score (averaged across all the hierarchy levels). We highlight the best result in \textbf{bolds}. The \ours \ significantly and consistently improves the overall sCRPS for \TourismL \ and \Favorita. In particular, the \ours-\GroupBU \ variant shows improvements of \EDIT{\TourismLgains} \EDITthree{against} the second-best alternative in the \TourismL \ dataset and \EDIT{\Favoritagains} \EDITthree{against} the second-best choice in the \Favorita \ dataset. 
In the \Traffic \ dataset, the \ours-\GroupBU \ variant does not benefit from modeling the uninformative correlations between highways, and subsequently does not improve upon the other compared methods.
\EDIT{We hypothesize holiday features explain the \Traffic\ New Year's day performance gap between \HierEtoE's and alternative approaches. As neither \ours\ nor other baselines use these features.}
%\RM{[TODO: Rohan to check if this makes sense] CRPS metrics for \Traffic \ dataset was evaluated on quarter-level forecast, where the individual time series to quarter mapping is provided. When fitting \ours-\GroupBU, the model does not have access to the mapping between quarter and highway (unlike alternative models), and was fitted by learning the general correlation across all individual time series. It is possible that this leads to \ours-\GroupBU's inferior forecast accuracy in CRPS metrics.}
%\EDIT{(TO-DO: We hypothesize the grouping of the highways is not informative but noisy.
%, possibly due to the small size of the dataset, condensing too much information and hindering the learning process. 
%Need to investigate further the \ours \  modeling on \Traffic \ data before the IJF submission.)} 
The \ours-\NaiveBU \ variant performs well on \Traffic \ relative to statistical baselines, and gives an acceptable performance on \TourismL \ and \Favorita \ compared to all alternatives.

Our results confirm observations from the community that a shared model, capable of learning from all the time series jointly, improves the \EDITtwo{forecasts} over those from univariate time series methods. Additionally, the qualitative  comparison\footnote{Figure~\ref{fig:pmmcnn_naivebu_hierarchical_predictions} and \ref{fig:pmmcnn_groupbu_hierarchical_predictions} show a qualitative exploration of \ours-\NaiveBU\ and \ours-\GroupBU\ versions.} between the \NaiveBU \ and \GroupBU \ methods shows that an expressive joint distribution framework capable of leveraging the hierarchical structure of the data, when informative, benefits the \EDITtwo{forecasts}' accuracy.

%For \Traffic \ data where the aggregate time series have narrow predictive distributions, \NaiveBU \ is more suitable than \GroupBU, because when predictive distributions are degenerate, the value of predicting accurate bottom-level time series exceeds that of learning correlations between time series. %but when such hierarchical structure is not informative, \NaiveBU\ could still be suitable. 

%\EDIT{In addition to metrics for distributional forecast, we also evaluate mean forecast accuracy among models referenced in this section. Please refer to Appendix \ref{subsection:complementary_mean_accuracy} for detailed discussion on mean forecast accuracy.}

%% file: tables/hyperparameters.tex
\begin{table}[t]
\caption{Considered hyperparameters for the \ourscomplete \ (\ours). The learning rate, random seed, and SGD epochs  that performed best on the validation set were selected automatically in each \HYPEROPT \ run. \EDITtwo{The remaining model parameters were configured once per dataset, as explained in \ref{section:parameter_details}}.\\
\EDIT{\tiny{\textsuperscript{*} 
The Parametrized Exponential Linear Unit (PeLU) modifies the ReLU activation improving the network's training speed \cite{clevert2015elu_activations}.
}}
}
\label{table:hyperparameters}
\tiny
% \scriptsize
%\footnotesize
\centering
    \begin{tabular}{ll}
    \toprule
    \textsc{Hyperparameter}                               & \textsc{Considered Values}           \\ \midrule
    Initial learning rate for SGD optimization.           & $\mathrm{lr}  \in \{0.00001, \dots ,0.01\}$    \\
    SGD full passes to dataset (epochs).                  & $\mathrm{n\_epochs} \in \{10, \dots, 3000\}$  \\
    Random seed that controls initialization of weights.  & $\mathrm{seed\_train} \in \{1, \dots,10\}$    \\ \hline
    SGD Batch Size.                                       & $\mathrm{batch\_size} \in \{4, \dots, 100\}$  \\    
    Activation Function.                                  & PeLU\tiny{\textsuperscript{*}}       \\
    Temporal Convolution Kernel Size.                     & $N_{ck} \in \{2,\,7\}$               \\
    Temporal Convolution Layers.                          & $N_{cl} \in \{3,\,5\}$               \\ 
    Temporal Convolution Filters.                         & $N_{cf} \in \{10,\,30\}$             \\    
    Future Encoder Dimension.                             & $N_{f} \in \{50\}$                   \\
    Static Encoder Dimension.                             & $N_{s} \in \{100\}$                  \\
    Horizon Agnostic Decoder Dimensions.                  & $N_{ag} \in \{50\}$                  \\
    Horizon Specific Decoder Dimensions.                  & $N_{sp} \in \{20\}$                  \\
    Poisson Mixture Weights Decoder Layers.               & $N_{wdl} \in \{3, 4\}$               \\ 
    Poisson Mixture Rate Decoder Layers.                  & $N_{rdl} \in \{2, 3, 4\}$            \\ 
    Local Decoder Dimensions.                             & $N_{k} \in \{25,\,50,\, 100\}$       \\
    \bottomrule
    \end{tabular}
\end{table}

%% file: tables/crps_results2.tex
\begin{table*}[tp]
\tiny
%\scriptsize
%\footnotesize
\centering
\caption{Empirical evaluation of \EDITtwo{probabilistic coherent} forecasts. Mean \emph{scaled continuous ranked probability score} (sCRPS) averaged over 8 runs, at each aggregation level, the best result is highlighted (lower \EDIT{measurements are preferred}). \EDIT{Methods without standard deviation have deterministic solutions.} \\ 
% \EDIT{\HierEtoE \ and \PERMBU \ produces probabilistically coherent distributional forecasts while the remaining ARIMA based models are not, as they are producing point forecasts satisfying aggregation constraints. In this table, we are currently evaluating CRPS by assuming that the point forecast methods produce degenerate distributional forecast with support at the mean forecast only; we did this for our own proof of concept as the most naive baseline.}
%For IJF submission, we will move the point forecast method results from this table that contains CRPS results, to a new table in the appendix, for comparing forecast accuracy of means against probabilistically coherent models.)\\
\tiny{
\EDITtwo{\textsuperscript{*} The \HierEtoE\ results differ from \cite{rangapuram2018deep_state_space}, sCRPS quantile interval space has granularity of 1 percent over its original 5 percent.}\\
\textsuperscript{**} 
\PERMBU-\MinT \ on \TourismL \ is unavailable because the original implementation, currently can't be applied to structures beyond single hierarchies.
}
}
\label{table:empirical_evaluation}
\setlength\tabcolsep{3.0pt}
\begin{tabular}{cc|cccc|cc}
\toprule
\textsc{Dataset} 
& \textsc{Level} & \thead{\ours-\GroupBU \\ (coherent)}    & \thead{\ours-\NaiveBU \\ (coherent)}    & \thead{\HierEtoE\textsuperscript{*} \\ (coherent)}         & \thead{\PERMBU-\MinT\textsuperscript{**} \\ (coherent)} & \thead{\ARIMA \\ (not coherent)} & \thead{\GLMPoisson \\ (not coherent)} \\
\midrule
\multirow{4}{*}{\Traffic}
 & Overall       & \EDIT{$0.0907 \pm 0.0024$}           & $\EDIT{0.0704 \pm 0.0014}$           & \EDITtwo{$\mathbf{0.0375 \pm 0.0058}$}   & $0.0677 \pm 0.0061$                  & \EDIT{0.0751} & \EDIT{0.0771}  \\
 & 1 (geo.)      & \EDIT{$0.0397 \pm 0.0044$}           & $\EDIT{\mathbf{0.0134 \pm 0.0022}}$  & \EDITtwo{$0.0183 \pm 0.0091$}            & $0.0331 \pm 0.0085$                  & \EDIT{0.0376} & \EDIT{0.0063}  \\
 & 2 (geo.)      & \EDIT{$0.0537 \pm 0.0024$}           & $\EDIT{0.0289 \pm 0.0017} $          & \EDITtwo{$\mathbf{0.0183 \pm 0.0081}$}   & $0.0341 \pm 0.0081$                  & \EDIT{0.0412} & \EDIT{0.0194}  \\
 & 3 (geo.)      & \EDIT{$0.0538 \pm 0.0022$}           & $\EDIT{0.0290 \pm 0.0011}$           & \EDITtwo{$\mathbf{0.0209 \pm 0.0071}$}   & $0.0417 \pm 0.0061$                  & \EDIT{0.0549} & \EDIT{0.0406}  \\
 & 4 (geo.)      & \EDIT{$0.2155 \pm 0.0022$}           & $\EDIT{0.2101 \pm 0.0008}$           & \EDITtwo{$\mathbf{0.0974 \pm 0.0021}$}   & $0.1621 \pm 0.0027$                  & \EDIT{0.1665} & \EDIT{0.2420}  \\
\midrule
\multirow{8}{*}{\TourismL}
& Overall        & $\EDIT{\mathbf{0.1249 \pm 0.0020}}$  & $\EDIT{0.1274 \pm 0.0028}$           & \EDITtwo{$0.1472 \pm 0.0029$}            & -                                    & \EDIT{0.1416} & \EDIT{0.1762}  \\
& 1 (geo.)       & $\EDIT{\mathbf{0.0431 \pm 0.0042}}$  & $\EDIT{0.0514 \pm 0.0030}$           & \EDITtwo{$0.0842 \pm 0.0051$}            & -                                    & \EDIT{0.0263} & \EDIT{0.0854}  \\
& 2 (geo.)       & $\EDIT{\mathbf{0.0637 \pm 0.0032}}$  & $\EDIT{0.0705 \pm 0.0026}$           & \EDITtwo{$0.1012 \pm 0.0029$}            & -                                    & \EDIT{0.0904} & \EDIT{0.1153}  \\
& 3 (geo.)       & $\EDIT{0.1084 \pm 0.0033}$           & $\EDIT{\mathbf{0.1068 \pm 0.0019}}$  & \EDITtwo{$0.1317 \pm 0.0022$}            & -                                    & \EDIT{0.1389} & \EDIT{0.1691}  \\
& 4 (geo.)       & $\EDIT{0.1554 \pm 0.0025}$           & $\EDIT{\mathbf{0.1507 \pm 0.0014}}$  & \EDITtwo{$0.1705 \pm 0.0023$}            & -                                    & \EDIT{0.1878} & \EDIT{0.2165}  \\
& 5 (prp.)       & $\EDIT{\mathbf{0.0700 \pm 0.0038}}$  & $\EDIT{0.0907 \pm 0.0061}$           & \EDITtwo{$0.0995 \pm 0.0061$}            & -                                    & \EDIT{0.0770} & \EDIT{0.0954}  \\
& 6 (prp.)       & $\EDIT{\mathbf{0.1070 \pm 0.0023}}$  & $\EDIT{0.1175 \pm 0.0047}$           & \EDITtwo{$0.1336 \pm 0.0042$}            & -                                    & \EDIT{0.1270} & \EDIT{0.1682}  \\
& 7 (prp.)       & $\EDIT{0.1887 \pm 0.0032}$           & $\EDIT{\mathbf{0.1836 \pm 0.0038}}$  & \EDITtwo{$0.1955 \pm 0.0025$}            & -                                    & \EDIT{0.2022} & \EDIT{0.2458}  \\
& 8 (prp.)       & $\EDIT{0.2629 \pm 0.0034}$           & $\EDIT{\mathbf{0.2481 \pm 0.0026}}$  & \EDITtwo{$0.2615 \pm 0.0016$}            & -                                    & \EDIT{0.2834} & \EDIT{0.3134}  \\
\midrule
 \multirow{4}{*}{\Favorita}  
& Overall        & \EDIT{$\mathbf{0.4020 \pm 0.0182}$} & \EDIT{$0.5301 \pm 0.0120$}            & \EDITtwo{$0.5298 \pm 0.0091$}            & \EDIT{$0.4670 \pm 0.0096$}           & \EDIT{0.4373} & \EDIT{0.4524}  \\
& 1 (geo.)       & \EDIT{$0.2760 \pm 0.0149$}          & \EDIT{$0.4166 \pm 0.0195$}            & \EDITtwo{$0.4714 \pm 0.0103$}            & \EDIT{$\mathbf{0.2692 \pm 0.0076}$}  & \EDIT{0.3112} & \EDIT{0.3611}  \\
& 2 (geo.)       & \EDIT{$0.3865 \pm 0.0207$}          & \EDIT{$0.5128 \pm 0.0108$}            & \EDITtwo{$0.5182 \pm 0.0107$}            & \EDIT{$\mathbf{0.3824 \pm 0.0092}$}  & \EDIT{0.4183} & \EDIT{0.4398}  \\
& 3 (geo.)       & \EDIT{$\mathbf{0.4068 \pm 0.0206}$} & \EDIT{$0.5317 \pm 0.0115$}            & \EDITtwo{$0.5291 \pm 0.0129$}            & \EDIT{$0.6838 \pm 0.0108$}           & \EDIT{0.4446} & \EDIT{0.4598}  \\
& 4 (geo.)       & \EDIT{$\mathbf{0.5387 \pm 0.0253}$} & \EDIT{$0.6594 \pm 0.0150$}            & \EDITtwo{$0.6012 \pm 0.0131$}            & \EDIT{$0.5532 \pm 0.0116$}           & \EDIT{0.5749} & \EDIT{0.5490}  \\
\bottomrule
\end{tabular}
\end{table*}

%% file: sections/section6_5_mean_forecast.tex
\input{tables/msse_results.tex}

\subsection{Complementary Mean Forecasting Results}
\label{subsection:complementary_mean_accuracy}
%\KO{Added description of the experiment settings and its results, added \SARIMA, \GLMPoisson\ and \SNaive measurements.}

\EDIT{As shown in Section~\ref{section:pmm}, the \ours's multivariate \EDITthree{Poisson mixture} defines a \EDITtwo{probabilistic coherent} system for its forecast distributions; the mean hierarchical coherence is naturally implied. In this experiment, we compare \ours\ mean hierarchical \EDITtwo{forecasts} (weighted average of Poisson rates) with the following point forecasting methods' \EDITtwo{forecasts}: 
(1) \ARIMA-\ERM\ \citep{taieb2019hierarchical_regularized_regression} \EDITtwo{that} performs an optimization-based reconciliation free of the unbiasedness assumption of the base forecasts, (2) \ARIMA-\MinT\ \citep{wickramasuriya2019hierarchical_mint_reconciliation} meant to reconcile unbiased independent forecasts and minimize the variance of the forecast errors, (3) \ARIMA-\NaiveBU\ \citep{orcutt1968hierarchical_bottom_up} that produces univariate bottom-level time-series \EDITtwo{forecasts} independently and then sums them according to the hierarchical constraints, (4) automatic \ARIMA\ \citep{hyndman2008automatic_arima}, (5) \GLMPoisson\ \citep{nelder1972generalized_linear} (6) and the \SNaive\ model. 

To evaluate we take recommendations from \cite{hyndman2006another_look_measures} and define the \emph{Mean Square Scaled Error} (MSSE) based on the following Equation:}
\begin{equation}
\label{equation:msse}
    % \color{blue}
    \mathrm{MSSE}(\mathbf{y}_{[i],\tau},\; \mathbf{\hat{y}}_{[i],\tau},\; \mathbf{\tilde{y}}_{[i],\tau}) = \frac{\mathrm{MSE}(\mathbf{y}_{[i],\tau},\; \mathbf{\hat{y}}_{[i],\tau})}{\mathrm{MSE}(\mathbf{y}_{[i],\tau},\; \mathbf{\tilde{y}}_{[i],\tau})} 
\end{equation}

\EDIT{where $\mathbf{y}_{[i],\tau},\;\mathbf{\hat{y}}_{[i],\tau},\; \mathbf{\tilde{y}}_{[i],\tau} \in \mathbb{R}^{N \times \EDITthree{h}}$ represent the time series observations, the mean \EDITtwo{forecasts} and the \Naive\ baseline \EDITtwo{forecasts} respectively. %As a relative measurement MSSE removes the data's scale while comparing to the baseline prediction.

Table~\ref{table:empirical_evaluation_msse} contains the MSSE measurements for the predicted means at each aggregation level. The top row reports the overall MSSE (averaged across all the hierarchy levels). We highlight the best result in \textbf{bolds}. \ours\ shows overall improvements or comparable results with the baselines'. With respect to mean hierarchical baselines \ours\ shows 4\% \Traffic\ improvements, 5\% \TourismL\ improvements , and \Favorita\ improvements of 7\%.}
%Although absolute errors and quantile losses are more robust to outliers that have accentuated importance in the MSSE; we opted for this measurement as it better complements the empirical evaluation results from Section~\ref{section:experiments} and verifies the \ours's ability to produce well-performing mean \EDITtwo{forecasts}.}

%% file: tables/msse_results.tex
\begin{table*}[htp]
\tiny
%\scriptsize
%\footnotesize
\centering
\caption{\EDIT{Empirical evaluation of \EDITtwo{mean hierarchical} forecasts. \emph{Mean squared scaled error} (MSSE) averaged over 8 runs, at each aggregation level, the best result is highlighted (lower measurements are preferred). Methods without standard deviation have deterministic solutions.\\
\tiny{\textsuperscript{*} The \ARIMA-\ERM \ results for \TourismL \ differ from \cite{rangapuram2021hierarchical_e2e}, as we improved the numerical stability of their implementation.}}
}
\label{table:empirical_evaluation_msse}
{
% \color{blue}
\setlength\tabcolsep{3.3pt}
\begin{tabular}{cc|ccccc|ccc}
\toprule
\textsc{Dataset} 
& \textsc{Level}     & \thead{\ours-\GroupBU \\ (hier.)}  & \thead{\ours-\NaiveBU \\ (hier.)}      & \thead{\ARIMA-\ERM \textsuperscript{*}  \\ (hier.)}   & \thead{\ARIMA-\MinT-ols \\ (hier.)}   & \thead{\ARIMA-\NaiveBU \\ (hier.)} & \thead{\ARIMA \\ (not hier.)}          & \thead{\GLMPoisson \\ (not hier.)}   & \thead{\SNaive \\ (not hier.)} \\
\midrule
\multirow{4}{*}{\Traffic}   
& Overall            & $0.1750 \pm 0.0099$               & $\mathbf{0.0168 \pm 0.0026}$          & 0.0199                   &  0.0425                       &  0.0217                    &  0.0433                         &  0.0175                         & 0.0709                    \\
& 1 (geo.)           & $0.1619 \pm 0.0099$               & $\mathbf{0.0033 \pm 0.0026}$          & 0.0133                   &  0.0344                       &  0.0168                    &  0.0302                         &  0.0001                         & 0.0547                    \\
& 2 (geo.)           & $0.1835 \pm 0.0101$               & $0.0240 \pm 0.0027 $                  & 0.0135                   &  0.0380                       &  0.0180                    &  0.0392                         &  $\mathbf{0.0109}$              & 0.0676                    \\
& 3 (geo.)           & $0.1819 \pm 0.0100$               & $\mathbf{0.0239 \pm 0.0027}$          & 0.0373                   &  0.0647                       &  0.0295                    &  0.0850                         &  0.0462                         & 0.0989                    \\
& 4 (geo.)           & $0.9964 \pm 0.043$                & $0.9561 \pm 0.0022$                   & 0.6355                   &  0.5876                       &  $\mathbf{0.5669}$         &  $\mathbf{0.5669}$              &  1.2119                         & 1.3118                    \\
\midrule
\multirow{8}{*}{\TourismL}  
& Overall            & $\mathbf{0.1113 \pm 0.0158}$      & $0.2680 \pm 0.0748$                   & 0.1178                   &  0.1251                       &  0.2979                    &  0.1414                           &  0.1944                         &  $0.1306$      \\
& 1 (geo.)           & $0.0597 \pm 0.0212$               & $0.3371 \pm 0.1506$                   & 0.0596                   &  0.0472                       &  0.4002                    &  $\mathbf{0.0343}$                &  0.2015                         &  0.0582                 \\
& 2 (geo.)           & $\mathbf{0.1121 \pm 0.0152}$      & $0.3186 \pm 0.1130$                   & 0.1293                   &  0.1476                       &  0.3340                    &  0.2530                           &  0.2274                         &  0.1628                 \\
& 3 (geo.)           & $\mathbf{0.2250 \pm 0.0196}$      & $0.3909 \pm 0.0822$                   & 0.2529                   &  0.3556                       &  0.4238                    &  0.4429                           &  0.3913                         &  0.3695                 \\
& 4 (geo.)           & $\mathbf{0.2980 \pm 0.0197}$      & $0.4198 \pm 0.0668$                   & 0.3236                   &  0.4288                       &  0.4012                    &  0.4835                           &  0.4238                         &  0.4766                 \\
& 5 (prp.)           & $0.0798 \pm 0.0195$               & $0.1459 \pm 0.0177$                   & 0.0895                   &  0.0856                       &  0.1703                    &  0.0973                           &  0.0961                         &  $\mathbf{0.0615}$      \\
& 6 (prp.)           & $\mathbf{0.1403 \pm 0.0150}$      & $0.1576 \pm 0.0113$                   & 0.1466                   &  0.1537                       &  0.1986                    &  0.1663                           &  0.1840                         &  0.1577                 \\
& 7 (prp.)           & $\mathbf{0.2654 \pm 0.0212}$      & $0.2537 \pm 0.0100$                   & 0.2705                   &  0.3017                       &  0.3151                    &  0.2914                           &  0.3293                         &  0.3699                 \\
& 8 (prp.)           & $0.3302 \pm 0.0235$               & $\mathbf{0.3030 \pm 0.0083}$          & 0.3543                   &  0.3970                       &  0.3769                    &  0.3769                           &  0.3908                         &  0.4969                 \\
\midrule  
\multirow{4}{*}{\Favorita}  
& Overall            & $\mathbf{0.7563 \pm 0.0713}$      & $0.9533 \pm 0.0201$                   & 0.8163                   &  0.9465                       & 0.8276                     &  0.9665                           & 0.8346                          &  1.1420                 \\
& 1 (geo.)           & $\mathbf{0.7944 \pm 0.0568}$      & $0.9188 \pm 0.0187$                   & 0.8362                   &  0.8999                       & 0.8415                     &  0.9217                           & 0.9054                          &  1.1269                 \\
& 2 (geo.)           & $\mathbf{0.7355 \pm 0.1057}$      & $1.0451 \pm 0.0310$                   & 0.7830                   &  1.0057                       & 0.8050                     &  1.0451                           & 0.8037                          &  1.1078                 \\
& 3 (geo.)           & $\mathbf{0.7303 \pm 0.1035}$      & $1.0317 \pm 0.0333$                   & 0.7986                   &  1.0418                       & 0.8192                     &  1.0881                           & 0.8003                          &  1.1315                 \\
& 4 (geo.)           & $\mathbf{0.6770 \pm 0.0351}$      & $0.8090 \pm 0.0180$                   & 0.8199                   &  0.8808                       & 0.8228                     &  0.8228                           & 0.6499                          &  1.2815                 \\
\bottomrule
\end{tabular}
}
\end{table*}

% 0.1312
% 0.0459
% 0.1386
% 0.3231
% 0.4177
% 0.0944
% 0.1773
% 0.3602
% 0.4681

%% file: sections/section7_conclusion.tex
In this work, we have introduced a novel method for \EDITtwo{coherent probabilistic} forecasting, 
%\EDIT{for count data}, 
the \ourscomplete \  (\ours), which focuses on learning the joint distribution of bottom level time series and naturally guarantees hierarchical probabilistic coherence. We have also shown through empirical evaluations that our model is accurate for count data. We observed overall significant improvements in sCRPS when compared with previous state-of-the-art probabilistically coherent models on \EDITthree{two} different hierarchical datasets, Australian domestic tourism (\EDIT{\TourismLgains}) and Ecuadorian grocery sales (\EDIT{\Favoritagains}). \EDIT{However, the model does not show improvement in sCRPS over alternative approaches when evaluated on San Francisco Bay Area traffic data.} 

The framework presented here is also extensible. We chose to focus on forecasting count data and used Poisson kernels, but one could also use Gaussian kernels to model joint distributions of real valued hierarchical data. In fact, any kernel which admits closed form expression for aggregated distributions under conditional independence \EDITtwo{akin to Equation~(\ref{eq:conditional_indep})} will work well, and it includes kernels like the Gamma and the Negative Binomial distributions in addition to the Poisson and the Gaussian distributions already mentioned. 

\EDIT{With respect to the definition of the groups considered in \EDITthree{\ours-\GroupBU, we followed the natural structure of the data and defined them based on } geographic proximity in this work. A promising line of research is an informed creation of such groups based on the series characteristics, for example via clustering.}

By formulating the model as a Mixture Density Network, we have separated the probabilistic model of the predictive distribution from the underlying network, making it compatible with any other archiecture. In the current paper we relied on the convolutional encoder version of the \MQForecaster\, architecture, but significant progress has been made in the last few years on neural network based forecasting models; for example, Transformer-based deep learning architectures \citep{eisenach2020mqtransformer} that can improve performance. We plan to explore both directions, new kernels and new neural network architectures in future work. 

\ours\, has its drawbacks as well. As is the case with any finite mixture model, the fidelity of the estimated distribution depends on the number of mixture components. A few hundred samples may be sufficient to describe a single marginal distribution but can be too sparse to describe the joint distribution in a high dimensional space. The sparsity will be particularly obvious if customers of hierarchical forecasting are interested in forecast distributions conditioned on partially observed data. The small number of samples will lead to overly confident posterior distributions. Another issue is the model misspecification during inference. The \emph{weight matching} method performs quite well in empirical evaluations but is somewhat unsatisfactory as a statistical model. To mitigate both issues we are exploring generative factor models where the mixture components are truly samples from an underlying distribution and correlations between marginal distributions will be captured by common factors. It will bring \ours\, closer to standard Hierarchical Bayesian formulation but with fewer and less strict assumptions.

%% file: sections/section_appendix.tex
\setcounter{table}{0}
\setcounter{figure}{0}
\renewcommand{\thetable}{A\arabic{table}}

% \input{sections/appendix/appendix_pmm_properties.tex}
% \newpage
% \clearpage
\input{sections/appendix/probabilistic_coherence.tex}
% \clearpage
\input{sections/appendix/correlation.tex}
\clearpage
\input{sections/appendix/data.tex}
% \clearpage
\input{sections/appendix/poisson_components.tex}
\clearpage
\input{sections/appendix/hyparameters}
% \clearpage
% \input{sections/appendix/qualitative_predictions.tex}
\clearpage
\input{sections/appendix/pmm_visualizations.tex}

% \clearpage
% \input{sections/appendix/mean_complementary.tex}

% \input{sections/appendix/section_appendix_hyperparameters.tex}
% \input{sections/appendix/section_appendix_ablation.tex}

%% file: sections/appendix/probabilistic_coherence.tex
\EDITtwo{
\section{\ours's Probabilistic Coherence}
\label{sec:probabilistic_coherence}

In this Appendix we prove that \ours's probabilistic coherence. Given access to the joint bottom-level forecast probability $\mathbb{\hat{P}}_{[b]}$ \EDITthree{defined in Equation~\ref{eq:joint_distribution}}, and the \emph{aggregation rule} for $\mathbb{\hat{P}}_{[a]}$ \EDITthree{defined in Equation~\ref{eq:aggregation_rule_concise}}, the implied forecast probability for the hierarchical series $\mathbb{\hat{P}}_{[a,b]}$ is coherent and satisfies Definition~\ref{def:probabilistic_coherence}. The proof first shows that $\mathbb{\hat{P}}_{[a]}$ is well defined, and then shows that \ours's aggregate marginal probability assigns a zero probability to any set that does not contain any coherent forecasts, which implies probabilistic coherence.

\begin{lemma}
\label{def:probabilistic_coherence_lemma2} 
Let $(\Omega_{[b]}, \mathcal{F}_{[b]}, \mathbb{\hat{P}}_{[b]})$ be a probabilistic forecast space, with $\mathcal{F}_{[b]}$ a $\sigma$-algebra on $\Omega_{[b]}$. The aggregation rule defines a probability measure over $\Omega_{[a]} = \mathbf{A}_{[a][b]}(\Omega_{[b]})$:
\begin{equation}
    \mathbb{\hat{P}}_{[a]}(\mathbf{y}_{[a]}) = \int_{\Omega_{[b]}} \mathbb{\hat{P}}_{[b]}(\mathbf{y}_{[b]}) \mathbbm{1} \{\mathbf{y}_{[a]} = \mathbf{A}_{[a][b]} \mathbf{y}_{[b]}\} d \mathbf{y}_{[b]}
\end{equation}
\end{lemma}

\begin{proof}
We prove that $\mathbb{\hat{P}}_{[a]}$ satisfies the Kolmogorov axioms on $(\Omega_{[a]}, \mathcal{F}_{[a]}, \mathbb{\hat{P}}_{[a]})$ with $\Omega_{[a]} = \mathbf{A}_{[a][b]}(\Omega_{[b]})$.
\begin{enumerate}
    \item $\mathbb{\hat{P}}_{[a]}(\mathcal{A}) \geq 0 \quad \forall \mathcal{A} \in \mathcal{F}_{[a]}$ : This follows from the positivity of $\mathbb{\hat{P}}_{[b]}(\mathcal{B})$ and the indicator function.
    
    \item $\mathbb{\hat{P}}_{[a]}(\Omega_{[a]})=1$: The unit measure assumption holds because 
    \begin{align*}
        %\mathbb{\hat{P}}_{[a]}(\Omega_{[a]}) 
        %= 
        \mathbb{\hat{P}}_{[a]}(\mathbf{A}_{[a][b]}(\Omega_{[b]}) 
        = 
        \int_{\Omega_{[a]}} \int_{\Omega_{[b]}} \mathbb{\hat{P}}_{[b]}(\mathbf{y}_{[b]}) \mathbbm{1} \{\mathbf{y}_{[a]} = \mathbf{A}_{[a][b]} \mathbf{y}_{[b]}\} d \mathbf{y}_{[b]} d \mathbf{y}_{[a]} %\\
        =
        \int_{\Omega_{[b]}} \mathbb{\hat{P}}_{[b]}(\mathbf{y}_{[b]}) d \mathbf{y}_{[b]} = 1
    \end{align*}

    \item $\mathbb{\hat{P}}_{[a]}\left(\bigcup _{i=1}^{\infty }\mathcal{A}_{i}\right)=\sum _{i=1}^{\infty }\hat{\mathbb{P}}(\mathcal{A}_{i})$ for disjoint sets $\mathcal{A}_{i}$'s: The $\sigma$-additivity assumption holds
    \begin{align*}
        \hat{\mathbb{P}}_{[a]}\left(\bigcup _{i=1}^{\infty }\mathcal{A}_{i}\right) 
        &=
        \hat{\mathbb{P}}_{[a]}\left(\bigcup _{i=1}^{\infty }\mathbf{A}_{[a][b]}(\mathcal{B}_{i})\right)
        =
        \int \hat{\mathbb{P}}_{[b]}\left(\bigcup _{i=1}^{\infty }\mathcal{B}_{i}\right) \mathbbm{1} \{\mathbf{y}_{[a]} = \mathbf{A}_{[a][b]} \mathbf{y}_{[b]}\} d \mathbf{y}_{[b]} \\
        &=
        \int \hat{\mathbb{P}}_{[b]}\left(\bigcup _{i=1}^{\infty }\mathcal{B}_{i}\right) d \mathbf{y}_{[b]} 
        =
        \sum _{i=1}^{\infty }\hat{\mathbb{P}}(\mathcal{B}_{i})
        =
        \sum _{i=1}^{\infty }\hat{\mathbb{P}}(\mathcal{A}_{i})
    \end{align*}
\end{enumerate}
\end{proof}

\begin{lemma}
\label{def:probabilistic_coherence_lemma} 
Let $(\Omega_{[b]}, \mathcal{F}_{[b]}, \hat{\mathbb{P}}_{[b]})$ be a probabilistic forecast space, with $\mathcal{F}_{[b]}$ a $\sigma$-algebra on $\Omega_{[b]}$. 
If a forecast distribution assigns a zero probability to sets that don't contain coherent forecasts, it defines a coherent probabilistic forecast space $(\Omega_{[a,b]}, \mathcal{F}_{[a,b]}, \hat{\mathbb{P}}_{[a,b]})$ with $\Omega_{[a,b]} = \mathbf{S}_{[a,b][b]}(\Omega_{[b]})$.
\begin{equation}
    \hat{\mathbb{P}}_{[a]}\left(\mathbf{y}_{[a]} \notin \mathbf{A}_{[a][b]}(\mathcal{B})\;|\;\mathcal{B}\right) = 0 \implies
    \hat{\mathbb{P}}_{[a,b]}\left(\mathbf{S}_{[a,b][b]}(\mathcal{B})\right) = \hat{\mathbb{P}}_{[b]}\left(\mathcal{B}\right)
    \quad \forall \mathcal{B} \in \mathcal{F}_{[b]} 
\end{equation}
\end{lemma}

\begin{proof} 
\begin{align*}
    \hat{\mathbb{P}}_{[a,b]}\left(\mathbf{S}_{[a,b][b]} (\mathcal{B}) \right)
    &= 
    \hat{\mathbb{P}}_{[a,b]}\left(
        \begin{bmatrix}
        \mathbf{A}_{[a][b]}\\
        \mathbf{I}_{[b][b]}
        \end{bmatrix}(\mathcal{B})
    \right) =
    \hat{\mathbb{P}}_{[a,b]}\left(
        \{
        \begin{bmatrix}
        \mathbf{A}_{[a][b]}(\mathcal{B})\\
        \mathbb{R}^{N_{b}}
        \end{bmatrix}\} 
        \cap
        \{
        \begin{bmatrix}
        \mathbb{R}^{N_{a}}\\
        \mathcal{B}
        \end{bmatrix}\}         
    \right) \\
    &= 
    \hat{\mathbb{P}}_{[a]}\left(
    \mathbf{A}_{[a][b]}(\mathcal{B})\;|\;\mathcal{B}
    \right)
    \hat{\mathbb{P}}_{[b]}\left(\mathcal{B}\right) =
    (1 - \hat{\mathbb{P}}_{[a]}\left(\mathbf{y}_{[a]} \notin \mathbf{A}_{[a][b]}(\mathcal{B})\;|\;\mathcal{B}\right))\times \hat{\mathbb{P}}_{[b]}\left(\mathcal{B}\right) = 
    \hat{\mathbb{P}}_{[b]}\left(\mathcal{B}\right)
\end{align*}
The first equality is the image of a set $\mathcal{B} \in \Omega_{[b]}$ corresponding the constraints matrix transformation, the second equality defines the spanned space as a subspace intersection of the aggregate series and the bottom series, the third equality uses conditional probability multiplication rule, the final equality uses the zero probability assumption.
\end{proof}

\begin{theorem}
\label{def:probabilistic_coherence_theorem} 
With $(\Omega_{[b]}, \mathcal{F}_{[b]}, \hat{\mathbb{P}}_{[b]})$ probabilistic forecast space, we can construct a coherent probabilistic forecast space $(\Omega_{[a,b]}, \mathcal{F}_{[a,b]}, \hat{\mathbb{P}}_{[a,b]})$ with Lemma \ref{def:probabilistic_coherence_lemma2}'s aggregation.
\end{theorem}

\begin{proof}
It follows from \ref{sec:probabilistic_properties}'s \emph{aggregation rule} that $\hat{\mathbb{P}}_{[a]}\left(\mathbf{y}_{[a]} \notin \mathbf{A}_{[a][b]}(\mathcal{B})\;|\;\mathcal{B}\right) = 0$, using 
Lemma \ref{def:probabilistic_coherence_lemma} we obtain a probabilistic coherent space $(\Omega_{[a,b]}, \mathcal{F}_{[a,b]}, \hat{\mathbb{P}}_{[a,b]})$.
\end{proof}
}

%% file: sections/appendix/correlation.tex
\section{Covariance Formula}\label{appen:correlation}
Here we present the derivation of the covariance formula in Equation~(\ref{eq:covariance}). 

\begin{proof}
Using the law of total covariance, we get 

\begin{equation}
\label{equation:covariance_single}
\mathrm{Cov}(Y_{\beta, \tau}, Y_{\beta', \tau'}) 
=    
\mathbb{E}\left[\mathrm{Cov}(Y_{\beta, \tau}, Y_{\beta', \tau'}\vert \lambda_{\beta,\kappa,\tau}, \lambda_{\beta',\kappa,\tau'} ) \right] + 
\mathrm{Cov}\left(\mathrm{E}\left[Y_{\beta, \tau} \vert \lambda_{\beta,\kappa,\tau}\right], \mathbb{E}\left[Y_{\beta', \tau'} \vert \lambda_{\beta',\kappa,\tau'}\right] \right)
\end{equation}

Using the conditional independence from Equation~(\ref{eq:conditional_indep}). We can rewrite the expectation of the conditional covariance:
\EDITthree{
\begin{eqnarray}
\mathbb{E}\left[\mathrm{Cov}(Y_{\beta,\tau}, Y_{\beta',\tau'}\vert \lambda_{\beta,\kappa,\tau}, \lambda_{\beta',\kappa,\tau'} ) \right]
&=& 
\mathbb{E}\left[\mathrm{Var}(Y_{\beta,\tau}\vert \lambda_{\beta,\kappa,\tau} ) \right]\mathbbm{1}(\beta=\beta')\mathbbm{1}(\tau=\tau') \nonumber \\
&=& \mathbb{E}\left[ \lambda_{\beta,\kappa,\tau}\right]\mathbbm{1}(\beta=\beta')\mathbbm{1}(\tau=\tau') \nonumber \\
&=& \overline{\lambda}_{\beta,\tau} \mathbbm{1}(\beta=\beta')\mathbbm{1}(\tau=\tau')
\end{eqnarray}
}
\EDITthree{
where $\overline{\lambda}_{\beta,\tau} = \sum_{\kappa=1}^{N_k} w_\kappa \lambda_{\beta, \kappa, \tau} $.\\

In the second term, because the conditional distributions are Poisson we have 

$$\mathrm{E}\left[Y_{\beta, \tau} \vert \lambda_{\beta,\kappa,\tau}\right] = \lambda_{\beta,\kappa,\tau} \quad \text{ and } \quad \mathrm{E}\left[Y_{\beta',\tau'} \vert \lambda_{\beta',\kappa,\tau'}\right] = \lambda_{\beta',\kappa,\tau'}$$

Which implies

\begin{equation}
\label{equation:covariance_pair}
\mathrm{Cov}\left(\mathbb{E}\left[Y_{\beta, \tau} \vert \lambda_{\beta,\kappa,\tau}\right], \mathbb{E}\left[Y_{\beta', \tau'} \vert \lambda_{\beta',\kappa,\tau'}\right] \right)
=
\sum_{\kappa=1}^{N_k} w_\kappa 
\left(\lambda_{\beta,\kappa,\tau} - \bar{\lambda}_{\beta,\tau}\right) 
\left( \lambda_{\beta',\kappa,\tau'} - \bar{\lambda}_{\beta',\tau} \right)    
\end{equation}
}

\EDITthree{
Therefore, the covariance formula is:
\begin{equation*}
\mathrm{Cov}(Y_{\beta, \tau}, Y_{\beta', \tau'}) = \overline{\lambda}_{\beta,\tau} \EDITthree{\mathbbm{1}(\beta=\beta')\mathbbm{1}(\tau=\tau')} + \sum_{\kappa=1}^{N_k} w_\kappa \left(\lambda_{\beta,\kappa,\tau} - \overline{\lambda}_{\beta,\tau}\right) \left( \lambda_{\beta',\kappa,\tau'} - \overline{\lambda}_{\beta',\tau'}\right)    
\end{equation*}
}
\end{proof}

\clearpage
\section{Mixture Components and Covariance Matrix Rank}
\label{sec:componence_covariance_expressiveness}
As mentioned in Section~\ref{section:training_methodology}, complex correlations across series benefit from a higher number of mixture components. We ground this intuition on the expressiveness of \ours'\EDITthree{s} Poisson Mixture covariance matrix controlled by its rank. We can show that for \ours'\EDITthree{s finite Poisson mixture} distribution, the bottom level's estimator of the \EDITthree{\emph{non-diagonal}} covariance matrix series is a matrix of rank at most $K-1$ given by:

\begin{equation}
\label{equaton:covariance_multivariate}
    \mathrm{Cov}(\mathbf{y}_{[b],\tau}) = 
    \sum^{K}_{\kappa=1} \mathbf{w}_{\kappa}
    (\blambda_{[b],\kappa,\tau}-\bar{\blambda}_{[b],\tau})
    (\blambda_{[b],\kappa,\tau}-\bar{\blambda}_{[b],\tau})^{\intercal} \in \mathbb{R}^{N_{b} \times N_{b}}
\end{equation}
% \RM{I don't know if this is correct, because in Eq C.1, the actual covariance matrix misses an additive diagonal matrix, and with the diagonal matrix added, should the covariance matrix be full-rank?}\KO{The proof is meant for the non diagonal terms, I added emph}
\begin{proof}
One can easily extend the pair-wise covariance from Equation~(\ref{equation:covariance_pair}) to multivariate covariance Equation~(\ref{equaton:covariance_multivariate}).\\

Let $\mathbf{z}_{\kappa} = \EDITthree{\mathbf{w}_{\kappa}}(\blambda_{[b],\kappa,\tau}-\bar{\blambda}_{[b],\tau})$, by construction we can show that $\sum^{K}_{\kappa=1} \mathbf{z}_{\kappa} = 0$.

Rewriting the last vector $\mathbf{z}_{K} = - \sum^{K-1}_{\kappa=1} \mathbf{z}_{k}$ we obtain a sum of $K-1$ rank-1 matrices
$$
\sum^{K}_{\kappa=1} \EDITthree{\frac{1}{\mathbf{w}_{\kappa}}}\mathbf{z}_{\kappa} \mathbf{z}^{\intercal}_{\kappa} = 
\sum^{K-1}_{\kappa=1} \EDITthree{\frac{1}{\mathbf{w}_{\kappa}}} \mathbf{z}_{\kappa} \mathbf{z}^{\intercal}_{\kappa} 
+ \left(- \sum^{K-1}_{\kappa=1} \mathbf{z}_{\kappa}\right) \EDITthree{\frac{1}{\mathbf{w}_{K}}}\mathbf{z}^{\intercal}_{K}
= \sum^{K-1}_{\kappa=1} \mathbf{z}_{\kappa} (\frac{\mathbf{z}_{\kappa}}{\EDITthree{\mathbf{w}_{\kappa}}}-\frac{\mathbf{z}_{K}}{\EDITthree{\mathbf{w}_{K}}})^{\intercal}
$$

which implies that \ours's modeled covariance matrix rank is upper bounded by $K-1$.
\end{proof}

%% file: sections/appendix/data.tex
%===============================================================
\section{Dataset Details}
\label{section:dataset_details}
%===============================================================

The \Traffic \ dataset, as mentioned, measures the occupancy rates of 963 freeway lanes from the Bay Area. The original data is at a 10-minute frequency from January 1st 2008 to March 30th 2009. The dataset is further aggregated from the 10-minute frequency into daily frequency with 366 observations. We match the sample procedure from previous hierarchical forecasting literature~\citep{taieb2019hierarchical_regularized_regression, rangapuram2021hierarchical_e2e}, and use the same 200 bottom level series from the 963 available. From these 200 bottom level series a hierarchy is randomly defined by aggregating them into quarters and halves of 50 and 100 series each, finally we consider the total aggregation. Table~\ref{table:traffic_data} describes the hierarchical structure.

\begin{table}[h]
\caption{San Francisco Bay Area Highway Traffic.\\
\scriptsize{\textsuperscript{*} The hierarchical structure is randomly defined.}}
\label{table:traffic_data}
\centering
\scriptsize
%\begin{tabular}{lcccc}
\begin{tabular}{lc}
\toprule
\textbf{\makecell{Geographical \\ Level \textsuperscript{*}}} 
& \textbf{\makecell{Series per \\ Level}} \\
%& \textbf{\makecell{Series per \\ Purpose}} & \textbf{Total} &  \\
\midrule
Bay Area    & 1     \\ 
Halves      & 2     \\ 
Quarters    & 4     \\ 
Bottom      & 200   \\ 
\midrule
Total       & 207   \\ 
\bottomrule
\end{tabular}
\end{table}

The \TourismL \ dataset \ contains 555 monthly series from 1998 to 2016, it is organized by geography and purpose of travel. The four-level geographical hierarchy comprises seven states, divided further into 27 zones and 76 regions. The categories for purpose of travel are holiday, visiting friends and relatives, business and other. This dataset has been referenced by important hierarchical forecasting studies like the one of the \MinT \ reconciliation strategy and the more recent \HierEtoE \ \citep{wickramasuriya2019hierarchical_mint_reconciliation, rangapuram2021hierarchical_e2e}. \TourismL \ is a grouped dataset, it has two dimensions in which it is aggregated, the total level aggregation and its four associated purposes. Table~\ref{table:tourisml_data} describes the group and hierarchical structures.

\begin{table}[h]
\caption{Australian Tourism flows.}
\label{table:tourisml_data}
\centering
\scriptsize
\begin{tabular}{lcccc}
\toprule
\textbf{\makecell{Geographical \\ Level}} 
& \textbf{\makecell{Series per \\ Level}} 
& \textbf{\makecell{Series per \\ Level \& Purpose}} & \textbf{Total} &  \\
\midrule
Australia   & 1     & 4    & 5     &  \\
States      & 7     & 28   & 35    &  \\
Zones       & 27    & 108  & 135   &  \\
Regions     & 76    & 304  & 380   &  \\
\midrule
Total       & 111   & 444  & 555   &  \\
\bottomrule
\end{tabular}
\end{table}

\clearpage
The \Favorita \ dataset, once balanced for items and stores, contains 217,944 bottom level series, in contrast the original competition considers 210,645 series. We resort for this balance because, for the moment the \GroupBU \ version of the \PMM \ requires balanced hierarchies for its estimation. In the case of the \Favorita \ experiment we consider a geographical hierarchy (93 nodes) conditional of each of grocery item (4,036). The hierarchy defines 153,368 new aggregate series at the item-country, item-state, and item-city levels. Table~\ref{table:favorita_data} describes the structure. 

Regarding the dataset preprocessing, we confirmed observations from the best submissions to the \href{https://www.kaggle.com/c/favorita-grocery-sales-forecasting/discussion/47582}{Kaggle competition}. Most holiday distances included in the dataset and covariates like oil production lack value for the \EDITtwo{forecasts}. The models did not benefit from a long history, filtering the training window to the 2017 year consistently produced better results.

% \EDIT{If we were to treat the hierarchy of this dataset \say{naively}, we would require to define the hierarchical constraints matrix $\mathbf{H} \in \mathbb{R}^{(N_{a}+N_{b}) \times N_{b}}$. In practice, the number of entries such matrix would surpass $3.4 \times 10^{10}$. Our method, in contrast to \HierEtoE, can fit hierarchical substructures since the \PMM \ composite likelihood approach allows us to model the interactions of groups as conditionally independent from one another, this would be analogous to operate on efficient sparsity assumptions for the constraints matrix.}

\begin{table}[ht]
\caption{Favorita Grocery Sales.}
\label{table:favorita_data}
\centering
\scriptsize
\begin{tabular}{lcccc}
\toprule
\textbf{\makecell{Geographical \\ Level}} 
& \textbf{\makecell{Nodes per \\ Level}} 
& \textbf{\makecell{Series per \\ Level}} & \textbf{Total} &  \\
\midrule
Ecuador     &  1     & 4,036       & 4,036     &  \\
States      & 16     & 64,576      & 64,576    &  \\
Cities      & 22     & 88,792      & 88,792    &  \\
Stores      & 54     & 217,944     & 217,944   &  \\
\midrule
Total       & 93     & 371,312     & 371,312   &  \\
\bottomrule
\end{tabular}
\end{table}

% \begin{figure}[!ht]
%     \centering
%     \subfigure[Traffic]{
%     \label{fig:traffic}
%     \includegraphics[width=0.25\linewidth]{images/TrafficS_Hmatrix.pdf}
%     }
%     \subfigure[Tourism]{
%     \label{fig:tourism}
%     \includegraphics[width=0.25\linewidth]{images/TourismL_Hmatrix.pdf}
%     }    
%     \subfigure[Favorita]{
%     \includegraphics[width=0.25\linewidth]{images/Favorita_Hmatrix.pdf}
%     \label{fig:favorita}
%     }
%     \caption{\EDIT{Visualization of the hierarchical constraints of the datasets used in the empirical evaluation. (a) \Traffic\ groups 200 highways' occupancy series into quarters, halves and total. (b) \TourismL\ groups its 555 regional visit series, into a combination travel purpose, zones, states and country geographical aggregations. (c) \Favorita\ groups its grocery sales geographically, by store, cities, states and country levels.}}
%     \label{fig:motivation}
% \end{figure}

%% file: sections/appendix/poisson_components.tex
%=======================================================================
\section{Poisson Mixture Size Ablation Study}
\label{sec:num_components_ablation_study}
%=======================================================================

\EDITtwo{As shown in Section~\ref{appen:correlation}, the job of \ours' mixture distribution goes beyond forecasting single time series but also modeling correlations across them too. Based on the covariance matrix expressiveness theoretical properties, it is reasonable to expect that the number of optimal components grows with the complexity of the modeled time series structure. 

In this ablation study, we empirically test these intuitions showing how the number of optimally selected \ours's mixture components grows with larger datasets. For the experiment, we measured the cross-validation performance of \ours' configurations as defined in Table~\ref{table:hyperparameters} explored automatically with \HYPEROPT~\citep{bergstra2011hyperopt} where we fix the number of mixture components.
}

\EDIT{
\begin{table*}[htp]
\scriptsize
%\footnotesize
\centering
\caption{\EDIT{Empirical evaluation of probabilistic coherent forecasts for different \ours-\GroupBU, varying the Poisson mixture size. Mean overall \emph{scaled continuous ranked probability score} (sCRPS). The best result is highlighted (lower measurements are preferred).}}
\label{table:empirical_evaluation_num_components}
{
% \color{blue}
\begin{tabular}{cc|ccccc}
\toprule
\textsc{Dataset} 
& \textsc{Level}     & $K=1$ & $K=10$ & $K=25$ & $K=50$ & $K=100$ \\
\midrule
\multirow{1}{*}{\Traffic}   
& Overall            & $0.1647 \pm 0.0009$               & $0.1435 \pm 0.0947$          & $\mathbf{0.0958\pm 0.0005}$      &  $0.1337\pm 0.0004$       &  $0.1261\pm 0.0037$                    \\
\midrule
\multirow{1}{*}{\TourismL}  
& Overall            & $0.1673 \pm 0.0052$               & $0.1380 \pm 0.0017$          & $\mathbf{0.1247 \pm 0.0025}$     &  $0.1284 \pm 0.0031$      &  $0.1251 \pm 0.0034$                   \\
\midrule  
\multirow{1}{*}{\Favorita}  
& Overall            &   $0.8390 \pm 0.0124$             & $0.4416 \pm 0.0152$          & $--$                             & $0.4204 \pm 0.0108$       &  $\mathbf{0.3758 \pm 0.0040}$          \\
\bottomrule
\end{tabular}
}
\end{table*}

Table~\ref{table:empirical_evaluation_num_components} reports the validation probabilistic forecast accuracy measured with sCRPS, across \Traffic, \TourismL, and \Favorita, for \ours-\GroupBU\ with different number of Poisson mixture components. For this experiment, we report the overall validation sCRPS averaged over four independent \HYPEROPT\ runs with twelve optimization steps and eight steps in the \Traffic\ dataset.

Table~\ref{table:empirical_evaluation_num_components}'s sCRPS measurements suggest that there is a \emph{bias-variance trade-off} controlled by the Poisson mixture size. When $K=1$, \ours-\GroupBU \ model corresponds to Poisson regression and treats each series as probabilistically independent, such model high-bias simple model produced \EDITtwo{forecasts} with the worst sCRPS. The \EDITtwo{forecast} accuracy improves as the number of Poisson components increases from $K=1$, but the accuracy begins to deteriorate beyond a certain threshold. We hypothesize that a small number of mixture components does not have enough degrees of freedom to describe the data, and too many mixture components lead to over-fitting the training data, resulting in large variance on the validation data.
% We hypothesize that when $K$ is sufficiently large, the forecast accuracy will reach its optimal point, and when $K$ is further increased, the accuracy will be degraded, \textcolor{red}{possibly due to the non-identifiability of finite Poisson models.}

We observed that the precise value of an optimal Poisson mixture components varies across the datasets. Larger datasets, or datasets with a complex time series correlation structure, appear to benefit from more flexible probability mixtures. \Traffic, our smallest dataset, produced optimal results with $K=25$ components, \TourismL, a medium-sized dataset, produced optimal results with $K=25$ components. Finally \Favorita\, our largest dataset, did not saturate even with the largest number of components we experimented with; we capped the choice of the number of mixture components at $K=100$ due to GPU memory constraints.

% we think this is due because of its relatively weak correlation between time series (as implied from the small gap of CRPS performance between \ours-\GroupBU \ and \ours-\NaiveBU\ in Table~\ref{table:empirical_evaluation})

\begin{figure*}[ht]
\centering
\includegraphics[width=0.65\linewidth]{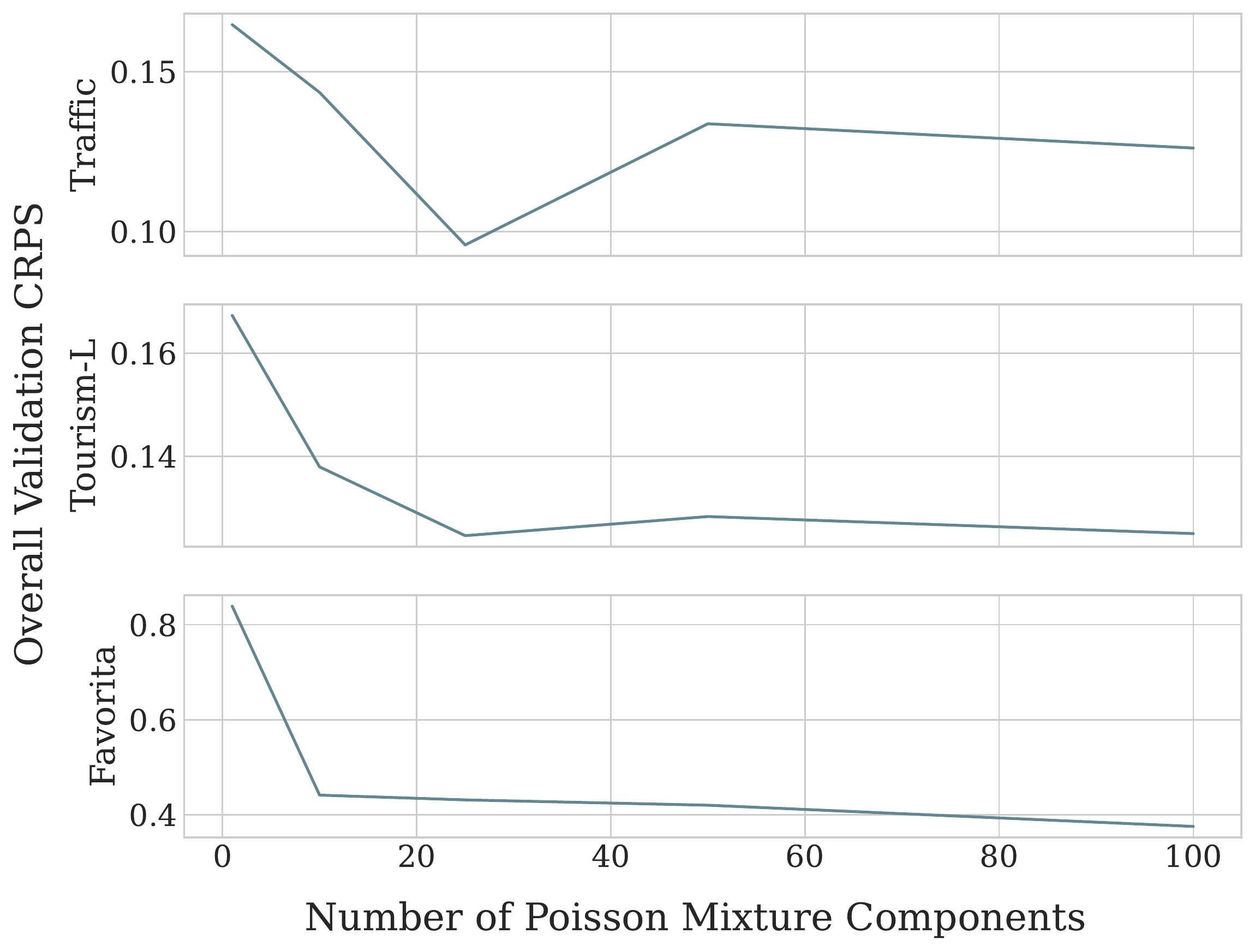}
\caption{\EDIT{Poisson Mixture size ablation study. We found interesting bias-variance trade-offs controlled by the number of mixture components, both \Traffic\ and \TourismL\ have an optimal value of 25 components, beyond which the sCRPS validation performance worsens. We observed a classic U-shaped pattern. In the case of \Favorita, the largest dataset, the validation sCRPS continued to improve through $K=100$.}} \label{fig:ablation_study_crps_vs_k}
\end{figure*}
}

% ablation_df = pd.DataFrame({'Components': [1, 10, 25, 50, 100],
%                             'Traffic': [0.1647, 0.1435, 0.0958, 0.1337, 0.1261],
%                             'Tourism-L': [0.1673, 0.1380, 0.1247, 0.1284, 0.1251],
%                             'Favorita': [0.8390, 0.4416, 0.4316, 0.4204, 0.3758]})

% ablerror_df = pd.DataFrame({'Components': [1, 10, 25, 50, 100],
%                             'Traffic': [0.0009, 0.0947, 0.0005, 0.004, 0.0037],
%                             'Tourism-L': [0.0052, 0.0017, 0.0025, 0.0031, 0.0034],
%                             'Favorita': [0.0124, 0.0152, np.nan, 0.0108, 0.0040]})
                            
% fig, axs = plt.subplots(nrows=3, ncols=1, sharex=True, figsize=(10, 8))

% axs[0].plot(ablation_df['Components'].values, ablation_df['Traffic'].values,
%           color='#628793')
% axs[0].set_ylabel('Traffic', fontsize=18)

% axs[1].plot(ablation_df['Components'].values, ablation_df['Tourism-L'].values,
%           color='#628793')
% axs[1].set_ylabel('Tourism-L', fontsize=18)

% axs[2].plot(ablation_df['Components'].values, ablation_df['Favorita'].values,
%           color='#628793')
% axs[2].set_ylabel('Favorita', fontsize=18)

% fig.text(0.5, 0.04, 'Number of Poisson Mixture Components', ha='center', fontsize=20)
% fig.text(0.01, 0.5, 'Overall Validation CRPS ', va='center', rotation='vertical', fontsize=20)

% plt.savefig('./results/DPMN_ablation.pdf', bbox_inches = 'tight')
% plt.show()
% plt.close()

%% file: sections/appendix/hyparameters.tex
%===============================================================
\section{Model Parameter Details}
\label{section:parameter_details}
%===============================================================

\input{tables/appendix_hyperparameters}

\EDITtwo{As mentioned in Section~\ref{section:training_methodology}, for the overall hyperparameter selection, we used a standard two-stage approach where we fixed the architecture, the probability distribution to estimate (and implicit training loss), and a second stage where we optimized the architecture's training procedure.}

\EDITtwo{In the first stage, we carefully fixed the architecture and the probability distribution to estimate. The most important  heuristic guiding this selection was to increase the architecture's and probability's capacity for larger or more complex datasets.} To increase the network's capacity, we increased the number of convolution layers $N_{cl}$ and convolution filters $N_{p}$ as well as the mixture weight and rate decoder layers $N_{wdl}, N_{rdl}$. In particular, since \Traffic \ is the smallest dataset, we opted for a reduced model size and the number of Poisson mixture components to control the model's variability. Additionally, due to the dataset's strong weekly seasonality pattern, we adjusted the convolution kernel size to encompass seven days. \EDITtwo{We control the probability's capacity with the mixture size and SGD batchsize. In an ablation study similar to \ref{sec:num_components_ablation_study}} we found that for datasets with strong correlations like \Favorita \ and \TourismL\, maximizing the batch size with respect to GPU memory limitations resulted in better validation performance; for \Traffic, though the entire dataset could fit in memory at once, it was preferable to feed in subsets to allow the model to learn from different randomly sampled highway groups in each epoch. %\EDITtwo{The relationship between the batch size and the dataset's correlations is similar to the relationship between the size of the multivariate mixtures and the dataset's correlations reported in the \ref{sec:num_components_ablation_study}.}

For the second hyperparameter selection stage, as reported in Section~\ref{section:training_methodology}, for each fixed architecture and probability we optimally explored its training procedure hyperparameters defined in Table~\ref{table:hyperparameters} using \HYPEROPT\ algorithm's Bayesian optimization  \citep{bergstra2011hyperopt}. The second phase only considers the optimal exploration of the learning rate, random initialization, and the number of SGD epochs, the selection is guided by temporal cross-validation signal obtained from the dataset partition introduced in Section~\ref{section:partition_preprocess}.

%\RM{Changed temporal conv. kernel size notation $N_k\to N_{ck}$}

%% file: tables/appendix_hyperparameters.tex
\begin{table} [h]
\caption{\ourscomplete \ (\ours) architecture parameters configured once per dataset. \EDITtwo{These hyperparameters correspond to the first selection phase preceding the automatic optimization.\\
\tiny{\textsuperscript{*} 
SGD batch selection follows an ablation study considering values between \{2,4,8,16,32,64,100\}. We report the best validation batch size.
}
}}
\label{table:model_parameters}
%\tiny
\scriptsize
%\footnotesize
\centering
    \begin{tabular}{lcccc}
    \toprule
    \textsc{Parameter}                                    & Notation          & \multicolumn{3}{c}{Considered Values} \\
                                                          &                   & \textsc{\Traffic} & \textsc{\TourismL} & \textsc{\Favorita}  \\
    \midrule                                                                                                                                  
    SGD Batch Size\EDITtwo{\textsuperscript{*}} .                   & -                 & 4                 & 4                  & 100                 \\
    Activation Function.                                  & -                 & PeLU              & PeLU               & PeLU                \\
    Temporal Convolution Kernel Size.                     & $N_{ck}$          & $\{7\}$           & $\{2\}$            & $\{2\}$             \\
    Temporal Convolution Layers.                          & $N_{cl}$          & $\{3\}$           & $\{5\}$            & $\{5\}$             \\ 
    Temporal Convolution Filters.                         & $N_{p}$           & $\{10\}$          & $\{30\}$           & $\{30\}$            \\    
    Future Encoder Dimension.                             & $N_{f}$           & $\{50\}$          & $\{50\}$           & $\{50\}$            \\
    Static Encoder Dimension.                             & $N_{s}$           & $\{100\}$         & $\{100\}$          & $\{100\}$           \\
    Horizon Agnostic Decoder Dimensions.                  & $N_{ag}$          & $\{50\}$          & $\{50\}$           & $\{50\}$            \\
    Horizon Specific Decoder Dimensions.                  & $N_{sp}$          & $\{20\}$          & $\{20\}$           & $\{20\}$            \\
    Poisson Mixture Weight Decoder Layers.                & $N_{wdl}$         & $\{3\}$           & $\{4\}$            & $\{4\}$             \\ 
    Poisson Mixture Rate Decoder Layers.                  & $N_{rdl}$         & $\{2\}$           & $\{3\}$            & $\{3\}$             \\ 
    Poisson Mixture Components.                           & $N_{k}$           & $\{25\}$          & $\{25\}$          & $\{100\}$           \\
    GPU Training Configuration.                           & -                 & 2 x NVIDIA V100   & 2 x NVIDIA V100     & 4 x NVIDIA V100     \\
    \bottomrule
    \end{tabular}
\end{table}

%% file: sections/appendix/pmm_visualizations.tex
%========================================================================
\onecolumn
\section{Qualitative Analysis of Poisson Mixture Rates}
\label{section:mixture_rates_composition}
%=======================================================================

\vspace{5mm}
\EDIT{
We use a histogram in Figure~\ref{fig:mixture_rates_composition} to visualize the distribution of Poisson Mixture rates learned by the \ours-\GroupBU\ method on the \Traffic, \TourismL, and \Favorita\ datasets. These rates correspond to those of the models from Table~\ref{table:empirical_evaluation}, for a base time series and a single time step ahead in the forecast horizon. In the histogram, we use the learned mixture weights $w_{k}$ in the vertical axis to account for the probability of each Poisson rate $\lambda_{k}$ represented in the horizontal axis.

Exciting patterns arise from visualizing the learned Poisson Mixture rates and their weights. First, the \ours\ likelihood is capable of modeling zero-inflated data in the case of the \EDITtwo{\Favorita\ bottom} level series; as we can see from the probability accumulation around zero, this may explain its superior performance on \Favorita\ bottom level series that tends to be sparse as previously shown in Figure~\ref{fig:train_methodology}. Second, all the Poisson rate distributions show multiple modes, a quality that cannot be replicated by a Gaussian, uni-modal, and symmetric distribution. This observation further motivates using a flexible distribution capable of better modeling the underlying data generation processes. Third, the distribution of Poisson rates in the \TourismL\ dataset can overcome the limitations of simple Poisson distributions that tend to have collapsed variances for aggregated series due to their scale since the variance of the series can still be correctly modeled by the variance of the Poisson rates.

\begin{figure*}[ht]
\centering
\includegraphics[width=0.65\linewidth]{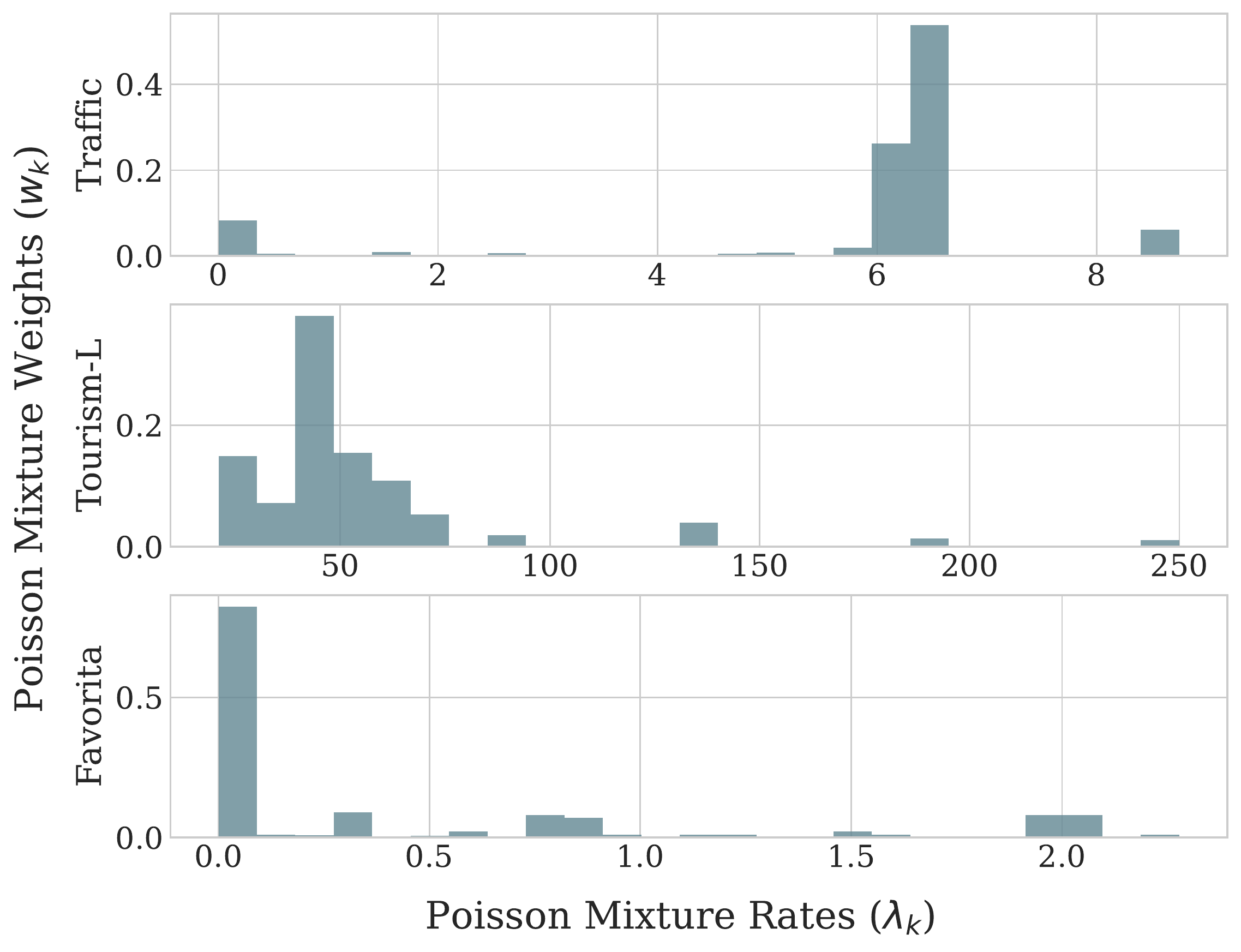}
\caption{\EDIT{Poisson Mixture rates distribution for the \ours-\GroupBU\ model reported in Table~\ref{table:empirical_evaluation}. A single bottom-level series and a time step ahead in the forecast horizon is considered. The mixture distribution is capable of flexibly modeling multimodal processes, overcome Poisson regression's limitations for aggregated data and model zero inflated processes for disaggregated data. These qualities make it exceptionally useful for hierarchical forecasting tasks.} } \label{fig:mixture_rates_composition}
\end{figure*}
}

% traffic_lams = pd.read_csv('./data/traffic_lam_wgt.csv')
% tourism_lams = pd.read_csv('./data/tourismL_lam_wgt.csv')
% favorita_lams = pd.read_csv('./data/favorita_lam_wgt.csv')

% fig, axs = plt.subplots(nrows=3, ncols=1, figsize=(10, 8))

% axs[0].hist(x=traffic_lams.Lambda.values, weights=traffic_lams.Weight.values,
%             alpha=0.8, bins=25, color='#628793')
% axs[0].set_ylabel('Traffic', fontsize=18)

% axs[1].hist(x=tourism_lams.Lambda.values, weights=tourism_lams.Weight.values,
%             alpha=0.8, bins=25, color='#628793')
% axs[1].set_ylabel('Tourism-L', fontsize=18)

% axs[2].hist(x=favorita_lams.Lambda.values, weights=favorita_lams.Weight.values,
%             alpha=0.8, bins=25, color='#628793')
% axs[2].set_ylabel('Favorita', fontsize=18)

% fig.text(0.5, 0.04, r'Poisson Mixture Rates ($\lambda_{k}$)', ha='center', fontsize=20)
% fig.text(0.01, 0.5, r'Poisson Mixture Weights ($w_{k}$)', va='center', rotation='vertical', fontsize=20)

% plt.savefig('./results/Poisson_rates.pdf', bbox_inches = 'tight')
% plt.show()
% plt.close()
% plt.show()

%% file: main.bbl
\begin{thebibliography}{65}
\expandafter\ifx\csname natexlab\endcsname\relax\def\natexlab#1{#1}\fi
\providecommand{\url}[1]{\texttt{#1}}
\providecommand{\href}[2]{#2}
\providecommand{\path}[1]{#1}
\providecommand{\DOIprefix}{doi:}
\providecommand{\ArXivprefix}{arXiv:}
\providecommand{\URLprefix}{URL: }
\providecommand{\Pubmedprefix}{pmid:}
\providecommand{\doi}[1]{\href{http://dx.doi.org/#1}{\path{#1}}}
\providecommand{\Pubmed}[1]{\href{pmid:#1}{\path{#1}}}
\providecommand{\bibinfo}[2]{#2}
\ifx\xfnm\relax \def\xfnm[#1]{\unskip,\space#1}\fi
%Type = Article
\bibitem[{Amir \& Souhaib(2016)}]{atiya2016multi_step_forecasting}
\bibinfo{author}{Amir, A.}, \& \bibinfo{author}{Souhaib, B.}
  (\bibinfo{year}{2016}).
\newblock \bibinfo{title}{A bias and variance analysis for multistep-ahead time
  series forecasting}.
\newblock {\it \bibinfo{journal}{IEEE transactions on neural networks and
  learning systems}\/},  {\it \bibinfo{volume}{27}\/},
  \bibinfo{pages}{2162--2388}. \URLprefix
  \url{https://pubmed.ncbi.nlm.nih.gov/25807572/}.
%Type = Article
\bibitem[{Athanasopoulos et~al.(2017)Athanasopoulos, Hyndman, Kourentzes \&
  Petropoulos}]{athanasopoulos2017hierarchical_temporal}
\bibinfo{author}{Athanasopoulos, G.}, \bibinfo{author}{Hyndman, R.~J.},
  \bibinfo{author}{Kourentzes, N.}, \& \bibinfo{author}{Petropoulos, F.}
  (\bibinfo{year}{2017}).
\newblock \bibinfo{title}{Forecasting with temporal hierarchies}.
\newblock {\it \bibinfo{journal}{European Journal of Operational Research}\/},
  {\it \bibinfo{volume}{262}\/}, \bibinfo{pages}{60--74}.
%Type = Inproceedings
\bibitem[{{Ben Taieb} et~al.(2017){Ben Taieb}, Taylor \&
  Hyndman}]{taieb2017coherent_prob_forecasts}
\bibinfo{author}{{Ben Taieb}, S.}, \bibinfo{author}{Taylor, J.~W.}, \&
  \bibinfo{author}{Hyndman, R.~J.} (\bibinfo{year}{2017}).
\newblock \bibinfo{title}{Coherent probabilistic forecasts for hierarchical
  time series}.
\newblock In \bibinfo{editor}{D.~Precup}, \& \bibinfo{editor}{Y.~W. Teh}
  (Eds.), {\it \bibinfo{booktitle}{Proceedings of the 34th International
  Conference on Machine Learning}\/} (pp. \bibinfo{pages}{3348--3357}).
\newblock \bibinfo{publisher}{PMLR} volume~\bibinfo{volume}{70} of {\it
  \bibinfo{series}{Proceedings of Machine Learning Research}\/}.
\newblock \URLprefix \url{http://proceedings.mlr.press/v70/taieb17a.html}.
%Type = Article
\bibitem[{{Ben Taieb} et~al.(2021){Ben Taieb}, Taylor \&
  Hyndman}]{taieb2021hierarchical_electricity}
\bibinfo{author}{{Ben Taieb}, S.}, \bibinfo{author}{Taylor, J.~W.}, \&
  \bibinfo{author}{Hyndman, R.~J.} (\bibinfo{year}{2021}).
\newblock \bibinfo{title}{Hierarchical probabilistic forecasting of electricity
  demand with smart meter data}.
\newblock {\it \bibinfo{journal}{Journal of the American Statistical
  Association}\/},  {\it \bibinfo{volume}{116}\/}, \bibinfo{pages}{27--43}.
  \URLprefix \url{https://doi.org/10.1080/01621459.2020.1736081}.
  \DOIprefix\doi{10.1080/01621459.2020.1736081}.
  \href{http://arxiv.org/abs/https://doi.org/10.1080/01621459.2020.1736081}{\tt
  arXiv:https://doi.org/10.1080/01621459.2020.1736081}.
%Type = Inproceedings
\bibitem[{Bergstra et~al.(2011)Bergstra, Bardenet, Bengio \&
  K\'{e}gl}]{bergstra2011hyperopt}
\bibinfo{author}{Bergstra, J.}, \bibinfo{author}{Bardenet, R.},
  \bibinfo{author}{Bengio, Y.}, \& \bibinfo{author}{K\'{e}gl, B.}
  (\bibinfo{year}{2011}).
\newblock \bibinfo{title}{Algorithms for hyper-parameter optimization}.
\newblock In \bibinfo{editor}{J.~Shawe-Taylor}, \bibinfo{editor}{R.~Zemel},
  \bibinfo{editor}{P.~Bartlett}, \bibinfo{editor}{F.~Pereira}, \&
  \bibinfo{editor}{K.~Q. Weinberger} (Eds.), {\it \bibinfo{booktitle}{Advances
  in Neural Information Processing Systems}\/} (pp.
  \bibinfo{pages}{2546--2554}).
\newblock \bibinfo{publisher}{Curran Associates, Inc.}
  volume~\bibinfo{volume}{24}.
\newblock \URLprefix
  \url{https://proceedings.neurips.cc/paper/2011/file/86e8f7ab32cfd12577bc2619bc635690-Paper.pdf}.
%Type = Techreport
\bibitem[{Bishop(1994)}]{bishop_mdn_1994}
\bibinfo{author}{Bishop, C.~M.} (\bibinfo{year}{1994}).
\newblock {\it \bibinfo{title}{Mixture density networks}\/}.
\newblock \bibinfo{type}{Technical Report} Aston University
  \bibinfo{address}{Birmingham}.
\newblock \URLprefix \url{https://publications.aston.ac.uk/id/eprint/373/}.
%Type = Misc
\bibitem[{Bolin \& Wallin(2019)}]{bolin2019scaled_crps}
\bibinfo{author}{Bolin, D.}, \& \bibinfo{author}{Wallin, J.}
  (\bibinfo{year}{2019}).
\newblock \bibinfo{title}{Local scale invariance and robustness of proper
  scoring rules}.
\newblock \URLprefix \url{https://arxiv.org/abs/1912.05642}.
  \DOIprefix\doi{10.48550/ARXIV.1912.05642}.
%Type = Article
\bibitem[{B\"{o}se et~al.(2017)B\"{o}se, Flunkert, Gasthaus, Januschowski,
  Lange, Salinas, Schelter, Seeger \& Wang}]{bose2017probabilistic_scale}
\bibinfo{author}{B\"{o}se, J.-H.}, \bibinfo{author}{Flunkert, V.},
  \bibinfo{author}{Gasthaus, J.}, \bibinfo{author}{Januschowski, T.},
  \bibinfo{author}{Lange, D.}, \bibinfo{author}{Salinas, D.},
  \bibinfo{author}{Schelter, S.}, \bibinfo{author}{Seeger, M.}, \&
  \bibinfo{author}{Wang, Y.} (\bibinfo{year}{2017}).
\newblock \bibinfo{title}{Probabilistic demand forecasting at scale}.
\newblock {\it \bibinfo{journal}{Proc. VLDB Endow.}\/},  {\it
  \bibinfo{volume}{10}\/}, \bibinfo{pages}{1694–1705}. \URLprefix
  \url{https://doi.org/10.14778/3137765.3137775}.
  \DOIprefix\doi{10.14778/3137765.3137775}.
%Type = Article
\bibitem[{Cho et~al.(2014)Cho, van Merrienboer, G{\"{u}}l{\c{c}}ehre, Bougares,
  Schwenk \& Bengio}]{cho2014Seq2SeqC}
\bibinfo{author}{Cho, K.}, \bibinfo{author}{van Merrienboer, B.},
  \bibinfo{author}{G{\"{u}}l{\c{c}}ehre, {\c{C}}.}, \bibinfo{author}{Bougares,
  F.}, \bibinfo{author}{Schwenk, H.}, \& \bibinfo{author}{Bengio, Y.}
  (\bibinfo{year}{2014}).
\newblock \bibinfo{title}{Learning phrase representations using {RNN}
  encoder-decoder for statistical machine translation}.
\newblock {\it \bibinfo{journal}{Proceedings of the 2014 Conference on
  Empirical Methods in Natural Language Processing (EMNLP)}\/},  {\it
  \bibinfo{volume}{abs/1406.1078}\/}, \bibinfo{pages}{1724--1734}. \URLprefix
  \url{http://arxiv.org/abs/1406.1078}.
  \href{http://arxiv.org/abs/1406.1078}{\tt arXiv:1406.1078}.
%Type = Article
\bibitem[{Christiansen \& Morris(1997)}]{christiansen1997hierarchical}
\bibinfo{author}{Christiansen, C.~L.}, \& \bibinfo{author}{Morris, C.~N.}
  (\bibinfo{year}{1997}).
\newblock \bibinfo{title}{Hierarchical poisson regression modeling}.
\newblock {\it \bibinfo{journal}{Journal of the American Statistical
  Association}\/},  {\it \bibinfo{volume}{92}\/}, \bibinfo{pages}{618--632}.
%Type = Article
\bibitem[{Clevert et~al.(2015)Clevert, Unterthiner \&
  Hochreiter}]{clevert2015elu_activations}
\bibinfo{author}{Clevert, D.-A.}, \bibinfo{author}{Unterthiner, T.}, \&
  \bibinfo{author}{Hochreiter, S.} (\bibinfo{year}{2015}).
\newblock \bibinfo{title}{Fast and accurate deep network learning by
  exponential linear units (elus)}.
\newblock {\it \bibinfo{journal}{arXiv preprint arXiv:1511.07289}\/}, .
%Type = Misc
\bibitem[{{Corporación Favorita}(2018)}]{favorita2018favorita_dataset}
\bibinfo{author}{{Corporación Favorita}} (\bibinfo{year}{2018}).
\newblock \bibinfo{title}{Corporación favorita grocery sales forecasting}.
\newblock \bibinfo{howpublished}{Kaggle Competition}.
\newblock \URLprefix
  \url{https://www.kaggle.com/c/favorita-grocery-sales-forecasting/}.
%Type = Article
\bibitem[{Diggle \& Brix(2001)}]{diggle2001}
\bibinfo{author}{Diggle, P.}, \& \bibinfo{author}{Brix, A.}
  (\bibinfo{year}{2001}).
\newblock \bibinfo{title}{Spatio-temporal prediction for log-gaussian cox
  processes}.
\newblock {\it \bibinfo{journal}{Journal of the Royal Statistical Society:
  Series B (Statistical Methodology}\/},  {\it \bibinfo{volume}{63}\/},
  \bibinfo{pages}{823--841}. \DOIprefix\doi{10.1111/1467-9868.00315}.
%Type = Book
\bibitem[{Diggle(2013)}]{diggle2013}
\bibinfo{author}{Diggle, P.~J.} (\bibinfo{year}{2013}).
\newblock {\it \bibinfo{title}{Statistical Analysis of Spatial and
  Spatio-Temporal Point Patterns}\/}.
\newblock (\bibinfo{edition}{Third edition} ed.).
\newblock \bibinfo{publisher}{Routledge}.
%Type = Misc
\bibitem[{Dua \& Graff(2017)}]{dua2017traffic_dataset}
\bibinfo{author}{Dua, D.}, \& \bibinfo{author}{Graff, C.}
  (\bibinfo{year}{2017}).
\newblock \bibinfo{title}{{UCI} machine learning repository}.
\newblock \URLprefix \url{http://archive.ics.uci.edu/ml}.
%Type = Inproceedings
\bibitem[{Eisenach et~al.(2021)Eisenach, Patel \&
  Madeka}]{eisenach2020mqtransformer}
\bibinfo{author}{Eisenach, C.}, \bibinfo{author}{Patel, Y.}, \&
  \bibinfo{author}{Madeka, D.} (\bibinfo{year}{2021}).
\newblock \bibinfo{title}{{MQTransformer: Multi-Horizon Forecasts with Context
  Dependent and Feedback-Aware Attention}}.
\newblock In \bibinfo{editor}{M.~F. Balcan}, \& \bibinfo{editor}{M.~Meila}
  (Eds.), {\it \bibinfo{booktitle}{Submitted to Proceedings of the 38th
  International Conference on Machine Learning}\/}.
\newblock \bibinfo{publisher}{PMLR. Working Paper version available at
  arXiv:2009.14799}.
%Type = Article
\bibitem[{Fliedner(1999)}]{fliedner1999hierarchical_top_down2}
\bibinfo{author}{Fliedner, G.} (\bibinfo{year}{1999}).
\newblock \bibinfo{title}{An investigation of aggregate variable time series
  forecast strategies with specific subaggregate time series statistical
  correlation}.
\newblock {\it \bibinfo{journal}{Computers and Operations Research}\/},  {\it
  \bibinfo{volume}{26}\/}, \bibinfo{pages}{1133–1149}. \URLprefix
  \url{https://doi.org/10.1016/S0305-0548(99)00017-9}.
  \DOIprefix\doi{10.1016/S0305-0548(99)00017-9}.
%Type = Misc
\bibitem[{{Fotios Petropoulos et. al.}(2021)}]{petropoulos2021forecasting}
\bibinfo{author}{{Fotios Petropoulos et. al.}} (\bibinfo{year}{2021}).
\newblock \bibinfo{title}{Forecasting: theory and practice}.
\newblock \href{http://arxiv.org/abs/2012.03854}{\tt arXiv:2012.03854}.
%Type = Article
\bibitem[{Gneiting \& Ranjan(2011)}]{gneiting2011comparing}
\bibinfo{author}{Gneiting, T.}, \& \bibinfo{author}{Ranjan, R.}
  (\bibinfo{year}{2011}).
\newblock \bibinfo{title}{Comparing density forecasts using threshold-and
  quantile-weighted scoring rules}.
\newblock {\it \bibinfo{journal}{Journal of Business \& Economic
  Statistics}\/},  {\it \bibinfo{volume}{29}\/}, \bibinfo{pages}{411--422}.
%Type = Article
\bibitem[{Gross \& Sohl(1990)}]{gross1990hierarchical_top_down}
\bibinfo{author}{Gross, C.~W.}, \& \bibinfo{author}{Sohl, J.~E.}
  (\bibinfo{year}{1990}).
\newblock \bibinfo{title}{Disaggregation methods to expedite product line
  forecasting}.
\newblock {\it \bibinfo{journal}{Journal of Forecasting}\/},  {\it
  \bibinfo{volume}{9}\/}, \bibinfo{pages}{233--254}. \URLprefix
  \url{https://onlinelibrary.wiley.com/doi/abs/10.1002/for.3980090304}.
  \DOIprefix\doi{10.1002/for.3980090304}.
%Type = Inproceedings
\bibitem[{Han et~al.(2021)Han, Dasgupta \& Ghosh}]{han2021hierarchical_sharq}
\bibinfo{author}{Han, X.}, \bibinfo{author}{Dasgupta, S.}, \&
  \bibinfo{author}{Ghosh, J.} (\bibinfo{year}{2021}).
\newblock \bibinfo{title}{Simultaneously reconciled quantile forecasting of
  hierarchically related time series}.
\newblock In \bibinfo{editor}{A.~Banerjee}, \& \bibinfo{editor}{K.~Fukumizu}
  (Eds.), {\it \bibinfo{booktitle}{Proceedings of The 24th International
  Conference on Artificial Intelligence and Statistics}\/} (pp.
  \bibinfo{pages}{190--198}).
\newblock \bibinfo{publisher}{PMLR} volume \bibinfo{volume}{130} of {\it
  \bibinfo{series}{Proceedings of Machine Learning Research}\/}.
\newblock \URLprefix \url{http://proceedings.mlr.press/v130/han21a.html}.
%Type = Article
\bibitem[{Hollyman et~al.(2021)Hollyman, Petropoulos \&
  Tipping}]{hollyman2021hierarchical_understanding_reconciliation}
\bibinfo{author}{Hollyman, R.}, \bibinfo{author}{Petropoulos, F.}, \&
  \bibinfo{author}{Tipping, M.~E.} (\bibinfo{year}{2021}).
\newblock \bibinfo{title}{Understanding forecast reconciliation}.
\newblock {\it \bibinfo{journal}{European Journal of Operational Research}\/},
  {\it \bibinfo{volume}{294}\/}, \bibinfo{pages}{149--160}. \URLprefix
  \url{https://www.sciencedirect.com/science/article/pii/S0377221721000199}.
  \DOIprefix\doi{https://doi.org/10.1016/j.ejor.2021.01.017}.
%Type = Misc
\bibitem[{Hong et~al.(2014)Hong, Pinson \& Fan}]{hong2014global}
\bibinfo{author}{Hong, T.}, \bibinfo{author}{Pinson, P.}, \&
  \bibinfo{author}{Fan, S.} (\bibinfo{year}{2014}).
\newblock \bibinfo{title}{Global energy forecasting competition 2012}.
%Type = Article
\bibitem[{Hyndman et~al.(2011)Hyndman, Ahmed, Athanasopoulos \&
  Shang}]{hyndman2011optimal_combination_hierarchical}
\bibinfo{author}{Hyndman, R.~J.}, \bibinfo{author}{Ahmed, R.~A.},
  \bibinfo{author}{Athanasopoulos, G.}, \& \bibinfo{author}{Shang, H.~L.}
  (\bibinfo{year}{2011}).
\newblock \bibinfo{title}{Optimal combination forecasts for hierarchical time
  series}.
\newblock {\it \bibinfo{journal}{Computational Statistics \& Data Analysis}\/},
   {\it \bibinfo{volume}{55}\/}, \bibinfo{pages}{2579 -- 2589}. \URLprefix
  \url{http://www.sciencedirect.com/science/article/pii/S0167947311000971}.
  \DOIprefix\doi{https://doi.org/10.1016/j.csda.2011.03.006}.
%Type = Book
\bibitem[{Hyndman \& Athanasopoulos(2018)}]{hyndman2017forecasting_book}
\bibinfo{author}{Hyndman, R.~J.}, \& \bibinfo{author}{Athanasopoulos, G.}
  (\bibinfo{year}{2018}).
\newblock {\it \bibinfo{title}{Forecasting: Principles and Practice}\/}.
\newblock \bibinfo{address}{{Melbourne, Australia}}:
  \bibinfo{publisher}{{OTexts}}.
\newblock \bibinfo{note}{Available at https://otexts.com/fpp2/}.
%Type = Article
\bibitem[{Hyndman \& Khandakar(2008)}]{hyndman2008automatic_arima}
\bibinfo{author}{Hyndman, R.~J.}, \& \bibinfo{author}{Khandakar, Y.}
  (\bibinfo{year}{2008}).
\newblock \bibinfo{title}{Automatic time series forecasting: The forecast
  package for r}.
\newblock {\it \bibinfo{journal}{Journal of Statistical Software, Articles}\/},
   {\it \bibinfo{volume}{27}\/}, \bibinfo{pages}{1--22}. \URLprefix
  \url{https://www.jstatsoft.org/v027/i03}.
  \DOIprefix\doi{10.18637/jss.v027.i03}.
%Type = Article
\bibitem[{Hyndman \& Koehler(2006)}]{hyndman2006another_look_measures}
\bibinfo{author}{Hyndman, R.~J.}, \& \bibinfo{author}{Koehler, A.~B.}
  (\bibinfo{year}{2006}).
\newblock \bibinfo{title}{Another look at measures of forecast accuracy}.
\newblock {\it \bibinfo{journal}{International Journal of Forecasting}\/},
  {\it \bibinfo{volume}{22}\/}, \bibinfo{pages}{679 -- 688}. \URLprefix
  \url{http://www.sciencedirect.com/science/article/pii/S0169207006000239}.
  \DOIprefix\doi{https://doi.org/10.1016/j.ijforecast.2006.03.001}.
%Type = Techreport
\bibitem[{Hyndman et~al.(2014)Hyndman, Lee \&
  Wang}]{hyndman2016hierarchical_groupedfast}
\bibinfo{author}{Hyndman, R.~J.}, \bibinfo{author}{Lee, A.}, \&
  \bibinfo{author}{Wang, E.} (\bibinfo{year}{2014}).
\newblock {\it \bibinfo{title}{{Fast computation of reconciled forecasts for
  hierarchical and grouped time series}}\/}.
\newblock \bibinfo{type}{Monash Econometrics and Business Statistics Working
  Papers} \bibinfo{number}{17/14} Monash University, Department of Econometrics
  and Business Statistics.
\newblock \URLprefix \url{https://ideas.repec.org/p/msh/ebswps/2014-17.html}.
%Type = Article
\bibitem[{Jeon et~al.(2019)Jeon, Panagiotelis \&
  Petropoulos}]{jeon2019coherent_quantile_forecasts}
\bibinfo{author}{Jeon, J.}, \bibinfo{author}{Panagiotelis, A.}, \&
  \bibinfo{author}{Petropoulos, F.} (\bibinfo{year}{2019}).
\newblock \bibinfo{title}{Probabilistic forecast reconciliation with
  applications to wind power and electric load}.
\newblock {\it \bibinfo{journal}{European Journal of Operational Research}\/},
  {\it \bibinfo{volume}{279}\/}, \bibinfo{pages}{364--379}. \URLprefix
  \url{https://www.sciencedirect.com/science/article/pii/S0377221719304242}.
  \DOIprefix\doi{https://doi.org/10.1016/j.ejor.2019.05.020}.
%Type = Article
\bibitem[{Kamarthi et~al.(2022)Kamarthi, Kong, Rodriguez, Zhang \&
  Prakash}]{kamarthi2022profhit_network}
\bibinfo{author}{Kamarthi, H.}, \bibinfo{author}{Kong, L.},
  \bibinfo{author}{Rodriguez, A.}, \bibinfo{author}{Zhang, C.}, \&
  \bibinfo{author}{Prakash, B.} (\bibinfo{year}{2022}).
\newblock \bibinfo{title}{{PROFHIT}: Probabilistic robust forecasting for
  hierarchical time-series}.
\newblock {\it \bibinfo{journal}{Computing Research Repository}\/}, .
  \URLprefix \url{https://arxiv.org/abs/2206.07940}.
%Type = Misc
\bibitem[{Kingma \& Ba(2014)}]{kingma2014adam_method}
\bibinfo{author}{Kingma, D.~P.}, \& \bibinfo{author}{Ba, J.}
  (\bibinfo{year}{2014}).
\newblock \bibinfo{title}{{ADAM}: A method for stochastic optimization}.
\newblock \URLprefix \url{http://arxiv.org/abs/1412.6980} \bibinfo{note}{cite
  arxiv:1412.6980Comment: Published as a conference paper at the 3rd
  International Conference for Learning Representations (ICLR), San Diego,
  2015}.
%Type = Article
\bibitem[{Lindsay(1988)}]{lindsay_1988}
\bibinfo{author}{Lindsay, B.~G.} (\bibinfo{year}{1988}).
\newblock \bibinfo{title}{Composite likelihood methods}.
\newblock {\it \bibinfo{journal}{Contemporary Mathematics}\/},  {\it
  \bibinfo{volume}{80}\/}, \bibinfo{pages}{221--239}.
%Type = Inproceedings
\bibitem[{Madeka et~al.(2018)Madeka, Swiniarski, Foster, Razoumov, Torkkola \&
  Wen}]{madeka2018sample}
\bibinfo{author}{Madeka, D.}, \bibinfo{author}{Swiniarski, L.},
  \bibinfo{author}{Foster, D.}, \bibinfo{author}{Razoumov, L.},
  \bibinfo{author}{Torkkola, K.}, \& \bibinfo{author}{Wen, R.}
  (\bibinfo{year}{2018}).
\newblock \bibinfo{title}{Sample path generation for probabilistic demand
  forecasting}.
\newblock In {\it \bibinfo{booktitle}{ICML workshop on Theoretical Foundations
  and Applications of Deep Generative Models}\/}.
%Type = Article
\bibitem[{Makridakis et~al.(2018{\natexlab{a}})Makridakis, Spiliotis \&
  Assimakopoulos}]{makridakis2018m4competition_results}
\bibinfo{author}{Makridakis, S.}, \bibinfo{author}{Spiliotis, E.}, \&
  \bibinfo{author}{Assimakopoulos, V.} (\bibinfo{year}{2018}{\natexlab{a}}).
\newblock \bibinfo{title}{The {M4} competition: Results, findings, conclusion
  and way forward}.
\newblock {\it \bibinfo{journal}{International Journal of Forecasting}\/},
  {\it \bibinfo{volume}{34}\/}, \bibinfo{pages}{802 -- 808}. \URLprefix
  \url{http://www.sciencedirect.com/science/article/pii/S0169207018300785}.
  \DOIprefix\doi{https://doi.org/10.1016/j.ijforecast.2018.06.001}.
%Type = Article
\bibitem[{Makridakis et~al.(2018{\natexlab{b}})Makridakis, Spiliotis \&
  Assimakopoulos}]{makridakis2018machine_learning_concerns}
\bibinfo{author}{Makridakis, S.}, \bibinfo{author}{Spiliotis, E.}, \&
  \bibinfo{author}{Assimakopoulos, V.} (\bibinfo{year}{2018}{\natexlab{b}}).
\newblock \bibinfo{title}{Statistical and machine learning forecasting methods:
  Concerns and ways forward}.
\newblock {\it \bibinfo{journal}{PLOS ONE}\/},  {\it \bibinfo{volume}{13}\/},
  \bibinfo{pages}{1--26}. \URLprefix
  \url{https://doi.org/10.1371/journal.pone.0194889}.
  \DOIprefix\doi{10.1371/journal.pone.0194889}.
%Type = Article
\bibitem[{Makridakis et~al.(2020)Makridakis, Spiliotis \&
  Assimakopoulos}]{makridakis2020m5competition_results}
\bibinfo{author}{Makridakis, S.}, \bibinfo{author}{Spiliotis, E.}, \&
  \bibinfo{author}{Assimakopoulos, V.} (\bibinfo{year}{2020}).
\newblock \bibinfo{title}{The {M}5 accuracy competition: Results, findings and
  conclusions}.
\newblock {\it \bibinfo{journal}{International Journal of Forecasting}\/}, .
  \URLprefix
  \url{https://www.researchgate.net/publication/344487258_The_M5_Accuracy_competition_Results_findings_and_conclusions}.
%Type = Article
\bibitem[{Makridakis et~al.(2022)Makridakis, Spiliotis, Assimakopoulos, Chen,
  Gaba, Tsetlin \& Winkler}]{makridakis2022m5_uncertainty}
\bibinfo{author}{Makridakis, S.}, \bibinfo{author}{Spiliotis, E.},
  \bibinfo{author}{Assimakopoulos, V.}, \bibinfo{author}{Chen, Z.},
  \bibinfo{author}{Gaba, A.}, \bibinfo{author}{Tsetlin, I.}, \&
  \bibinfo{author}{Winkler, R.~L.} (\bibinfo{year}{2022}).
\newblock \bibinfo{title}{The m5 uncertainty competition: Results, findings and
  conclusions}.
\newblock {\it \bibinfo{journal}{International Journal of Forecasting}\/},
  {\it \bibinfo{volume}{38}\/}, \bibinfo{pages}{1365--1385}. \URLprefix
  \url{https://www.sciencedirect.com/science/article/pii/S0169207021001722}.
  \DOIprefix\doi{https://doi.org/10.1016/j.ijforecast.2021.10.009}.
\newblock \bibinfo{note}{Special Issue: M5 competition}.
%Type = Article
\bibitem[{Matheson \& Winkler(1976)}]{matheson1976evaluation_crps}
\bibinfo{author}{Matheson, J.~E.}, \& \bibinfo{author}{Winkler, R.~L.}
  (\bibinfo{year}{1976}).
\newblock \bibinfo{title}{Scoring rules for continuous probability
  distributions}.
\newblock {\it \bibinfo{journal}{Management Science}\/},  {\it
  \bibinfo{volume}{22}\/}, \bibinfo{pages}{1087--1096}. \URLprefix
  \url{http://www.jstor.org/stable/2629907}.
%Type = Article
\bibitem[{Nelder \& Wedderburn(1972)}]{nelder1972generalized_linear}
\bibinfo{author}{Nelder, J.~A.}, \& \bibinfo{author}{Wedderburn, R. W.~M.}
  (\bibinfo{year}{1972}).
\newblock \bibinfo{title}{Generalized linear models}.
\newblock {\it \bibinfo{journal}{Journal of the Royal Statistical Society.
  Series A (General)}\/},  {\it \bibinfo{volume}{135}\/},
  \bibinfo{pages}{370--384}. \URLprefix
  \url{http://www.jstor.org/stable/2344614}.
%Type = Article
\bibitem[{Olivares et~al.(2022)Olivares, Challu, Marcjasz, Weron \&
  Dubrawski}]{olivares2022nbeatsx}
\bibinfo{author}{Olivares, K.~G.}, \bibinfo{author}{Challu, C.},
  \bibinfo{author}{Marcjasz, G.}, \bibinfo{author}{Weron, R.}, \&
  \bibinfo{author}{Dubrawski, A.} (\bibinfo{year}{2022}).
\newblock \bibinfo{title}{Neural basis expansion analysis with exogenous
  variables: Forecasting electricity prices with nbeatsx}.
\newblock {\it \bibinfo{journal}{International Journal of Forecasting}\/}, .
  \URLprefix
  \url{https://www.sciencedirect.com/science/article/pii/S0169207022000413}.
  \DOIprefix\doi{https://doi.org/10.1016/j.ijforecast.2022.03.001}.
%Type = Article
\bibitem[{van~den {O}ord et~al.(2016)van~den {O}ord, Dieleman, Zen, Simonyan,
  Vinyals, Graves, Kalchbrenner, Senior \& Kavukcuoglu}]{vandenoord2016wavenet}
\bibinfo{author}{van~den {O}ord, A.}, \bibinfo{author}{Dieleman, S.},
  \bibinfo{author}{Zen, H.}, \bibinfo{author}{Simonyan, K.},
  \bibinfo{author}{Vinyals, O.}, \bibinfo{author}{Graves, A.},
  \bibinfo{author}{Kalchbrenner, N.}, \bibinfo{author}{Senior, A.~W.}, \&
  \bibinfo{author}{Kavukcuoglu, K.} (\bibinfo{year}{2016}).
\newblock \bibinfo{title}{{W}ave{N}et: {A} generative model for raw audio}.
\newblock {\it \bibinfo{journal}{Computer Research Repository}\/},  {\it
  \bibinfo{volume}{abs/1609.03499}\/}. \URLprefix
  \url{http://arxiv.org/abs/1609.03499}.
  \href{http://arxiv.org/abs/1609.03499}{\tt arXiv:1609.03499}.
%Type = Article
\bibitem[{Orcutt et~al.(1968)Orcutt, Watts \&
  Edwards}]{orcutt1968hierarchical_bottom_up}
\bibinfo{author}{Orcutt, G.~H.}, \bibinfo{author}{Watts, H.~W.}, \&
  \bibinfo{author}{Edwards, J.~B.} (\bibinfo{year}{1968}).
\newblock \bibinfo{title}{Data aggregation and information loss}.
\newblock {\it \bibinfo{journal}{The American Economic Review}\/},  {\it
  \bibinfo{volume}{58}\/}, \bibinfo{pages}{773--787}. \URLprefix
  \url{http://www.jstor.org/stable/1815532}.
%Type = Techreport
\bibitem[{Panagiotelis et~al.(2020)Panagiotelis, Gamakumara, Athanasopoulos \&
  Hyndman}]{panagiotelis2020hierarchical_probabilistic_coherence}
\bibinfo{author}{Panagiotelis, A.}, \bibinfo{author}{Gamakumara, P.},
  \bibinfo{author}{Athanasopoulos, G.}, \& \bibinfo{author}{Hyndman, R.~J.}
  (\bibinfo{year}{2020}).
\newblock {\it \bibinfo{title}{{Probabilistic Forecast Reconciliation:
  Properties, Evaluation and Score Optimisation}}\/}.
\newblock \bibinfo{type}{Monash Econometrics and Business Statistics Working
  Papers} \bibinfo{number}{26/20} Monash University, Department of Econometrics
  and Business Statistics.
\newblock \URLprefix \url{https://ideas.repec.org/p/msh/ebswps/2020-26.html}.
%Type = Article
\bibitem[{Panagiotelis et~al.(2023)Panagiotelis, Gamakumara, Athanasopoulos \&
  Hyndman}]{panagiotelis2023probabilistic_reconciliation}
\bibinfo{author}{Panagiotelis, A.}, \bibinfo{author}{Gamakumara, P.},
  \bibinfo{author}{Athanasopoulos, G.}, \& \bibinfo{author}{Hyndman, R.~J.}
  (\bibinfo{year}{2023}).
\newblock \bibinfo{title}{Probabilistic forecast reconciliation: Properties,
  evaluation and score optimisation}.
\newblock {\it \bibinfo{journal}{European Journal of Operational Research}\/},
  {\it \bibinfo{volume}{306}\/}, \bibinfo{pages}{693--706}. \URLprefix
  \url{https://www.sciencedirect.com/science/article/pii/S0377221722006087}.
  \DOIprefix\doi{https://doi.org/10.1016/j.ejor.2022.07.040}.
%Type = Inproceedings
\bibitem[{Paria et~al.(2021)Paria, Sen, Ahmed \&
  Das}]{paria2021hierarchical_hired}
\bibinfo{author}{Paria, B.}, \bibinfo{author}{Sen, R.}, \bibinfo{author}{Ahmed,
  A.}, \& \bibinfo{author}{Das, A.} (\bibinfo{year}{2021}).
\newblock \bibinfo{title}{{Hierarchically Regularized Deep Forecasting}}.
\newblock In {\it \bibinfo{booktitle}{Submitted to Proceedings of the 39th
  International Conference on Machine Learning}\/}.
\newblock \bibinfo{publisher}{PMLR. Working Paper version available at
  arXiv:2106.07630}.
%Type = Article
\bibitem[{Park \& Lord(2009)}]{park2009application}
\bibinfo{author}{Park, B.-J.}, \& \bibinfo{author}{Lord, D.}
  (\bibinfo{year}{2009}).
\newblock \bibinfo{title}{Application of finite mixture models for vehicle
  crash data analysis}.
\newblock {\it \bibinfo{journal}{Accident Analysis \& Prevention}\/},  {\it
  \bibinfo{volume}{41}\/}, \bibinfo{pages}{683--691}.
%Type = Article
\bibitem[{Puwasala et~al.(2018)Puwasala, Panagiotelis~Anastasios \&
  Hyndman}]{gamakumara_2018}
\bibinfo{author}{Puwasala, G.}, \bibinfo{author}{Panagiotelis~Anastasios, G.,
  Athanasopoulos}, \& \bibinfo{author}{Hyndman, R.~J.} (\bibinfo{year}{2018}).
\newblock \bibinfo{title}{{Probabilisitic Forecasts in Hierarchical Time
  Series}}.
\newblock {\it \bibinfo{journal}{Department of Econometrics and Business
  Statistics Working Paper Series 11/18}\/}, .
%Type = Inproceedings
\bibitem[{Rangapuram et~al.(2018)Rangapuram, Seeger, Gasthaus, Stella, Wang \&
  Januschowski}]{rangapuram2018deep_state_space}
\bibinfo{author}{Rangapuram, S.~S.}, \bibinfo{author}{Seeger, M.~W.},
  \bibinfo{author}{Gasthaus, J.}, \bibinfo{author}{Stella, L.},
  \bibinfo{author}{Wang, Y.}, \& \bibinfo{author}{Januschowski, T.}
  (\bibinfo{year}{2018}).
\newblock \bibinfo{title}{Deep state space models for time series forecasting}.
\newblock In \bibinfo{editor}{S.~Bengio}, \bibinfo{editor}{H.~Wallach},
  \bibinfo{editor}{H.~Larochelle}, \bibinfo{editor}{K.~Grauman},
  \bibinfo{editor}{N.~Cesa-Bianchi}, \& \bibinfo{editor}{R.~Garnett} (Eds.),
  {\it \bibinfo{booktitle}{Advances in Neural Information Processing
  Systems}\/}.
\newblock \bibinfo{publisher}{Curran Associates, Inc.}
  volume~\bibinfo{volume}{31}.
\newblock \URLprefix
  \url{https://proceedings.neurips.cc/paper/2018/file/5cf68969fb67aa6082363a6d4e6468e2-Paper.pdf}.
%Type = Inproceedings
\bibitem[{Rangapuram et~al.(2021)Rangapuram, Werner, Benidis, Mercado, Gasthaus
  \& Januschowski}]{rangapuram2021hierarchical_e2e}
\bibinfo{author}{Rangapuram, S.~S.}, \bibinfo{author}{Werner, L.~D.},
  \bibinfo{author}{Benidis, K.}, \bibinfo{author}{Mercado, P.},
  \bibinfo{author}{Gasthaus, J.}, \& \bibinfo{author}{Januschowski, T.}
  (\bibinfo{year}{2021}).
\newblock \bibinfo{title}{End-to-end learning of coherent probabilistic
  forecasts for hierarchical time series}.
\newblock In \bibinfo{editor}{M.~F. Balcan}, \& \bibinfo{editor}{M.~Meila}
  (Eds.), {\it \bibinfo{booktitle}{Proceedings of the 38th International
  Conference on Machine Learning}\/} Proceedings of Machine Learning Research.
\newblock \bibinfo{publisher}{PMLR}.
%Type = Article
\bibitem[{Ravuri et~al.(2021)Ravuri, Lenc, Willson, Kangin, Lam, Mirowski,
  Fitzsimons, Athanassiadou, Kashem, Madge, Prudden, Mandhane, Clark, Brock,
  Simonyan, Hadsell, Robinson, Clancy, Arribas \&
  Mohamed}]{ravuri2021weather_deepmind}
\bibinfo{author}{Ravuri, S.~V.}, \bibinfo{author}{Lenc, K.},
  \bibinfo{author}{Willson, M.}, \bibinfo{author}{Kangin, D.},
  \bibinfo{author}{Lam, R.}, \bibinfo{author}{Mirowski, P.},
  \bibinfo{author}{Fitzsimons, M.}, \bibinfo{author}{Athanassiadou, M.},
  \bibinfo{author}{Kashem, S.}, \bibinfo{author}{Madge, S.},
  \bibinfo{author}{Prudden, R.}, \bibinfo{author}{Mandhane, A.},
  \bibinfo{author}{Clark, A.}, \bibinfo{author}{Brock, A.},
  \bibinfo{author}{Simonyan, K.}, \bibinfo{author}{Hadsell, R.},
  \bibinfo{author}{Robinson, N.~H.}, \bibinfo{author}{Clancy, E.},
  \bibinfo{author}{Arribas, A.}, \& \bibinfo{author}{Mohamed, S.}
  (\bibinfo{year}{2021}).
\newblock \bibinfo{title}{Skillful precipitation nowcasting using deep
  generative models of radar}.
\newblock {\it \bibinfo{journal}{Nature}\/},  {\it \bibinfo{volume}{597}\/},
  \bibinfo{pages}{672--691}. \URLprefix
  \url{https://www.nature.com/articles/s41586-021-03854-z.pdf}.
%Type = Techreport
\bibitem[{Rosenblatt(1961)}]{rosenblatt1961principles}
\bibinfo{author}{Rosenblatt, F.} (\bibinfo{year}{1961}).
\newblock {\it \bibinfo{title}{Principles of neurodynamics. perceptrons and the
  theory of brain mechanisms}\/}.
\newblock \bibinfo{type}{Technical Report} Cornell Aeronautical Lab Inc Buffalo
  NY.
%Type = Article
\bibitem[{Semenoglou et~al.(2021)Semenoglou, Spiliotis, Makridakis \&
  Assimakopoulos}]{spiliotis2021cross_learning}
\bibinfo{author}{Semenoglou, A.-A.}, \bibinfo{author}{Spiliotis, E.},
  \bibinfo{author}{Makridakis, S.}, \& \bibinfo{author}{Assimakopoulos, V.}
  (\bibinfo{year}{2021}).
\newblock \bibinfo{title}{Investigating the accuracy of cross-learning time
  series forecasting methods}.
\newblock {\it \bibinfo{journal}{International Journal of Forecasting}\/},
  {\it \bibinfo{volume}{37}\/}, \bibinfo{pages}{1072--1084}. \URLprefix
  \url{https://www.sciencedirect.com/science/article/pii/S0169207020301850}.
  \DOIprefix\doi{https://doi.org/10.1016/j.ijforecast.2020.11.009}.
%Type = Article
\bibitem[{Shang \& Hyndman(2017)}]{shang2017coherent_quantile_forecasts}
\bibinfo{author}{Shang, H.~L.}, \& \bibinfo{author}{Hyndman, R.~J.}
  (\bibinfo{year}{2017}).
\newblock \bibinfo{title}{Grouped functional time series forecasting: An
  application to age-specific mortality rates}.
\newblock {\it \bibinfo{journal}{Journal of Computational and Graphical
  Statistics}\/},  {\it \bibinfo{volume}{26}\/}, \bibinfo{pages}{330--343}.
  \URLprefix \url{https://doi.org/10.1080/10618600.2016.1237877}.
  \DOIprefix\doi{10.1080/10618600.2016.1237877}.
  \href{http://arxiv.org/abs/https://doi.org/10.1080/10618600.2016.1237877}{\tt
  arXiv:https://doi.org/10.1080/10618600.2016.1237877}.
%Type = Inproceedings
\bibitem[{Souhaib \&
  Bonsoo(2019)}]{taieb2019hierarchical_regularized_regression}
\bibinfo{author}{Souhaib, B.}, \& \bibinfo{author}{Bonsoo, K.}
  (\bibinfo{year}{2019}).
\newblock \bibinfo{title}{Regularized regression for hierarchical forecasting
  without unbiasedness conditions}.
\newblock In {\it \bibinfo{booktitle}{Proceedings of the 25th ACM SIGKDD
  International Conference on Knowledge Discovery \& Data Mining}\/} KDD '19
  (p. \bibinfo{pages}{1337–1347}).
\newblock \bibinfo{address}{New York, NY, USA}: \bibinfo{publisher}{Association
  for Computing Machinery}.
\newblock \URLprefix \url{https://doi.org/10.1145/3292500.3330976}.
  \DOIprefix\doi{10.1145/3292500.3330976}.
%Type = Article
\bibitem[{Spiliotis et~al.(2020)Spiliotis, Petropoulos, Kourentzes \&
  Assimakopoulos}]{spiliotis2020hierarchical_cross}
\bibinfo{author}{Spiliotis, E.}, \bibinfo{author}{Petropoulos, F.},
  \bibinfo{author}{Kourentzes, N.}, \& \bibinfo{author}{Assimakopoulos, V.}
  (\bibinfo{year}{2020}).
\newblock \bibinfo{title}{Cross-temporal aggregation: Improving the forecast
  accuracy of hierarchical electricity consumption}.
\newblock {\it \bibinfo{journal}{Applied Energy}\/},  {\it
  \bibinfo{volume}{261}\/}, \bibinfo{pages}{114339}.
%Type = Article
\bibitem[{{Tianqi Chen et al.}(2015)}]{chen2015mxnet}
\bibinfo{author}{{Tianqi Chen et al.}} (\bibinfo{year}{2015}).
\newblock \bibinfo{title}{Mxnet: {A} flexible and efficient machine learning
  library for heterogeneous distributed systems}.
\newblock {\it \bibinfo{journal}{Computing Research Repository}\/},  {\it
  \bibinfo{volume}{1512.01274}\/}. \URLprefix
  \url{http://arxiv.org/abs/1512.01274}.
%Type = Misc
\bibitem[{{Tourism Australia, Canberra}(2019)}]{canberra2019tourism}
\bibinfo{author}{{Tourism Australia, Canberra}} (\bibinfo{year}{2019}).
\newblock \bibinfo{title}{Detailed tourism {R}esearch {A}ustralia (2005),
  {T}ravel by {A}ustralians}.
\newblock \bibinfo{note}{Accessed at
  https://robjhyndman.com/publications/hierarchical-tourism/}.
%Type = Incollection
\bibitem[{Van~Erven \& Cugliari(2015)}]{van2015game}
\bibinfo{author}{Van~Erven, T.}, \& \bibinfo{author}{Cugliari, J.}
  (\bibinfo{year}{2015}).
\newblock \bibinfo{title}{Game-theoretically optimal reconciliation of
  contemporaneous hierarchical time series forecasts}.
\newblock In {\it \bibinfo{booktitle}{Modeling and stochastic learning for
  forecasting in high dimensions}\/} (pp. \bibinfo{pages}{297--317}).
\newblock \bibinfo{publisher}{Springer}.
%Type = Article
\bibitem[{Varin et~al.(2011)Varin, Reid \& Firth}]{varin_2011}
\bibinfo{author}{Varin, C.}, \bibinfo{author}{Reid, N.}, \&
  \bibinfo{author}{Firth, D.} (\bibinfo{year}{2011}).
\newblock \bibinfo{title}{An overview of composite likelihood methods}.
\newblock {\it \bibinfo{journal}{Statistica Sinica}\/},  {\it
  \bibinfo{volume}{21}\/}, \bibinfo{pages}{5--42}. \URLprefix
  \url{http://www.jstor.org/stable/24309261}.
%Type = Inproceedings
\bibitem[{Wen et~al.(2017)Wen, Torkkola, Narayanaswamy \&
  Madeka}]{wen2017mqrcnn}
\bibinfo{author}{Wen, R.}, \bibinfo{author}{Torkkola, K.},
  \bibinfo{author}{Narayanaswamy, B.}, \& \bibinfo{author}{Madeka, D.}
  (\bibinfo{year}{2017}).
\newblock \bibinfo{title}{A {M}ulti-horizon {Q}uantile {R}ecurrent
  {F}orecaster}.
\newblock In {\it \bibinfo{booktitle}{31st Conference on Neural Information
  Processing Systems {NIPS} 2017, Time Series Workshop}\/}.
\newblock \URLprefix \url{https://arxiv.org/abs/1711.11053}.
  \href{http://arxiv.org/abs/1711.11053}{\tt arXiv:1711.11053}.
%Type = Article
\bibitem[{Wickramasuriya(2023)}]{wickramasuriya2023probabilistic_gaussian}
\bibinfo{author}{Wickramasuriya, S.~L.} (\bibinfo{year}{2023}).
\newblock \bibinfo{title}{{Probabilistic forecast reconciliation under the
  Gaussian framework}}.
\newblock {\it \bibinfo{journal}{Accepted at Journal of Business and Economic
  Statistics}\/}, .
%Type = Article
\bibitem[{Wickramasuriya et~al.(2019)Wickramasuriya, Athanasopoulos \&
  Hyndman}]{wickramasuriya2019hierarchical_mint_reconciliation}
\bibinfo{author}{Wickramasuriya, S.~L.}, \bibinfo{author}{Athanasopoulos, G.},
  \& \bibinfo{author}{Hyndman, R.~J.} (\bibinfo{year}{2019}).
\newblock \bibinfo{title}{Optimal forecast reconciliation for hierarchical and
  grouped time series through trace minimization}.
\newblock {\it \bibinfo{journal}{Journal of the American Statistical
  Association}\/},  {\it \bibinfo{volume}{114}\/}, \bibinfo{pages}{804--819}.
  \URLprefix \url{https://robjhyndman.com/publications/mint/}.
  \DOIprefix\doi{10.1080/01621459.2018.1448825}.
%Type = Article
\bibitem[{Wikle et~al.(1998)Wikle, Berliner \& Cressie}]{wikle1998hierarchical}
\bibinfo{author}{Wikle, C.~K.}, \bibinfo{author}{Berliner, L.~M.}, \&
  \bibinfo{author}{Cressie, N.} (\bibinfo{year}{1998}).
\newblock \bibinfo{title}{Hierarchical bayesian space-time models}.
\newblock {\it \bibinfo{journal}{Environmental and ecological statistics}\/},
  {\it \bibinfo{volume}{5}\/}, \bibinfo{pages}{117--154}.
%Type = Article
\bibitem[{Yao et~al.(2007)Yao, Rosasco \& Andrea}]{yuan2007early_stopping}
\bibinfo{author}{Yao, Y.}, \bibinfo{author}{Rosasco, L.}, \&
  \bibinfo{author}{Andrea, C.} (\bibinfo{year}{2007}).
\newblock \bibinfo{title}{On early stopping in gradient descent learning}.
\newblock {\it \bibinfo{journal}{Constructive Approximation}\/},  {\it
  \bibinfo{volume}{26(2)}\/}, \bibinfo{pages}{289--315}. \URLprefix
  \url{https://doi.org/10.1007/s00365-006-0663-2}.
%Type = Inproceedings
\bibitem[{Yu et~al.(2018)Yu, Yin \& Zhu}]{yu2017graphconv_traffic}
\bibinfo{author}{Yu, B.}, \bibinfo{author}{Yin, H.}, \& \bibinfo{author}{Zhu,
  Z.} (\bibinfo{year}{2018}).
\newblock \bibinfo{title}{Spatio-temporal graph convolutional neural network:
  {A} deep learning framework for traffic forecasting}.
\newblock In {\it \bibinfo{booktitle}{Proceedings of the 27th International
  Joint Conference on Artificial Intelligence (IJCAI)}\/}.
\newblock \URLprefix \url{http://arxiv.org/abs/1709.04875}.

\end{thebibliography}
